\documentclass[11pt]{article}

\usepackage{times}
\usepackage{graphicx}
\usepackage{fullpage}
\usepackage{amsmath,amssymb,amsthm,mathtools}
\usepackage{enumerate}
\usepackage[mathcal]{euscript}
\usepackage{comment}
\usepackage[round]{natbib}

\usepackage[T1]{fontenc}
\usepackage{microtype}
\usepackage{enumitem}
\usepackage{algorithm}
\usepackage{algorithmic}
\usepackage[textsize=tiny]{todonotes}
\usepackage{comment}

\usepackage{commath}
\usepackage[normalem]{ulem}

\usepackage[plainpages=false,pdfpagelabels,colorlinks=true,linkcolor=BrickRed,citecolor=plum]{hyperref}

\usepackage{cleveref}

\DeclareMathOperator{\sym}{sym}

\definecolor{jz}{rgb}{0.1,0.45,0.1}
\newcommand{\jz}[1]{{\color{jz}{#1}}}

\definecolor{rr}{rgb}{0,0,1}

\newcommand{\rrtd}[2][]{\todo[color=rr!40,#1]{#2}}

\definecolor{ymm}{rgb}{1,0.53,0.0}

\newcommand{\dkl}{\mathcal{D}_{\rm KL}}

\newcommand{\R}{\mathbb{R}}
\newcommand{\N}{\mathbb{N}}
\newcommand{\bsv}{{\boldsymbol v}}
\newcommand{\bsw}{{\boldsymbol w}}

\newcommand{\bse}{{\boldsymbol e}}

\newcommand{\bsalpha}{{\boldsymbol \alpha}}

\renewcommand{\set}[2]{\{#1\,:\,#2\}}
\newcommand{\setl}[2]{ \{#1\, :\,#2 \}}
\newcommand{\setc}[2]{\left\{#1\, :\,#2\right\}}
\DeclareMathOperator{\tr}{tr}
\newcommand{\ii}{\mathrm{i}}
\newcommand{\dd}{\;\mathrm{d}}
\newcommand{\e}{\mathrm{e}}

\numberwithin{equation}{section}

\definecolor{plum}{rgb}{.4,0,.4}
\definecolor{BrickRed}{rgb}{0.6,0,0}

\def\ddefloop#1{\ifx\ddefloop#1\else\ddef{#1}\expandafter\ddefloop\fi}

\def\ddef#1{\expandafter\def\csname c#1\endcsname{\ensuremath{\mathcal{#1}}}}
\ddefloop ABCDEFGHIJKLMNOPQRSTUVWXYZ\ddefloop

\def\ddef#1{\expandafter\def\csname s#1\endcsname{\ensuremath{\mathsf{#1}}}}
\ddefloop ABCDEFGHIJKLMNOPQRSTUVWXYZ\ddefloop

\def\argmin{\operatornamewithlimits{arg\,min}}

	\def\tr{{\mathrm{tr}}}

\def\1{{\mathbf 1}}

\def\eps{\varepsilon}

\newtheorem{theorem}{Theorem}[section]
\newtheorem{lemma}[theorem]{Lemma}
\newtheorem{proposition}[theorem]{Proposition}
\newtheorem{corollary}[theorem]{Corollary}
\newtheorem{assumption}[theorem]{Assumption}
\newtheorem{remark}[theorem]{Remark}
\newtheorem{example}[theorem]{Example}
\newtheorem{definition}[theorem]{Definition}

\begin{document}

\title{Distribution learning via neural differential equations: minimal energy regularization and approximation theory}

\author{Youssef Marzouk\thanks{Massachusetts Institute of Technology, Cambridge, MA 02139, USA (\href{mailto:ymarz@mit.edu}{ymarz@mit.edu}, \href{mailto:zren@mit.edu}{zren@mit.edu})} \and Zhi Ren\footnotemark[1]
\and Jakob Zech\thanks{Heidelberg University, 69120 Heidelberg, Germany (\href{mailto:jakob.zech@uni-heidelberg.de}{jakob.zech@uni-heidelberg.de})} 
}

\date{\today}

\maketitle
	
\begin{abstract} 
Neural ordinary differential equations (ODEs) provide expressive representations of invertible transport maps that can be used to approximate complex probability distributions, e.g., for generative modeling, density estimation, and Bayesian inference. 
We show that for a large class of transport maps $T$, there exists a time-dependent ODE velocity field realizing a straight-line interpolation $(1-t)x + tT(x)$, $t \in [0,1]$, of the displacement induced by the map. Moreover, we show that such velocity fields are minimizers of a training objective containing a specific minimum-energy regularization. 
We then derive explicit upper bounds for the $C^k$ norm of the velocity field that are polynomial in the $C^k$ norm of the corresponding transport map $T$; in the case of triangular (Knothe--Rosenblatt) maps, we also show that these bounds are polynomial in the $C^k$ norms of the associated source and target densities. 
Combining these results with stability arguments for distribution approximation via ODEs, we show that Wasserstein or Kullback--Leibler approximation of the target distribution to any desired accuracy $\epsilon > 0$ can be achieved by a deep neural network representation of the velocity field whose size 
is bounded 
in terms of $\epsilon$, the dimension, and the smoothness of the source and target densities. 
The same neural network ansatz yields guarantees on the value of the regularized training objective.

\end{abstract}

\newcommand{\co}{[0,1]}
\newcommand{\ci}{[0,1]^i}
\newcommand{\cd}{[0,1]^d}
\newcommand{\essinf}{{\rm ess\,inf}}
\newcommand{\supp}{{\rm supp}}
\tableofcontents

\section{Introduction}\label{sec:intro}

Sampling from an arbitrary probability distribution is a central problem in computational statistics and machine learning. Transportation of measure \citep{OptimalOldAndNew} offers a useful approach to this problem: the idea is to construct a measurable map that \emph{pushes forward} a relatively simple source distribution (usually chosen to be uniform or standard Gaussian) to the target probability distribution. One can then simulate from the target distribution by drawing samples from the source distribution and evaluating the transport map. This construction is useful for both generative modeling \citep{ffjord,glow} and variational inference \citep{moselhy2012,PlanarFlow}. When the map is invertible, one can also estimate the \emph{density} of the target measure by evaluating the density of the pushforward of the source distribution under the transport map.


Many parameterizations of such transports, ranging from monotone polynomial-based approximations 
\citep{measure-transport,ZM1,ZM2,baptista23}
to input-convex neural networks \citep{huang2021convex,2310.16975} have been proposed. In the machine learning literature, normalizing flows \citep{cms/1266935020,NormalizingFlowIntro} represent transport maps by composing a sequence of relatively simple invertible functions. It is typically required that the Jacobian determinant of the resulting map be easily computable, as it appears in expressions for the pushforward or pullback densities via the change-of-variables formula. Common constructions for normalizing flows including planar and radial flows \citep{PlanarFlow}, affine coupling flows \citep{CouplingFlows}, autoregressive flows \citep{autoregressiveflow}, and neural autoregressive flows \citep{NeuralAutoFlow}.

More recent work in deep learning has elucidated connections between deep neural networks and differential equations \citep{TransportAnalysisofDL,NumericalSchemeODE,PDEmotivatedNN}. In particular, neural ordinary differential equations (neural ODEs) \citep{NeuralODE} are ODEs whose velocity fields are represented by neural networks. A neural ODE can be understood intuitively as the ``continuous-time limit'' of a normalizing flow, insofar as the flow map of the ODE is given by the composition of infinitely many incremental transformations. Neural ODEs can be used to represent distributions \citep{ffjord}---again for the purposes of generative modeling, inference, or density estimation---as follows. Let $\pi$ denote the target distribution from which we wish to sample and let $\rho$ denote the source distribution. Solving the initial value problem
\begin{equation}\label{eq:ODE}
    \begin{cases}
        \frac{d }{dt} X(x,t) &= f \left (X(x,t), t \right )\\
                    X(x, 0) &= x
    \end{cases}
\end{equation}
up to time $t=1$, for some initial condition $x \in \Omega_0$, yields a flow map $x \mapsto X^f(x, 1)$. The goal is to learn a Lipschitz continuous velocity field $f$, parameterized as a neural network, such that $x\sim \rho$ implies $X^f(x,1)\sim \pi$; in other words, the flow map pushes forward the source distribution to the target distribution.
    
Such dynamical representations of transport enjoy several desirable properties. Invertibility of $x\mapsto X^f(x,1)$ is guaranteed for any Lipschitz $f$ satisfying suitable boundary conditions, as one can solve \eqref{eq:ODE} backward in time. Therefore, and in contrast to other methods that directly parameterize the displacement (such as invertible neural networks \citep{InvNN}, normalizing flows \citep{NormalizingFlowIntro}, or other transport maps \citep{baptista23}), no further restrictions need to be imposed on the vector field $f$ that is to be learned. The density $\eta$ of the ODE state at time $t$, 
i.e., the density of $X(\cdot,t)_\sharp \rho$, 
can also easily be computed as it obeys the dynamics
$$\frac{d\log \eta(x,t)}{d t} = -\tr \left (\nabla_X f(X(x,t), t) \right ),$$ 
which is known as the \emph{instantaneous change of variables formula}; see \citet{NeuralODE}.



To understand the usefulness of neural ODEs in learning distributions, there are at least three natural questions to ask.
{First}, it is known that when the source and target measures are well-behaved (for example, both absolutely continuous with respect to the Lebesgue measure), there are in general infinitely many transport maps that push forward one measure onto the other. Moreover, even if we require the time-one flow map $x\mapsto X^f(x,1)$ to be a particular transport map $T$, there are in general still infinitely many velocity fields $f$ that realize $T$. It has thus been observed \citep{OTFlow,HowToTrain} that without any form of regularization, learned ODE trajectories connecting $x$ to $T(x)$ may be very irregular. It is therefore natural to ask how we can regularize the training objective to improve the training process.
{Second}, given a regularized training objective, we would like to characterize the structure of its minimizers, and in particular to quantify how well a neural network of a given size (e.g., width, depth, sparsity) can approximate the velocity field corresponding to an optimal solution. 
{Third}, we would like to know how to bound the distance between the target measure and the pushforward of the source measure under the ODE-induced flow map, when a neural network is used to approximate the velocity field $f$. 

Our work is the first attempt to address these questions in a unified way. One of our key goals is to obtain explicit rates for distribution approximation using regularized neural ODEs, linking properties of the source and target measures to bounds on the size of a deep neural network representation of the velocity field that achieves a given distributional approximation error. Furthermore, we aim to connect this distribution approximation error to the value of the regularized training objective.



\subsection{Contributions}

We summarize the main contributions of this paper as follows: 
\begin{itemize}

\item \textit{Realizing straight-line trajectories.} For a large class of transport maps $T$, we show that there exists a corresponding time-dependent velocity field $f(x,t)$ achieving the interpolation $(1-t)x + tT(x)$ and hence producing straight-line trajectories. In a Lagrangian frame, these trajectories have constant velocity and hence zero acceleration.

\item \textit{Regularity of the space-time domain.} We show that, under certain conditions, the space-time domain covered by such ODE trajectories is a Lipschitz domain, and thus admits suitable extensions that enable the application of existing neural network approximation results.

\item \textit{Minimum energy regularization.} We propose a new regularization scheme for ODE velocity fields, based on penalizing the average kinetic energy of trajectories. We characterize the minimizers of the resulting optimization problem (whose objective is the sum of a divergence and this penalty), and show that these time-dependent velocity fields take the straight-line interpolation form above.

\item \textit{Regularity of the ODE velocity field.} We derive explicit upper bounds for the $C^k$ norm of a straight-line velocity field that are polynomial in the $C^k$ norm of the corresponding transport map $T$. This development involves applying a multivariate Fa\`{a} di Bruno formula in Banach spaces, a procedure that may be of independent interest. 

\item \textit{Explicit links to $C^k$-smooth densities.} 
For specific transports $T$, in particular \emph{triangular} (Knothe--Rosenblatt) transport maps, we then construct upper bounds on the $C^k$ norm of the corresponding {straight-line} velocity field that 
have polynomial dependence on
the $C^k$ norm of
the source and target densities. Along the way, we obtain an explicit upper bound for the $C^k$ norm of the Knothe--Rosenblatt map that depends polynomially on the $C^k$ norms 
of the source and target densities.


\item \textit{Distributional stability.} We relate the approximation error in the velocity field to the error in the distribution induced by the time-one flow map of the resulting ODE, in both Wasserstein distance and Kullback--Leibler (KL) divergence.

\item \textit{Neural network approximation and optimization.} Combining our analysis of the regularity of the velocity field with the preceding stability results, we show that approximation of the target distribution to any desired accuracy $\eps > 0$, in Wasserstein distance or KL divergence, can be achieved by a deep neural network representation of the velocity field whose depth, width, and sparsity are bounded explicitly in terms of $\eps$ and 
the smoothness and dimension of the source and target densities.
%
We then provide guarantees for minimizers of the regularized optimization problem over such neural network classes, showing that there exist velocity fields that render the objective---and hence both the distribution approximation error and the departure from minimum kinetic energy---arbitrarily small.




\end{itemize}

\subsection{Related work}


Several papers have studied the approximation power of discrete normalizing flows. \citet{DisAppro2} investigate basic flow structures (e.g., planar flows, radial flows, Sylvester flows, Householder flows) for $L^1$ approximation of a target density on $\mathbb{R}^d$. The authors establish a universal approximation result for $d=1$, but show partially negative results for $d>1$: there exist distributions that cannot be exactly coupled by such flows, and there are other distributions for which accurate approximation requires composing together a prohibitive number of layers. \citet{DisAppro1}, on the other hand, show that normalizing flows based on the so-called ``affine coupling'' construction are universal approximators of diffeomorphisms, and consequently that their pushforward distributions converge weakly to any desired target as the complexity of the flow increases in suitable way. Additional universal approximation results have been developed for more specific flow architectures \citep{NeuralAutoFlow}. However, none of these works characterizes the \emph{rate} of convergence of the approximation---i.e., how the distribution approximation error that can be achieved scales with the size of the model.

For triangular (Knothe--Rosenblatt) transport maps on $[0,1]^d$, \citet{ZM1} develop a complete approximation theory under the assumption of analytic source and target densities, for neural network or sparse polynomial representations of the map. These results encompass approximation of the maps themselves, but also distribution approximation via the pushforward of a given source distribution by the approximate map. \citet{ZM2} extend this analysis to triangular transport maps in infinite dimensions, i.e., on $[0,1]^\infty$. \citet{baptista2023approximation} provide a general framework for analyzing the error of variational approximations of distributions realized using transport maps, and shows how this framework can be applied in specific situations to derive approximation rates.

The preceding works considered approximation of the transport map itself, i.e., direct representations of the `displacement' of points from the support of the source to the support of the target. In contrast, other recent efforts have addressed approximation issues in \textit{dynamic} representations of transport. Most of these have focused on neural ODEs. 
\citet{SupApproximation} show that neural ODEs are universal approximators of smooth diffeomorphisms on $\mathbb{R}^d$ in appropriate Sobolev norms. \citet{DynamicalSystem} adapt ideas from dynamical systems to show that neural ODEs are universal approximators of continuous functions from $\mathbb{R}^d$ to $\mathbb{R}^m$ (hence, not only diffeomorphisms) in a $L^2$ sense, for $d \geq 2$. Yet these universal approximation results do not characterize approximation \emph{rates}, e.g., relate bounds on function approximation error to the size of the network representing the velocity field. Moreover, we note that universal {function} approximation results are not necessarily relevant to {distribution} learning: that is, universal approximation of distributions does not require universal function approximation.


    

\citet{NeuralODE-Control} study distribution approximation with neural ODEs, and prove universal approximation for certain target distributions in Wasserstein-1 distance, using an approach that is discrete and constructive. Specifically, they analyze neural ODE-type models from a controllability perspective, explicitly constructing piecewie constant approximations of the target density using a neural network velocity field with ReLU activations. Their analysis does not consider higher-order smoothness, however, and their velocity construction is different from that considered here. In subsequent work, \citet{alvarez2024constructive} show that ReLU velocity fields chosen to be piecewise constant in time can approximate a target distribution arbitrarily closely in relative entropy. 
 

As mentioned earlier, several other works identify neural ODE velocity fields via a \textit{regularized} training objective, i.e., by minimizing a linear combination of a statistical divergence (or negative log-likelihood) and some regularization term.
\citet{HowToTrain} argue that a good way of measuring the regularity of
the velocity field is through the ``acceleration'' experienced by a ``particle'' $X^f(x,t)$ starting at some point $x$. This acceleration is the total derivative of $f$ in time: 
\begin{equation}\label{eq:straightlinereg}
  \frac{ D f(X,t)}{D t} = \nabla_X f(X,t) \cdot \frac{ \partial X}{ \partial t} + \frac{\partial f(X,t)}{ \partial t} = \nabla_X f(X,t) \cdot f(X,t) + \frac{\partial f(X,t)}{\partial t}.
\end{equation}
When this term is zero, particle trajectories will be straight lines. Since the Jacobian matrix $\nabla_X f(X,t)$ is
in general not easily accessible, \citet{HowToTrain} choose to
implicitly penalize this term by penalizing $|f(X,t)|^2$ and the
Frobenius norm $\|\nabla_X f(X,t)\|_F^2$ instead. Similar to
\citet{ffjord}, stochastic methods are used  to
estimate $\|\nabla_X f\|_F^2$ and $\tr(\nabla_X f(X,t))$ in training, where the latter is used in the change of variables formula for the log-likelihood. In more recent work, \citet{OTFlow} propose a discretize-then-optimize
approach to training neural ODEs, where a ResNet structure is used to
implement the underlying neural network in the ODE. This approach
enables exact computation of the Jacobian matrix as well as its trace
from the recursive structure of the ResNet. Then, automatic
differentiation is used to update the parameters in the neural
network, instead of solving the adjoint equation as in \citet{ffjord}
and \citet{HowToTrain}. By adopting this discretize-then-optimize approach, we propose to penalize \eqref{eq:straightlinereg} directly. We will show that the velocity field $f$ making this term zero is the unique velocity field that yields the minimal kinetic energy among all velocity fields $f$ that produce the same transport map $T$ at time one; hence it is termed the \textit{minimal energy regularization}.


While the neural ODE framework learns the velocity field by minimizing the KL divergence from the target distribution to the pushforward distribution of the source under ODE flow maps, another recent line of research  aims to specify the velocity field \textit{a priori} using conditional expectations and to learn the velocity field directly via least-squares regression \citep{StochasticInterpolant1, StochasticInterpoalnt2, FlowMatching, RectifiedFlow}. However, the same velocity field approximation questions and distributional stability questions are present in that setting as well. We note that the approximation results we develop in this work are independent of training scheme, and in particular, the straight-line velocity fields we analyze here are central to the rectified flows proposed in \citet{RectifiedFlow}.




        
        

Indeed, there are a variety of  related \textit{distributional stability} results in recent literature (cf.\ Section~\ref{sec:DistApproximation}), addressing the question of how error in the pushforward distribution under an ODE flow map (in some distance or divergence) is controlled by error in the velocity field (in some norm). \citet{BentonErrorBoundsFlowMatching} show that $W_2$ error in distribution is controlled by $L^2$ approximation error of the true velocity field and the time-averaged spatial Lipschitz constant of the approximate flow. In a study of the convergence of continuous normalizing flows, \citet{GaoConvergenceofContinusNormalizingFlows} obtain a stability result almost exactly the same as that in \citet{BentonErrorBoundsFlowMatching}, where $W_2$ error in distribution is controlled by $L^2$ error in velocity field times the exponential of spatial Lipschitz constant of the velocity field.
    \citet{LiProbabilisticODEConvergence} analyze a discrete-time version of the probability flow ODE, where TV error in distribution is bounded by terms involving the $L^2$ error in the score function and in its Jacobian. Again in the setting of probability flow ODEs, \citet{HuangProbabilisticODEConvergence} start in continuous time and then considers discretization using a Runge--Kutta scheme. At the continuous level, it is shown that TV error in distribution is controlled by 
    $L^2$ error in the approximation of the score function and
    the first and second
derivatives of the estimated score; at the discrete level, it is shown that the $p$-th order integrator also converges under an additional assumption that the estimated score function’s first $p + 1$ derivatives are bounded. In our work, we show that $W^p$ error in the distribution, for $p \in [1, \infty]$, is controlled by the space-time $L^\infty$ (i.e., $C^0$) error and spatial Lipschitz constant of $f$, on compact domains; these are further related to properties of the densities. In addition, we obtain that distribution approximation error in KL is controlled by the $C^1$ norm of $f$, again on compact domains.

        
        


There are also some results linking properties of the velocity field (e.g., Lipschitz constant in space or time) to properties of the underlying densities. In \citet{BentonErrorBoundsFlowMatching}, the time-averaged spatial Lipschitz constant is related to assumptions on the regularity of all the intermediate distributions between $t=0$ and $t=1$, along with some Gaussian smoothing; an upper bound is obtained that depends on properties of the chosen interpolant. We note that their regularity assumption is rather different than the $C^k$ smoothness we assume here, as they do not consider higher-order smoothness of the velocity field. \citet{GaoConvergenceofContinusNormalizingFlows} focus on flow matching with linear interpolation. It is shown that the Lipschitz constant of the target velocity field in both the space and time variables is bounded under assumptions of log-concavity/convexity of distributions and boundedness of the domain. In addition, they show that the velocity field itself grows at most linearly with respect to the space variable. No Gaussian smoothing is used in their setting, but they require an early stopping before reaching time $t = 1$. Both \citet{BentonErrorBoundsFlowMatching} and \citet{GaoConvergenceofContinusNormalizingFlows} are specific to certain stochastic interpolants, which are different than the straight-line ansatz we analyze here. Also, they do not consider higher-order smoothness of the velocity field or derive upper bounds that are explicit in the densities.


With regard to neural network approximation results, \citet{GaoConvergenceofContinusNormalizingFlows}
    construct neural network classes that capture the  Lipschitz properties of the velocity field, and derive rates of approximation in the $L^\infty$ sense. Our companion paper \citep{StatisticalNODE} derives explicit neural network approximation rates for general $C^k$ velocity fields, but its main focus is on statistical finite sample guarantees for neural ODEs trained through likelihood maximization, different from our focus here.


There are also stochastic differential equation (SDE) and specifically neural SDE methods for distribution learning \citep{NeuralSDE, SongDiffusionModels}. However, they are rather different than the deterministic ODE approach, and again are not the focus of this work.

\section{Preliminaries}
\subsection{Notation and definitions}
 {\bf ODEs and flow maps.}
 We write $X(x,t)$ or $X_f(x,t)$ for the solution of \eqref{eq:ODE}
 with initial condition $x$ at time $t=0$, i.e.,
  \begin{equation}\label{eq:flowmap}
    X_f(x,t) = x + \int_0^tf(X_f(x,s),s)ds.
  \end{equation}
  Given an initial distribution $\pi_0$, we write $\pi_t$
  or $\pi_{f,t}$ for the 
  pushforward measure $X_f(\cdot,t)_\sharp \pi_0$ and $\pi(x,t)$ or $\pi_{f}(x,t)$ for the corresponding density. 

\medskip

  \noindent {\bf Vectors and multiindices.} 
  For $x=(x_1,\dots,x_d)^\top\in\R^d$, $|x|$ is the Euclidean norm.
  With $\N=\{0,1,2,\dots\}$, we denote multiindices by bold letters
  such as $\bsv = (v_1,v_2,\dots,v_d) \in \mathbb{N}^d$, and we use the
  standard multiindex notations $|\bsv| = \sum_{i = 1}^d v_i$ and
  $\bsv! = \prod_{i=1}^d(v_i!)$.
  Additionally, $x^\bsv = \prod_{i=1}^d x_i^{v_i}$ and 
  $x_{[k]} \coloneqq (x_1,\dots,x_k)\in\mathbb{R}^k$ for all $k\leq d$.
  For two multiindices $\bsv$, $\bsw\in\N^d$, $\bsw \prec \bsv$ if and
  only if one of the following holds: \textit{(i)} $|\bsw| < |\bsv|$,
  \textit{(ii)} $|\bsw| = |\bsv|$ and there exists a $k < d$ such
  that $w_1 = v_1$, \dots, $w_k = v_k$, but $w_{k+1} < v_{k+1}$.

\medskip

  \noindent {\bf Derivatives.} For $f\in C^1(\R^d,\R^m)$,
  $\nabla f:\R^d\to\R^{d\times m}$ is the gradient. In case $f$
  depends on multiple variables, we write, for example,
  $\nabla_xf(x,t)$. For a multiindex
  $\bsv = (v_1,v_2,\dots,v_d) \in \mathbb{N}^d$, where
  $\N=\{0,1,2,\dots\}$, we write
  $D^\bsv f(x)= \frac{\partial^{|\bsv|}}{\partial x_1^{v_1} \dots 
    \partial x_d^{v_d}} f(x)$ for the partial derivative and similarly to the notation above, $D_x^\bsv f(x,t)$.

\medskip

  \noindent {\bf Function spaces.}
  For two Banach spaces $X$, $Y$ 
  and $n\in\mathbb{N}$ we denote by
  $\cL^n(X;Y)$
  the space of all $n$-linear maps from $X^n\to Y$, and by
    $\cL^n_{\sym}(X;Y)$ the space of all symmetric $n$-linear maps
    from $X^n\to Y$ (i.e., $A\in\cL^n_{\sym}$ iff
    $A(x_{\sigma(1)},\dots,x_{\sigma(n)})$ is independent of the
    permutation $\sigma$ of $\{1,\dots,n\}$). The norms
  on these spaces are defined as
  \begin{equation*}
    \|A\|_{\cL^n(X;Y)} \coloneqq \sup_{\|x_i\|_X\le 1} \|A(x_1,\dots,x_n)\|_Y,\qquad
    \|A\|_{\cL^n_{\sym}(X;Y)} \coloneqq \sup_{\|x\|_X\le 1} \|A(x,\dots,x)\|_Y.
  \end{equation*}
  We recall that if $A\in\cL^n_{\sym}(X;Y)$ and $B\in\cL^n(X;Y)$
    such that $A(x^n)=B(x^n)$ for all $x\in X$, then, see e.g.,
    \cite[14.13]{chae}, 
  \begin{equation}\label{eq:symnorm}
    \|A\|_{\cL^n_{\sym}(X;Y)}
    \le \|B\|_{\cL^n(X;Y)}\le\exp(n)\|A\|_{\cL^n_{\sym}(X;Y)}.
  \end{equation}

  Recall that for $f\in C^k(X,Y)$, the $k$-th Fr\'echet derivative
  $D^kf(x)$ of $f$ at $x\in X$ belongs to $\cL^k_{\sym} (X;Y)$. For
  $S\subseteq X$ open and $f\in C^k(S,Y)$ we write
  \begin{equation}\label{eq:Ck}
    \|f\|_{C^k(S)} \coloneqq \sup_{n\le k}\sup_{x\in S}\|D^n f(x)\|_{\cL^n(X;Y)}.
  \end{equation}

%

  For $X=\R^d$ and $S\subseteq X$, we use the usual notation
  $W^{k,p}(S)$, $k\in\N$, $p\in [1,\infty]$, to denote functions with
  weak derivative up to order $k$ belonging to $L^p(S)$. As a norm,
  we will use
  \begin{equation*}
    \|f\|_{W^{k, p}(S)} =
    \begin{cases}
      (\sum_{|\bsalpha|\leq k} \|D^\bsalpha f\|^p_{L^p(S)})^\frac{1}{p} & \text{if $ 1\leq p < \infty$}\\
      \max_{|\bsalpha|\leq k} \|D^\bsalpha f\|_{L^\infty(S)}& \text{if
        $ p = \infty$}.
    \end{cases}
  \end{equation*}

\noindent{\bf Divergences between distributions.} 
Let $(\Omega, \mathcal{F}, \mu)$ be a probability space. For two
probability measures $\rho$ and $\pi$ such that $\rho\ll\mu$, $\pi\ll\mu$, the information divergences we consider are the following:
\begin{itemize}
\item KL (Kullback--Leibler) divergence: Assuming also $\rho\ll\pi$, we define
  $\text{KL}(\rho, \pi) = \int_\Omega
  \log \frac{d\rho}{d\pi}(x) \rho(dx)$.
    
\item If $(\Omega, m)$ is also a metric space, then
  the Wasserstein distance of order $p$ is defined as:
  $$W_p(\rho, \pi) = \left(\inf_{\gamma\in\Gamma(\rho, \pi)}\int_{\Omega\times\Omega} m(x,
    y)^p\gamma(dxdy)\right)^\frac{1}{p},$$ where $\Gamma(\rho, \pi)$ is
  the set of all measures on $\Omega \times \Omega$ with marginals
  $\rho$ and $\pi$. In $\mathbb{R}^d$, this is simply
  $W_p(\rho, \pi) = \left(\inf_{\gamma\in\Gamma(\rho,
      \pi)}\int_{\mathbb{R}^d\times\mathbb{R}^d}|x-y|^p\gamma(dxdy)\right)^\frac{1}{p}$.
    
\end{itemize}

\subsection{Problem setup}\label{sec:setup}
In the following we denote by $\pi$ a \emph{target measure} and
  by $\rho$ a \emph{source measure} on $\R^d$. Our general goal is to
  sample from the target. The source measure is an auxiliary measure
  that is easy to sample from, and may be chosen at will. Throughout we
  work under the following assumptions:


\begin{assumption}[compact support]\label{ass:dens1}
  With $\Omega_0 \coloneqq {\rm supp}(\pi)$ and
  $\Omega_1 \coloneqq {\rm supp}(\rho)$, it holds that $\Omega_0$,
  $\Omega_1\subseteq\R^d$ are compact and convex sets. Both $\rho$
  and $\pi$ are absolutely continuous with respect to the Lebesgue
  measure.
\end{assumption}

By abuse of notation, we denote the (Lebesgue-) densities of $\rho$
and $\pi$ by $\rho(x)$ and $\pi(x)$, respectively.
  
\begin{assumption}[regularity]\label{ass:dens2}
  There exist constants $L_1 > 1$, $L_2>0$ and $k\in\N, k \geq 2$ such that
  $\|\pi\|_{C^k(\Omega_0)} \leq L_1$, $\|\rho\|_{C^k(\Omega_1)} \leq L_1$ and  $\inf_{x\in\Omega_0}\pi(x) \geq L_2$, $\inf_{x\in\Omega_1}\rho(x) \geq L_2$. 
\end{assumption}



We consider two types of problems:
\begin{itemize}
\item[P1] The target measure $\pi$ is known through a collection of
  iid samples. This is the problem considered in, e.g., \citet{ffjord,HowToTrain,OTFlow}. The goal is
  to 
  learn 
  a velocity field $f$ in \eqref{eq:ODE}
  such that with initial distribution chosen to be the target, $\pi_{0} = \pi$, the time-one distribution satisfies $\pi_{f,1}(\cdot) = X_f(\cdot,1)_\sharp \pi \approx \rho$. Since the flow map $x \mapsto X_f(x,1)$ is by construction invertible (and its inverse can be evaluated by solving \eqref{eq:ODE} backwards in time), one can draw new samples from the source measure and apply the inverse of the flow map to generate (approximate) samples from $\pi$. The learned flow map can also be used, without inversion, to estimate the density of $\pi$.
  

\item[P2] The target measure is known up to a normalizing constant;
  that is, we can evaluate the unnormalized target density $\Tilde{\pi}$. This setting is ubiquitous in Bayesian statistics, since the posterior normalizing constant of a Bayesian model is usually unavailable. This problem is in the reverse direction of the previous problem \citep{measure-transport,PlanarFlow,moselhy2012}; that is, we choose the initial distribution to be the source distribution, $\pi_0 = \rho$, and learn a velocity field $f$ such that $\pi_{f,1}(\cdot) = X_f(\cdot,1)_\sharp \rho \approx \pi$. 
    
\end{itemize}


From the approximation and algorithmic perspectives, there is no essential difference between problems P1 and P2 above. In both cases, algorithms for learning the velocity field $f$ require: \textit{(i)} a sample from the chosen initial distribution and \textit{(ii)} the ability to evaluate the \emph{desired} time-one density up to a normalizing constant. For the rest of the paper, we will thus adopt the setting of P1 (with initial distribution for the ODE system chosen to be the target $\pi$). Our results can be translated to P2 simply by exchanging $\pi$ and $\rho$. 

The objective functional considered in this work takes the following form: 
\begin{equation}\label{Jobjective}
J(f) = \mathcal{D}(X_f(\cdot, 1)_\sharp\pi, \rho) + R(f).
\end{equation}
The first part of the objective is an information divergence between two probability distributions (for example, KL, Wasserstein, etc.). The second part is a regularization term that follows from the discussion in Section \ref{sec:intro}: we would like the trajectory of the ODE, starting from any initial condition, to be a straight line with constant velocity. In other words, we would like the acceleration in a Lagrangian frame, $d f (X(x,t), t) / d t$, to be zero for all $x \in \Omega_0$ and $t \in [0,1]$. To this end, we integrate the squared acceleration from \eqref{eq:straightlinereg} along the trajectory of a particle $x\in\Omega_0$:
\begin{equation}
R(x,t) = \int_0^t \left \vert \left ( \nabla_X f(X(x,s),s) \right) f(X(x,s),s) +
\partial_s f(X,s) \right \vert^2 ds,    
\end{equation}

and 
\begin{equation}\label{eq:R}
R(f) = \int_{\Omega_0}\int_0^1 \left \vert \left ( \nabla_X f(X(x,s),s) \right) f(X(x,s),s) +
\partial_s f(X,s) \right \vert^2 dsdx      
\end{equation}

We comment here that while our theoretical analysis works for general divergence $\mathcal{D}$, KL-divergence is the most common objective used in practice. For this purpose, we derived the training algorithm for it in Appendix \ref{app:training}.

By the Picard-–Lindel\"{o}f theorem \citep{ArnoldODE}, existence and uniqueness of solutions to the ODE \eqref{eq:ODE} requires that the velocity field $f(x,t)$ be continuous in $t$ and Lipschitz continuous in $x$. Therefore, searching over the space of functions 
  \begin{equation}\label{eq:V}
    \mathcal{V} = \setc{f: \mathbb{R}^d\times[0,1]\rightarrow\mathbb{R}^d}{f \ \text{is Lipschitz continuous in } x, \text{continuous in }t }, \end{equation}
we have the following optimization problem:
  \begin{equation*} \label{eq:OP} \tag{OP}
    \begin{aligned}
      & \underset{f\in\mathcal{V}}{\text{minimize}}
      & & J(f)\\
      & \text{with}
      & & R(f) \text{ defined in } \eqref{eq:R}
    \end{aligned}
  \end{equation*}

\begin{remark}
  In practice, the conditions of the Picard–Lindel\"{o}f theorem will
  always be satisfied for a neural network of finite size with Lipschitz activation functions. In particular, these conditions hold true for ReLU networks, which is what we consider in our theoretical analysis.
\end{remark}

\section{Existence and structure of minimizers}\label{sec:minimizers}


The objective $J(f)$ is nonnegative, since it is the sum of an information divergence and a nonnegative regularizer.  Moreover, as we will see, the optimal solution will make both terms in the objective $J(f)$ zero under our assumptions. 
First, we state necessary and sufficient conditions on a transport map $T$ such that there exists a velocity field $f$ whose time-one flow map yields straight-line trajectories from $x$ to $T(x)$.

\subsection{Existence}\label{subsec:existence}
\begin{lemma}\label{lemma:Tinjective}
  Let $\Omega_0\subseteq\R^d$ be convex and
    $T\in C^1(\Omega_0,\R^d)$ such that $\det \nabla_x T(x)\neq 0$ for
    all $x\in\Omega_0$.  Then $T$ is injective.  
\end{lemma}
\begin{proof}[Proof of Lemma \ref{lemma:Tinjective}]
  Assume not. Then there exist $x$, $y\in\Omega_0$ such that
    $x \neq y$ and $T(x) = T(y)$. For $s\in [0,1]$ set
    $f(s)  \coloneqq  T((1-s)x + sy)$. Since $f(0) = f(1)$, by the mean
    value theorem there exists $s\in (0,1)$ such that
  $f'(s) = 0$. Then $f'(s) = \nabla_x T((1-s)x + sy)(y-x) =
  0$. Since $v = y-x \neq 0$, we have
    $\det(\nabla_x T((1-s)x + sy))=0$, which is a contradiction.
\end{proof}	

Denote in the following, for $x\in\Omega_0\subset\R^d$,
  $t\in [0,1]$, and a map $T:\Omega_0\to\R^d$,
  \begin{equation}\label{eq:Tt}
    T_t(x)  \coloneqq  (1-t)x + tT(x),
  \end{equation}
  i.e., $[0,1]\ni t\mapsto (T_t(x),t)$ parameterizes the straight line of constant velocity
  between the points $(x,0)$ and $(T(x),1)$ in $\R^d\times [0,1]$. 
  We
  refer to $t\mapsto T_t(x)$ as the \emph{displacement interpolation}
  of $T$. We now investigate under which conditions these lines do not
  cross for different $x$, which is necessary for $T_t(x)$ to
  be expressible as a flow $X(x,t)$ solving \eqref{eq:ODE} for a
  certain $f$. In other words, we check under what conditions the map
  $\Omega_0\ni x\mapsto T_t(x)$ is injective for all $t\in [0,1]$. To
  state the following lemma, for $A\in\R^{d\times d}$ we let $\sigma(A) \coloneqq \setc{\lambda\in\R}{\det(A-\lambda I)=0}$ denote its spectrum.

  \begin{assumption}\label{ass:T}
    It holds that $T\in C^1(\Omega_0,\R^d)$ and
    \begin{equation}\label{eq:spectrum}
      (-\infty, 0]\cap \sigma(\nabla_x T(x)) = \emptyset\qquad
      \forall\,x\in\Omega_0.
    \end{equation}
  \end{assumption}



\begin{lemma}\label{lemma:spectrum}
  Let $\Omega_0\subseteq\R^d$ be convex and
    $T\in C^1(\Omega_0,\R^d)$. Then $\det(\nabla_x T_t(x))>0$ for all
    $x\in\Omega_0$ and all $t\in[0,1]$, iff \eqref{eq:spectrum} holds.
  
\end{lemma}
\begin{proof}[Proof of Lemma \ref{lemma:spectrum}]
  Since $\nabla_x T_t(x) = (1-t)I + t\nabla_x T(x)$, the map
    $t\mapsto\det(\nabla_x T_t(x))\in\mathbb{R}$ is continuous.
    Because of $\nabla_x T_0(x) = I$, to prove the lemma it is
  sufficient to show that for every $x\in\Omega_0$,
  $\sigma(\nabla_x T(x)) \cap (-\infty, 0] = \emptyset$ iff
  $\det(\nabla_x T_t(x))\neq 0$ for all $t\in[0,1]$.

  Fix $x\in\Omega_0$.  Assume for contradiction that for some
  $t\in[0,1]$, we have $\det(\nabla_x T_t(x)) = 0$. Then there
    exists $v\neq 0$ such that $\nabla_x T_t(x)v = 0\in\R^d$. Thus
  $\nabla_x T(x)v = -\frac{1-t}{t}v$ and hence
  $-\frac{1-t}{t}\in (-\infty,0]$ is an eigenvalue of
  $\nabla_x T(x)$. The reverse implication follows similarly.
    Assume that $s\in \sigma(\nabla_x T(x))\cap (-\infty,0]$.  Then
    there exists $v\neq 0$ such that $\nabla_x T(x)v=sv$. Since
    $t\mapsto -\frac{1-t}{t}:(0,1]\to (-\infty,0]$ is bijective, we
    can find $t\in (0,1]$ such that $\nabla_x T(x)v=-\frac{1-t}{t}v$,
    implying $v\in \ker(\nabla_xT_t(x))$ and thus
    $\det(\nabla_x T_t(x))=0$.
\end{proof}

Combining the previous two statements
establishes the existence of a velocity field such that the time-one
flow map of the ODE \eqref{eq:ODE} realizes the map
$x\mapsto T(x)$, and the ODE dynamics produce straight-line
trajectories of constant speed.  

\begin{theorem}\label{thm:f}
  Let $k\in\N$ and let $\Omega_0\subseteq\R^d$ be convex and
    compact. Assume that $T\in C^k(\Omega_0,\R^d)$ for some $k\geq 2$
    satisfies \eqref{eq:spectrum}.
    With $T_t$ in \eqref{eq:Tt} set
    \begin{equation}\label{eq:Omega}
      \Omega_{[0,1]}  \coloneqq \setc{(T_t(x),t)}{x\in\Omega_0,~t\in [0,1]}\subset\R^{d+1}.
    \end{equation}
    Then there exists a unique $f:\Omega_{[0,1]}\to\R^d$ such that the
    solution $X:\Omega_0\times [0,1]\to\R^d$ of the ODE \eqref{eq:ODE}
    satisfies $X(x,t)=T_t(x)$ for all $x\in\Omega_0$, $t\in [0,1]$.
    Moreover $\Omega_{[0,1]}$ is simply connected and
    $f\in C^k(\Omega_{[0,1]}^\circ)$.
\end{theorem}

A key element of the proof of Theorem ~\ref{thm:f} is that the velocity field $f$ corresponding to the displacement interpolation $T_t$ can be defined implicitly in terms of $T$: $f\left ( (1-t)x + t T(x), t \right ) = T(x) - x$.

\begin{proof}[Proof of Theorem \ref{thm:f}]
      By \eqref{eq:ODE} and because $X(x,t)=T_t(x)=(1-t)x+tT(x)$,
        we have for $x\in\Omega_0$ and $t\in [0,1]$
        \begin{equation}\label{eq:Tx-x}
          T(x)-x = \frac{d}{dt}X(x,t)=f(X(x,t),t)=f(T_t(x),t).
        \end{equation}
        By Lemma \ref{lemma:Tinjective} and Lemma \ref{lemma:spectrum},
        the map
        \begin{equation*}
          (x,t)\mapsto G(x,t) \coloneqq (T_t(x),t)\in\Omega
        \end{equation*}
        is injective on $\Omega_0\times [0,1]$.  Thus \eqref{eq:Tx-x}
        uniquely defines $f$ at each point $G(x,t)\in\Omega$. By
        construction, this $f$ yields a flow map $X$ as in
        \eqref{eq:ODE} satisfying $X(x,t)=T_t(x)$.

      The map $G:\Omega_0\times [0,1]\to\Omega$ is a continuous
        bijection, and since $\Omega_0\times [0,1]\subseteq\R^{d+1}$
        is compact, $G^{-1}:\Omega\to \Omega_0\times [0,1]$ is also a
        continuous bijection. Since $\Omega_0\times [0,1]$ is a convex
        set, it is simply connected. Hence, the homotopy equivalent
        set $\Omega_{[0,1]}$ must also be simply connected. Moreover, the
        interior $\Omega^\circ$ of $\Omega$ is the image of
        $\Omega_0^\circ\times (0,1)$ under $G$.
        
        It remains to show $f\in C^k(\Omega_{[0,1]}^\circ)$. Fix
        $x\in \Omega_0^\circ$ and $t\in (0,1)$. Then
        \begin{equation*}
          \nabla_{(x,t)}G(x,t)=\begin{pmatrix}
            \nabla_x T_t(x) & T(x) - x \\
            0 & 1
          \end{pmatrix}\in\R^{(d+1)\times (d+1)},
        \end{equation*}
        and this matrix is regular by Lemma
        \ref{lemma:spectrum}. Since $G\in C^k(\Omega_0\times [0,1])$,
        the inverse function theorem (see, e.g., \cite{CalculusOnManifolds}[{Theorem~2.11}]) implies
        that $G^{-1}$ locally belongs to $C^k$ in a neighbourhood of
        $G(x,t)$. Since $x\in\Omega_0^\circ$ and $t\in (0,1)$ were
        arbitrary, we have $G^{-1}\in
        C^k(\Omega_{[0,1]}^\circ,\R^{d+1})$. Denote $G^{-1}=(F,E)$ such that
        $F:\Omega_{[0,1]}\to\R^d$ corresponds to the first $d$ components of
        $G^{-1}$.  By \eqref{eq:Tx-x}, for all $(y,s)\in\Omega_{[0,1]}^\circ$
        it holds that $f(y,s)=T(F(y,s))-F(y,s)$, so that $f$ belongs to
        $C^k(\Omega_{[0,1]}^\circ)$ as a composition of two $C^k$ functions.
    \end{proof}

      \begin{remark}
        Note that $f\in C^k(\Omega_{[0,1]}^\circ)$ means only that $f$ is
        $C^k$ on the interior of $\Omega_{[0,1]}$. To show that the
        derivatives are well-defined on the boundary of $\Omega_{[0,1]}$ and that $f$ can be extended to a $C^k$ function outside
        of $\Omega_{[0,1]}$, certain regularity conditions of the domain are required, which will be discussed in later parts of this section.
      \end{remark}

    \subsection{Properties of $\Omega_{[0,1]}$}\label{subsection:propertiesofOmega}

   The set $\Omega_{[0,1]}$ is simply connected, but unlike
      $\Omega_0\times [0,1]$, it need not be convex:

\begin{example}[Rotation]
  Let $\Omega_0=\setl{x\in\R^2}{|x|\le 1}$ be the unit disc and let
  $T:\R^2\to\R^2$ be the rotation by $\bsalpha\in [0,2\pi)$ around
  $0\in\R^2$. Then
  \begin{equation}
    \nabla_x T(x) = \begin{pmatrix}
      \cos(\bsalpha) & -\sin(\bsalpha)\\
      \sin(\bsalpha) & \cos(\bsalpha)
    \end{pmatrix}.
  \end{equation}
  The spectrum of this matrix consists of the two values
  $\exp(\pm\ii\bsalpha)$, where $\ii$ denotes the imaginary root of
  $-1$. Thus \eqref{eq:spectrum} holds iff $\bsalpha\neq \pi$. If
  $\bsalpha=\pi$, then $T$ is the negative identity, and thus
  $T_{1/2}(x)=\frac{1}{2}x - \frac{1}{2}x=0$ for all $x\in\Omega_0$,
  so that the all straight lines connecting $x$ and $T(x)$ for
  $x\in\Omega_0$, meet at $t=\frac{1}{2}$ in the midpoint $0$ of the
  disc.  For all $\bsalpha\in [0,2\pi)\backslash\{\pi\}$, by
  Theorem~\ref{thm:f} and with $\Omega_{[0,1]}$ as in \eqref{eq:Omega}, there
  exists a vector field $f\in C^\infty(\Omega_{[0,1]})$ such that
  $T_t(x)=X(x,t)$ for $X$ as in \eqref{eq:ODE}. One can check that
  \begin{equation*}
    \Omega_{[0,1]} = \setc{(x,t)\in\R^2\times [0,1]}{|x|\le
      \sqrt{\sin\left(\frac{\pi}{2}-\frac{\bsalpha}{2}\right)^2+\left[t\cos\left(\frac{\pi}{2}-\frac{\bsalpha}{2}\right)+(1-t)\cos\left(\frac{\pi}{2}+\frac{\bsalpha}{2}\right)\right]^2}
    },
  \end{equation*}
  which is convex
  if $\bsalpha=0$ and nonconvex for all $\bsalpha\in (0,2\pi)$.
\end{example}

To approximate the velocity field $f$ from Theorem ~\ref{thm:f} with a
  neural network, we also need to understand the regularity of the
domain $\Omega_{[0,1]}$
on which $f$ is defined. 
As we will see, $\Omega_{[0,1]}$ is a Lipschitz domain.

\begin{definition}
  A bounded domain
  $\Omega$ is called a Lipschitz domain if there exist numbers $\delta
  > 0$, $M > 0$,
  $J\in\N$, and a finite cover of open sets
  $\{U_j\}_{j=1}^J$ of $\partial\Omega$ such that:
  \begin{itemize}
  \item For every pair of points $x_1$, $x_2
    \in\Omega$ such that $|x_1 - x_2| < \delta$ and $\text{dist}(x_i,
    \partial\Omega) < \delta$, $i = 1$, $2$, there exists an index
    $j$ such that $x_i \in U_j$ , $i = 1$, $2$, and $\text{dist}(x_i,
    \partial U_j) > \delta$, $i = 1$, $2$.
  \item For each
    $j$ there exists some coordinate system
    $\{\zeta_{j,1},\dots ,\zeta_{j,d}\}$ in
    $U_j$ such that the set $\Omega \cap
    U_j$ consists of all points satisfying $\zeta_{j,d} \leq
    f_j(\zeta_{j,2},\dots ,\zeta_{j,d-1})$, where
    $f_j:\mathbb{R}^{d-1}\rightarrow\mathbb{R}$ is a Lipschitz
    function with Lipschitz constant $M$.
  \end{itemize}
\end{definition}


To show that $\Omega_{[0,1]}$ is a Lipschitz domain, we first need an auxiliary result, Theorem ~\ref{thm:LipTransformation} in Appendix~\ref{app:AuxResults}, establishing that the image of a Lipschitz domain under a sufficiently regular map remains a Lipschitz domain. We can then show the following:
\begin{theorem}\label{thm:LipDomain}
  Consider the setting of Theorem ~\ref{thm:f}.  Then
    $\Omega_{[0,1]}\subset\R^{d+1}$ in \eqref{eq:Omega} is a Lipschitz
    domain.
\end{theorem}

\begin{proof}[Proof of Theorem \ref{thm:LipDomain}]
  To show that $\Omega_{[0,1]}$ is a Lipschitz domain, first we observe that $\Omega_{[0,1]}$ is the image of $\Omega_0\times [0,1]$ under the map
    $(x,t)\rightarrow G(x,t)  \coloneqq  (tT(x) + (1-t)x, t)$ for
    $x\in\Omega_0,t\in[0,1]$. Since $\Omega_0\times [0,1]$ is a
    product of two convex sets, which is convex, Lemma
    \ref{lemma:convexLip} shows that the cylinder
    $\Omega_0\times [0,1]$ is a Lipschitz domain. To apply Theorem
    \ref{thm:LipTransformation}, we need to find a
    $C^1$-diffeomorphism from an open neighborhood $\mathcal{O}$ of
    $\Omega_0\times [0,1]$ onto its image. In the context of Theorem
    \ref{thm:f}, we have $\det(\nabla_x T_t(x)) > 0$ for all
    $(x,t)\in\Omega_0\times[0,1]$. Since $\Omega_0\times[0,1]$ is a
    compact set, the infimum of the continuous function
    $(x,t)\rightarrow\det(\nabla_x T_t(x))$ is achieved at some point
    in the set and thus we can conclude that
    $\inf_{(x,t)\in\Omega_0\times[0,1]}\det(\nabla_x T_t(x)) > 0$.

    On the other hand, since $T\in C^k(\Omega_0, \mathbb{R}^d)$ for
    some $k\geq 2$, it follows that
    $T \in W^{k, \infty}(\Omega_0, \mathbb{R}^d)$. By the extension
    theorem \ref{thm:functionExtension}, $T$ can be extended to a
    function $\tilde{T}\in W^{k, \infty}(\mathbb{R}^d,
    \mathbb{R}^d)$. Since $k \geq 2$, Sobolev embedding shows that
    $\tilde{T}\in C^1(\mathbb{R}^d, \mathbb{R}^d)$. Now consider the
    map $\tilde{T}_t(x) = t\tilde{T}(x) + (1-t)x$ for
    $(x,t)\in\mathbb{R}^{d+1}$. It is clear that $\tilde{T}_t(x)$ is
    $C^1$ in $(x,t)$ and also
    $\tilde{T}_t(x)|_{\Omega_0\times[0,1]} = T_t(x)$. By the
    continuity of determinant operator and
    $\inf_{(x,t)\in\Omega_0\times[0,1]}\det(\nabla_x T_t(x)) > 0$, it
    follows that there exists an open neighborhood
    $\mathcal{O}\subset\mathbb{R}^{d+1}$ of $\Omega_0\times[0,1]$ such
    that $\det(\nabla_x \tilde{T}_t(x)) > 0$ for all
    $x\in\mathcal{O}$. Without loss of generality, we can assume
    $\mathcal{O}$ is convex. This is because we can choose the
    neighborhood
    $\Omega_0\times[0,1]\cup
    \{B_\epsilon((x,t))|(x,t)\in\partial(\Omega_0\times[0,1])\}$, which
    is an open and convex set that can be made arbitrarily close to
    $\Omega_0\times[0,1]$ when $\epsilon\rightarrow 0$.

    Now, consider the extension of $G$,
    $\tilde{G}(x,t) = (t\tilde{T}(x) + (1-t)x, t)$ for
    $(x,t)\in\mathcal{O}$. We have
    \begin{equation*}
      \nabla_{(x,t)}\tilde{G}(x,t)=\begin{pmatrix}
        \nabla_x \tilde{T}_t(x) & \tilde{T}(x) - x\\
        0 & 1
      \end{pmatrix}\in\R^{(d+1)\times (d+1)}
    \end{equation*}
    is a regular matrix for fixed $(x,t)\in\mathcal{O}$. Then, the same
    arguments as in the proof of Theorem \ref{thm:f} show that
    $\tilde{G}(x,t)$ has a global inverse and $\tilde{G}^{-1}$ is
    $C^1$. Therefore, we have a $C^1$-diffeomorphism from
    $\mathcal{O}$ onto its image, and Theorem
    \ref{thm:LipTransformation} shows that $\Omega_{[0,1]} = \{(T_t(x),t)\}$
    for $x\in\Omega_0$, $t\in[0,1]$ is a Lipschitz domain.
\end{proof}

For Sobolev functions on Lipschitz domains, we have the following extension theorem:
\begin{theorem}[{\cite[Chap.\
    3]{SteinBook}}]\label{thm:functionExtension}
  Let $D \subset\mathbb{R}^d$ be a Lipschitz
  domain.\footnote{The result in \citet{SteinBook} is stated in
      terms of so-called ``minimally smooth domains,'' which is a
      generalization of the notion of Lipschitz domains.}
  Then there exists a linear operator $\mathcal{E}$ mapping functions
  on $D$ to functions on $\mathbb{R}^d$ with the following properties:
  \begin{itemize}
  \item $\mathcal{E}(f)|_D = f$, that is, $\mathcal{E}$ is an
    extension operator.
  \item $\mathcal{E}$ maps $W^{k,p}(D)$ continuously into
    $W^{k,p}(\mathbb{R}^d)$ for all $1\leq p\leq \infty$ and all
      nonnegative integer $k$.
      
  \end{itemize}
\end{theorem}


Combining Theorem ~\ref{thm:f} and \ref{thm:LipDomain}, the extension result in  Theorem ~\ref{thm:functionExtension} shows that the velocity field $f$ of Theorem ~\ref{thm:f} can be extended to a function in $W^{k, \infty}$ on all of $\mathbb{R}^{d+1}$.

\subsection{Regularized solutions}\label{sec:optsols}

In the theorem below, we show that the straight-line connections
between \jz{$x$} and $T(x)$, for a transport map $T$ that pushes
forward $\pi$ to $\rho$, yield the minimal average kinetic
energy, which is why we name the construction \textit{minimal energy regularization}. Here, the ``average kinetic energy'' is the squared magnitude of the ODE velocity averaged along trajectories, given a distribution $\pi$ on the initial condition. For any ODE velocity field $g(x,t)$, this quantity can be written in either Lagrangian or Eulerian frames as follows:
$$
\int_{\Omega_0} \int_0^1 \pi(x) \vert g(X_g(x,t),t) \vert^2 dt \, dx = \int_{\mathbb{R}^d} \int_0^1 \pi_{g,t}(x) \vert  g(x,t) \vert^2 dt \, dx. 
$$

  \begin{theorem}\label{thm:minimalenergy}
    Let $\Omega_0, \Omega_1\subseteq\R^d$ with $\Omega_0$  convex,
    and let $T\in C^1(\Omega_0,\Omega_1)$ satisfy
    \eqref{eq:spectrum}. Assume that $\pi$ and $\rho$ are
    probability measures on $\Omega_0$, $\Omega_1$, respectively, and that
    $T_\sharp\pi = \rho$. Then with
\begin{equation*}
  \mathcal{H} \coloneqq \setc{g\in \mathcal{V}}{X_g(\cdot,1)|_{\Omega_0}=T}
\end{equation*}
and $f$ from Theorem ~\ref{thm:f}, it holds that
\begin{equation*}
  f = \argmin_{g\in\mathcal{H}}\int_{\mathbb{R}^d}\int_0^1\pi_{g,t}(x)|g(x,t)|^2\dd t\dd x.
\end{equation*}

 \end{theorem}
\begin{proof}[Proof of Theorem \ref{thm:minimalenergy}]
  By Theorem \ref{thm:f}, we know the existence of velocity fields
  that realize these constructions. We then bound the
  average kinetic energy from below, using Lagrangian coordinates, as follows:
  \begin{align*}
    &\int_{\mathbb{R}^d} \int_0^1 \pi_{g,t}(x)|g(x,t)|^2dtdx = \int_{\Omega_0} \int_0^1 \pi_{g,0}(x)|g(X(x,t),t)|^2dtdx\\
    &=\int_{\Omega_0} \int_0^1 \pi(x)|\partial_t X(x,t)|^2dtdx \geq \int_{\Omega_0} \pi(x) \left ( \int_0^1|\partial_t X(x,t)|dt \right )^2dx 
    &\geq \int_{\Omega_0} \pi(x)|X(x,1)- X(x,0)|^2dx\\
    &= \int_{\Omega_0} \pi(x)|X(x,1)-x|^2dx = \int_{\Omega_0} \pi(x)|T(x) -x|^2dx,
  \end{align*}
   where the second inequality
  is due to Jensen's inequality, and equality holds iff
  $\partial_t X(x,t) = X(x,1) - X(x,0) = T(x) - x$. Then the optimal
  choice of $X$ is given by $X(x,t) = x + t(T(x) - x)$, which is
  exactly the displacement interpolation. As a result, the optimal
  choice for $f$ is given by $f(X(x,t), t) = T(x) - x$, which is the
  straight line construction from Theorem ~\ref{thm:f}.
\end{proof}

\begin{remark}
  This construction has important connections to the \emph{fluid dynamics} formulation of optimal transport \citep{OT-CFD}. Theorem ~\ref{thm:minimalenergy} shows that for a \emph{fixed} transport map $T$, the straight-line construction gives the minimal average kinetic energy. The \emph{optimal} transport map $T$ is then just the transport map
  $T$ that minimizes $\int_{\mathbb{R}^d} \pi(x)|T(x) -x|^2dx$, which
  is the $L^2$-Wasserstein distance.
\end{remark}

With this machinery developed, we are now ready to prove that under our assumptions on the measures $\pi$ and $\rho$, \eqref{eq:OP} admits optimal solutions that realize the displacement interpolation of  transport maps $T$ that push forward $\pi$ to $\rho$.

\begin{theorem}\label{thm:ExSol}
  Let $\pi$ and $\rho$ be measures supported on $\Omega_0$ and $\Omega_1$, respectively, and let Assumptions~\ref{ass:dens1} and \ref{ass:dens2} be satisfied. Then there exists at least one velocity field $f \in \mathcal{V}$
  that achieves the global minimum of zero in the optimization problem \eqref{eq:OP}. Moreover, all global minimizers of \eqref{eq:OP} take the form $f(y,t) = T(x) - x$, with $y = T_t(x)$, for some transport map $T$ such that $T_\sharp\pi = \rho$, where $(y,t)\in\Omega_{[0,1]} \coloneqq \setl{(T_t(x),t)}{x\in\Omega_0,\ t\in[0,1]} \subset
  \mathbb{R}^{d+1}$ and $T$ satisfies \eqref{eq:spectrum}.
\end{theorem}
To show the existence of a velocity field that achieves $J(f) = 0$, in the proof of Theorem ~\ref{thm:ExSol} we consider a velocity field that realizes the optimal transport map. 
\begin{proof}[Proof of Theorem \ref{thm:ExSol}]
  Clearly, the objective function is bounded from below by zero. We
  first show that the velocity field corresponding to the
  optimal transport map achieves this minimum value. The existence of optimal transport map (associated with quadratic cost) under our assumptions and the fact that the map can be written as the gradient of a convex potential $\phi$ are well-established results in the theory of optimal transport maps (see e.g., \cite{OptimalOldAndNew} and \cite{BrenierMap}). By Assumptions \ref{ass:dens1},\ref{ass:dens2} and Theorem  \ref{Thm:OptimalRegularity},the optimal transport map is given by $T(x) = \nabla\phi(x)$ for some $\phi\in C^{k+2}(\Omega_0)$ that is strictly convex. Therefore, $\nabla T(x)$ has real and nonnegative eigenvalues. 
  
  Since $\phi\in C^{k+2}(\Omega_0)$, the following Monge-Ampere equation is satisfied in the classical sense (\cite{BrenierMap}), : 
  $$\det(\nabla^2\phi(x)) = \frac{\pi(x)}{\rho(\nabla\phi(x))}, \forall x\in\Omega_0.$$
  
  Since the densities are both bounded away from zero, we can conclude from the Monge-Ampere equation that $\det \nabla T(x) = \det \nabla^2\phi(x) > 0, \forall x\in\Omega_0$. In particular, Assumption \ref{ass:T} is satisfied and Theorem \ref{thm:f} establishes the existence
of a velocity field such that the time-one flow map of the ODE
realizes this $T(x)$ and the ODE dynamics yield straight line trajectories.

Now, suppose that there is a continuous velocity field $f$ that
achieves zero loss in \eqref{eq:OBJ}. Since the densities are continuous and bounded from below by a constant, the expectation $\mathbb{E}_{\pi}[R(x,1)] = 0$
implies that $R(x,1) = 0\  \forall x \in \Omega_0$.  That is, along each
trajectory $X(x,t)$ starting from $x\in\Omega_0$, we have
$\frac{ d f( X(x,t),t )}{d t} = 0$, i.e., $f$ is
constant in time along each trajectory. In other words,
$f(X(x,t),t) = g(x)$ for some function $g$. 
Now consider the ODE $\frac{dX}{dt} = f(X,t) = g(x)$; the solution is $X(t) = g(x)t + C$, where $C$ is constant in $t$. To make the KL-divergence zero,
we must have $X(1) = T(x)$ for some transport map $T$ such that
$T_\sharp\pi = \rho$, and we also have $X(0) = x$ as the initial
condition. Solving the equations gives $g(x) = T(x) - x$. That is, the
velocity field must take the form $T(x) - x$ for some transport map
$T$.

\end{proof}

\section{
  Regularity of the velocity field $f$}\label{sec:regularity}
 As we have seen in Theorem \ref{thm:f}, for a transport
  $T:\Omega_0\to\Omega_1$ as in Assumption \ref{ass:T}, there exists a
  unique velocity field $f:\Omega_{[0,1]}\to\R^d$ such that
  $T(x)=X_f(x,1)$ for all $x\in\Omega_0$. This $f$ is the unique
  minimizer of the objective \eqref{eq:OBJ}. Furthermore we have an
  explicit formula for $f$: With
  $G:\Omega_0\times [0,1]\to\Omega_{[0,1]}$,
  $G(x,t)  \coloneqq  (tT(x) + (1-t)x, t)$ define $F:\Omega_{[0,1]}\to\Omega_0$
  as the first $d$ components of $G^{-1}$, then
  \begin{equation}\label{eq:fexpl}
    f(y,s) = T(F(y,s)) - F(y,s)\qquad\forall (y,s)\in\Omega_{[0,1]}.
  \end{equation}
  Based on \eqref{eq:fexpl}, in this section we investigate the
  regularity of the velocity field $f$.

  As we will see, $f$ inherits the regularity of $T$, in the sense
  that $T\in C^k(\Omega_0)$ implies $f\in
  C^k(\Omega_{[0,1]})$. Essentially, this follows by the inverse
  function theorem, which yields $G^{-1}\in C^k(\Omega_{[0,1]})$ so
  that $f$ in \eqref{eq:fexpl} is a composition of $C^k$
  functions. Since the approximability of $f$ by neural networks
  crucially depends on $\|f\|_{C^k(\Omega_{[0,1]})}$ (see Section
  \ref{sec:NN} ahead), we will carefully derive a precise bound on
  this norm. We proceed as follows: In Section \ref{sec:genreg}, we
  discuss regularity of $f$ for arbitrary transport maps
  $T$. Subsequently, in Section \ref{sec:trireg} we deepen the
  discussion in the special case of triangular transport maps (which
  yield triangular velocity fields $f$).

\subsection{General transports}\label{sec:genreg}
To bound the norm of $f$ in $C^k(\Omega_{[0,1]})$, our proof
  strategy is to first upper bound $\|F\|_{C^k(\Omega_{[0,1]})}$, and
  then use a version of Fa\'{a} di Bruno's formula to estimate the
  norm of the composition of $F$ with $T$.


  Since $F$ is obtained as the first $d$ components of the inverse
  $G:\Omega_{[0,1]}\to\Omega_0\times [0,1]$, we first provide some
  abstract results about 
  how to bound the $k$th derivative of an inverse function. We start
  by recalling Fa\'{a} di Bruno's formula in Banach spaces. For
  completeness we have added the proof in Appendix \ref{app:AuxResults}, but emphasize
  that the argument is the same as for real valued functions \cite{}. In the
  following, for a multilinear map $A\in L^n(X,Y)$ we write $A(x^n)$
  to denote $A(x,\dots,x)$.

  \begin{theorem}[Fa\'{a} di Bruno]\label{thm:FdB}
    Let $k\in\N$.
  Let $X$, $Y$ and $Z$ be three Banach spaces, and let
  $F\in C^k(X,Y)$ and $G\in C^k(Y,Z)$.

  Then for all $0\le n\le k$ and with
  $T_n := \setl{\bsalpha \in\N^n}{\sum_{j=1}^n j\alpha_j =
    n}$, 
  for all $x$, $h\in X$ the $n$th derivative
  $[D^n(G\circ F)](x)(h^n)\in Z$ of $G\circ F$ at $x$ evaluated at
  $h^n\in X^n$ equals
  \begin{equation*}
     \sum_{\bsalpha\in T_n}
    \frac{n!}{\bsalpha!}
    [D^{|\bsalpha|}G](F(x)) \Bigg(\smash[b]{\underbrace{\frac{[DF(x)](h)}{1!},\ldots, \frac{[DF(x)](h)}{1!}}_{\text{$\alpha_1$ times}}},\ldots,\underbrace{\frac{[D^{n}F(x)](h^n)}{n!},\ldots, \frac{[D^{n}F(x)](h^n)}{n!}}_{\text{$\alpha_n$ times}}\Bigg).
  \end{equation*}

\end{theorem}
Additionally we need the inverse function theorem, the proof of which can also be found in Appendix \ref{app:AuxResults}. 
\begin{theorem}[Inverse function theorem]\label{thm:invfunc}
  Let $k\ge 1$, let $X$, $Y$ be two Banach spaces, and let $F\in C^k(X,Y)$.
  At every $x\in X$ for which $DF(x)\in L^1(X,Y)$ is an isomorphism,
  there exists an open neighbourhood $O\subseteq Y$ of $F(x)$ and a function
  $G\in C^k(O,X)$ such that $F(G(y))=y$ for all $y\in O$.

  Moreover, for every $n\le k$ there exists a continuous function
  $C_n:\R_+^{n+1}\to \R_+$ (independent of $F$, $G$, $O$) such
  that for $y=F(x)$ with $x$ as above
  \begin{equation}\label{eq:Cn}
    \|D^n G(y)\|_{\cL^n_{\sym}(Y;X)} \le C_n(\|[DF(x)]^{-1}\|_{\cL^1_{\sym}(Y;X)}, \|DF(x)\|_{\cL^1_{\sym}(X;Y)}, \ldots, \|D^nF(x)\|_{\cL^n_{\sym}(X;Y)}).
  \end{equation}
\end{theorem}

We start by giving a bound on the derivatives of the composition of
functions.

\begin{corollary}\label{cor:composition}
  Let $S_1\subseteq X$, $S_2\subseteq Y$, $S_3\subseteq Z$ be three
  open subsets of the Banach spaces $X$, $Y$ and $Z$.
  Suppose $k\in\N$ and
  $F\in C^k(S_1,S_2)$ and $G\in C^k(S_2,S_3)$.

  Then $\|G\circ F\|_{C^k(S_1)}\le k^k\exp(k)\|G\|_{C^k(S_2)}
  \max\{1,\|F\|_{C^k(S_1)}\}^k$.
  \end{corollary}
  \begin{proof}
    By Theorem \ref{thm:FdB}, for all $x\in S_1$
    \begin{align}\label{eq:DnFG}
    \|D^n(G\circ F)](x)\|_{\cL^n_{\sym(X;Y)}}
    &\le
      \|G\|_{C^n(S_2)}
      \sum_{\bsalpha\in T_n}
      \frac{n!}{\bsalpha!}
      \prod_{j=1}^n \frac{\|F\|_{C^j(S_1)}^{\alpha_j}}{(j!)^{\alpha_j}}\nonumber\\
    &\le
      \|G\|_{C^n(S_2)}\max\{1,\|F\|_{C^j(S_1)}\}^n
      \sum_{\bsalpha\in T_n}
      \frac{n!}{\bsalpha!}\prod_{j=1}^n\frac{1}{(j!)^{\alpha_j}},
  \end{align}
  where we used $\sum_{j=1}^n \alpha_j\le n$ for all
  $\bsalpha\in T_n$, and that $\|F\|_{C^j(S_1)}\le\|F\|_{C^n(S_1)}$
  for all $j\le n$ by definition of the norm.
  By Lemma \ref{lemma:stirling}, the last sum in \eqref{eq:DnFG} is
  bounded by $n^n$. Finally \eqref{eq:symnorm} and the definition
  of the $C^n$-norm in \eqref{eq:Ck} imply claim.
\end{proof}

We next use Faa di Bruno's formula, to bound the derivatives of the
inverse $G^{-1}$ of a function $G$.
\begin{proposition}\label{prop:NormInverse}
  Consider the setting of Theorem \ref{thm:invfunc}. Let
  $S\subseteq X$ be open such that $DG(x)\in L(X,Y)$ is an isomorphism
  for each $x\in S$ and let there be a continuous function
  $F:G(S)\to X$ such that $F(G(y))=y$ for all $y\in
  G(S)$. Additionally assume that for some $\gamma>0$
  \begin{equation}\label{eq:gammaass}
    \sup_{x\in S}\|[DG(x)]^{-1}\|_{\cL^1(Y;X)} \le\gamma\qquad\text{and}\qquad
    \|G\|_{C^k(S)}\le\gamma.
  \end{equation}

  Then 
  \begin{equation}\label{eq:NormInverse}
    \|D^nF(y)\|_{\cL^n(Y,X)}\le \e^kk^{k^2}\gamma^{3k-2}
    \qquad\forall y\in G(S). 
  \end{equation}
\end{proposition}
\begin{proof}[Proof of Proposition \ref{prop:NormInverse}]
    We proceed by induction over $n$, and will show that for all
    $y\in G(S)$ and all $1\le n\le k$
   \begin{equation}\label{eq:indclaimNI}
      \|D^n F(y)\|_{\cL_{\sym^n(X;Y)}}\le n^{n^2}\gamma^{3n-2}.
    \end{equation}
    Then \eqref{eq:symnorm} implies the claim.

    For $n=1$,
    ${\rm Id} = D(G\circ F)(y)$ for all $y \in G(S)$. By the chain
    rule, $(D(G\circ F)(y)\circ (DF(y)) = {\rm Id}$, so that
    $DF(y) = [DG (F(y))]^{-1}$. Thus $\|DF(y)\|_{\cL^1(X;Y)}\le \gamma$ by
    \eqref{eq:gammaass}. This implies \eqref{eq:indclaimNI} for
    $n=1$.

    Now let $n \ge 1$. Then for any $y\in G(S)$ by \eqref{eq:Dng}
\begin{align*}
\|D^nF(y)\|_{\cL^n_{\sym(Y;X)}} \le &\|([DG](F(y)))^{-1}\|_{\cL^1(Y;X)}\\ &\cdot \left ( \sum_{\bsalpha\in \bar{T}_n}\frac{n!}{\bsalpha!}\|D^{|\bsalpha|}G(F(y))\|_{\cL^{|\bsalpha|}(X;Y)} \prod_{m=1}^{n-1} \left (\frac{\|D^m F(y)\|_{\cL_{\sym^m(Y;X)}}}{m!} \right )^{\alpha_m}\right).
\end{align*}
Using the induction Assumption \eqref{eq:indclaimNI} to bound
$\|D^mF(y)\|_{\cL_{\sym^m(Y;X)}}\le m^{m^2}\gamma^{3m-2}$ for $1\le m<n$, we
find for $y\in G(S)$
\begin{align*}
  \|D^nF(y)\|_{\cL_{\sym^n(Y;X)}} &\le \gamma^2 \left ( \sum_{\bsalpha\in
                           \bar{T}_n}\frac{n!}{\bsalpha!} \prod_{m=1}^{n-1} \left (\frac{m^{m}\gamma^{3m-2}}{m!} \right )^{\alpha_m}\right)\nonumber\\
                                &\le \gamma^2
                                  m^{\sum_{m=1}^{n-1}m^2\alpha_m}\gamma^{3n-4}
\left ( \sum_{\bsalpha\in
                           \bar{T}_n}\frac{n!}{\bsalpha!} \prod_{m=1}^{n-1}\frac{1}{(m!)^{\alpha_m}} \right)\\
                           &\le \gamma^{3n-2}n^{(n-1)n}\left ( \sum_{\bsalpha\in
                           \bar{T}_n}\frac{n!}{\bsalpha!} \prod_{m=1}^{n-1}\frac{1}{(m!)^{\alpha_m}} \right) ,
\end{align*}
where we have used $\sum_{m=1}^{n-1}m\alpha_m=n$ and
$\sum_{m=1}^{n-1} \alpha_m\ge 2$ for all $\bsalpha\in \bar T_n$. The
term in the sum is bounded by $n^n$ according to Lemma
\ref{lemma:stirling}.  Thus for $2\le n\le k$,
\begin{equation*}
  \|D^nF(y)\|_{\cL_{\sym}^n(Y;X)}\le n^{n^2}\gamma^{3n-2}
\end{equation*}
which shows \eqref{eq:indclaimNI} and concludes the proof.
\end{proof}

We now present our first bound on $\|f\|_{C^k}$. The estimate depends
on $\|T\|_{C^k}$ as well as
\begin{equation}\label{eq:cT}
  c_T:= \sup_{t\in[0,1]}\sup_{x\in \Omega_0}\|(\nabla T_t(x))^{-1}\|_{\R^{d\times d}}.
\end{equation}
We subsequently discuss situations in which we can give precise bounds
on this constant.

\begin{theorem}\label{thm:fNormGeneral}
  Let Assumption \ref{ass:dens1} be satisfied. Let $k\in\N$ and let
  $T\in C^k(\Omega_0,\Omega_1)$ satisfy Assumption \ref{ass:T}. Then
  for the velocity field $f:\Omega_{[0,1]}\to\R^d$ in \eqref{eq:fexpl} it
  holds with
\begin{equation*}
  \gamma :=\max\{2,1+c_T\}(1+\|T\|_{C^k(\Omega_0)}+\sup_{x\in\Omega_0}\|x\|),
\end{equation*}
that
\begin{equation*}
  \|f\|_{C^k(\Omega_{[0,1]})}\le 
 2k^{k^3 + k}\e^{k^2+k}\gamma^{3k^2-2k+1}.
\end{equation*}
\end{theorem}

\begin{proof}[Proof of Theorem \ref{thm:fNormGeneral}]
  Due to $f=T\circ F-F$ (cp.~\eqref{eq:fexpl}),
  \begin{equation}\label{eq:ftriangle}
    \|f\|_{C^k(\Omega_{[0,1]})} \leq \|T\circ F\|_{C^k(\Omega_{[0,1]})} + \|F\|_{C^k(\Omega_{[0,1]})}.
  \end{equation}
  Moreover, since $F$ is given as the first $d$ components of
  $G^{-1}:\Omega_{[0,1]}\to\Omega_0\times [0,1]$, it holds
  $\|F\|_{C^k(\Omega, \mathbb{R}^d)} \leq
  \|G^{-1}\|_{C^k(\Omega,\mathbb{R}^{d+1})}$.
  We start by bounding the latter norm.
  
  By definition of $G(x,t)=(T_t(x),t)=(tT(x)+(1-t)x,t)$,
  \begin{equation*}
    DG(x,t)=
        \begin{pmatrix}
          \nabla_x T_t(x)& T(x) - x \\
      0 & 1 \\
      \end{pmatrix}=
    \begin{pmatrix}
      t\nabla T(x) + (1-t)I & T(x) - x \\
      0 & 1 \\
      \end{pmatrix}.
    \end{equation*}
    An application of
    Lemma \ref{lemma:blockTriangle} and the assumption that
    $c_T=\sup_{x\in\Omega_0}\|(\nabla_xT_t(x))^{-1}\|<\infty$ gives for
    all $(x,t)\in \Omega_0\times [0,1]$
\begin{equation*}
  \|[DG(x,t)]^{-1}\|_{2}
  \le 1+(1+\|T-{\rm Id}\|_{L^\infty(\Omega_0)})\|(\nabla_xT_t(x))^{-1}\|_{2}
  \le 1+(1+\|T-{\rm Id}\|_{L^\infty(\Omega_0)}) c_T.
\end{equation*}

Next we bound the derivatives of $G$.
For $n=1$,
\begin{align*}
  \|DG(x,t)\|_{\cL^1(\R^{d+1};\R^{d+1})}\le \|\nabla_x T_t(x)\|_2+\|T(x)-x\|_2+1
  &\le \|T\|_{C^1(\Omega_0)}+1+\|T-{\rm Id}\|_{C^0(\Omega_0)}+1\nonumber\\
  &\le 2+2\|T\|_{C^1(\Omega_0)}+\max_{x\in\Omega_0}\|x\|_2.
\end{align*}
For $n\ge 2$, we first write $G(x,t)=(G_1(x,t),G_2(x,t))$ where
$G_1(x,t)=T_t(x)$ and $G_2(x,t)=t$. Then for $2\le n\le k$ and
$(x,t)\in\Omega_0\times [0,1]$
\begin{equation*}
  \|D^nG(x,t)\|_{\cL^n(\R^{d+1};\R^{d+1})}\le
  \|D^nG_1(x,t)\|_{\cL^n(\R^{d+1};\R^{d})}+
  \|D^nG_2(x,t)\|_{\cL^n(\R^{d+1};\R)}.
\end{equation*}
The second term is bounded by $1$ since $t\in [0,1]$. For the first term,
due to $D_t^2G_1(x,t)\equiv 0$,
\begin{align*}
  &\|D^n G_1(x,t)\|_{\cL^n(\R^{d+1};\R^{d+1})}
    \le \|D_x^n G_1(x,t)\|_{\cL^n(\R^d;\R^d))}
    +\|D_x^{n-1} (T(x)-x)\|_{\cL^{n-1}(\R^d;\R^d)}\nonumber\\
  &\qquad \le \|D_x^n (x)\|_{\cL^n(\R^{d};\R^{d})}+\|D_x^n T(x)\|_{\cL^n(\R^{d};\R^{d})}
    +\|D_x^{n-1} T(x)\|_{\cL^{n-1}(\R^d;\R^d)}
    +\|D_x^{n-1} (x)\|_{\cL^{n-1}(\R^d;\R^d)}\nonumber\\
  &\qquad\le \|D_x^n T(x)\|_{\cL^n(\R^{d};\R^{d})}+\|D_x^{n-1} T(x)\|_{\cL^{n-1}(\R^{d};\R^{d})}+1.
\end{align*}
We conclude with $M:=\max_{x\in\Omega_0}\|x\|_2$
that for all $0 \le n\le k$ and all $(x,t)\in \Omega_0\times [0,1]$
\begin{equation*}
  \|D^nG(x,t)\|_{\cL^n_{\sym{\R^{d+1};\R^{d+1}}}}\le 2(\|T\|_{C^n(\Omega_0)}+1+M)\le 2(\|T\|_{C^k(\Omega_0)}+1+M).
\end{equation*}

With
\begin{equation*}
  \gamma :=\max\{2,1+c_T\}(\|T\|_{C^k(\Omega_0)}+1+M),
\end{equation*}
Proposition \ref{prop:NormInverse} then implies
\begin{equation*}
  \|F\|_{C^k(\Omega_{[0,1]})}\le
  \|G^{-1}\|_{C^k(\Omega_{[0,1]})}\le \e^kk^{k^2}\gamma^{3k-2}.
\end{equation*}
Furthermore, Corollary \ref{cor:composition} and \eqref{eq:ftriangle}
yield
\begin{equation*}
  \|f\|_{C^k(\Omega_{[0,1]})}\le k^k\e^{k}\|T\|_{C^k(\Omega_0)}(\e^kk^{k^2}\gamma^{3k-2})^k+
  \e^kk^{k^2}\gamma^{3k-2}
  \le 2k^{k^3 + k}\e^{k^2+k}\gamma^{3k^2-2k+1}.
\end{equation*}
\end{proof}

Our main observation is, that $\|f\|_{C^k(\Omega_0)}$ behaves at worst
polynomial in $\|T\|_{C^k(\Omega_0)}$ and $c_T$ in \eqref{eq:cT}. One
important instance where we can give a precise bound on $c_T$, is if
$\nabla T$ is close to identity matrix $I_d\in\R^{d\times d}$ in the
sense $\|\nabla T(x)-I_d\|_2<1$. Since $T$ is a transport pushing
forward the source $\pi$ to the target $\rho$, this condition can be
interpreted as the source and the target not being too different.
\begin{lemma}
  Suppose that  $\sup_{x\in\Omega_0}\|\nabla
  T(x)-I_d\|_{2}=\delta<1$, where $I_d$ is $d$-by-$d$ identity matrix. Then the constant in \eqref{eq:cT}
  satisfies $c_T\le \frac{1}{1-\delta}$.
\end{lemma}
\begin{proof}
  The assumption implies that for all $t\in [0,1]$
  \begin{equation*}
    \|\nabla_x T_t(x)-I_d\|_2
    =\|t \nabla T(x)+(1-t)I_d-(tI_d+(1-t)I_d)\|_2
    = t \|\nabla T(x)-I_d\|_2 \le\delta.
  \end{equation*}
  Since for any $A\in \R^{d\times d}$ with $\|A-I_d\|_2=\delta<1$
  we have $A^{-1}=\sum_{j\in\N}(I-A)^j$ and thus
  $\|A^{-1}\|_2\le \frac{1}{1-\delta}$, the claim follows.
\end{proof}

Another instance where $c_T$ can be bounded is if $T$ is a triangular
transport. We next discuss this case in more detail.

\subsection{Triangular transports}\label{sec:trireg}
A special type of transport map commonly used in practice is the
  so-called Knothe-Rosenblatt (KR) map. To avoid further
  techincalities, throughout this section we restrict ourselves to
  measures on $d$-dimensional cubes, i.e.
  \begin{equation}\label{eq:cube}
    \Omega_0=\Omega_1=[0,1]^d.
  \end{equation}
  The KR map, is the unique transport satisfying \emph{triangularity}
  and \emph{monotonicity}. To formally introduce these notions, recall
  that we use the notation convention $x_{[j]}=(x_i)_{i=1}^j$.
  \begin{definition}
    We say that a map $T=(T_j)_{j=1}^d:[0,1]^d\to [0,1]^d$ is {\bf
      triangular} iff $T_j$ depends only on $x_{[j]}$ (but not on
    $x_{j+1},\dots,x_d$) for each $j=1,\dots,d$. A triangular map $T$
    is called {\bf monotone} iff for each $j=1,\dots,d$ the map
    $[0,1]\ni x_j\mapsto T_j(x_{[j]})$ is differentiable and
    $\frac{\partial}{\partial x_j} T_j(x_{[j]})>0$ for all
    $x_{[j]}\in [0,1]^j$.
  \end{definition}
  Under rather mild assumptions on the reference and target one can
  show existence and uniqueness of the KR map (\cite{OTAppliedMathematician}).
  Moreover, it allows an explicit construction, which we recall in
  Appendix \ref{app:KRMap}. The goal of the present section is to
  derive bounds on the norms of the velocity field associated with the
  KR-map.

  We start our analysis by pointing out that triangularity of the
  transport $T$ implies a similar structure for the corresponding
  velocity field $f$:
  \begin{lemma}
    Consider the setting of Theorem ~\ref{thm:f} and let $\Omega_0$,
    $\Omega_1$ be as in \eqref{eq:cube}. If $T:\Omega_0\to\Omega_1$ is
    triangular, then $f:\Omega_{[0,1]}\to\R^d$ is triangular in the
    sense $f(x,t)=(f_i(x_{[i]},t))_{i=1}^d$.
  \end{lemma}
 
  \begin{proof}
    Let the solution $X:\Omega_{0}\times [0,1]\to
    \R^d$ of \eqref{eq:ODE} satisfy $X(x,t) = tT(x) +(1-t)x$.
    Then for the velocity field
    $f:\Omega_{[0,1]}\to\R^d$ in \eqref{eq:ODE} it holds $f(X(x,t), t)
    = T(x) - x$, i.e.\ for each $i=1,\dots,d$
    \begin{equation*}
      f_i(X(x,t), t) = T_i(x_{[i]}) - x_i\qquad\forall (x,t)\in\Omega_{0}\times [0,1].
    \end{equation*}
    To prove the lemma we show that for all
    $i\in\{1,\dots,d\}$ it holds
    \begin{equation}\label{eq:claimtri}
      \frac{\partial}{\partial X_j}f_i(X(x,t), t) = 0
      \qquad\forall j>i,\quad\forall (x,t)\in\Omega_0\times [0,1].
    \end{equation}

    Fix $i\in\{1,\dots,d\}$. To prove \eqref{eq:claimtri}, we proceed
    by induction over $j=i+1,\dots,d$ starting with $j=d$. Note that
    the triangularity of $T$ implies that also
    $X(x,t)=(X_l(x_{[l]},t))_{l=1}^d$ inherits this triangular
    structure. Hence $\frac{\partial}{\partial x_d}X_l(x_{[l]})=0$ for
    all $l<d$.  Consequently
    \begin{equation*}
      \frac{\dd}{\dd x_d} f_i(X(x,t),t) = \frac{\partial}{\partial X_d} f_i(X(x,t),t)
      \frac{\partial}{\partial x_d} X_d(x,t)=0\qquad\forall
      (x,t)\in\Omega_{0}\times [0,1].
    \end{equation*}
    By the monotonicity assumption on $T$ it holds
    $\frac{\partial T_d(x)}{\partial x_d}>0$, and therefore
    \begin{equation*}
      \frac{\partial}{\partial x_d} X_d(x,t) = t\frac{\partial
      }{\partial x_d}T_d(x)+(1-t)>0\qquad\forall
      (x,t)\in\Omega_{0}\times [0,1].
    \end{equation*}
    Hence $\frac{\partial f_i(X(x,t),t)}{\partial X_d}=0$ for all
    $(x,t)\in\Omega_0\times [0,1]$.

    Now suppose \eqref{eq:claimtri} is true for all $j=k+1,\dots,d$
    and some $k\ge i$. Then, using
    $\frac{\partial X_j(x,t)}{\partial x_k}=0$ whenever $k>j$ and as
    well as \eqref{eq:claimtri} whenever $j>k$
    \begin{equation*}
      \frac{\dd}{\dd x_k}f_i(X(x,t), t) =
      \sum_{j=1}^d\frac{\partial}{\partial X_{j}}f_i(X(x,t),
      t)\frac{\partial }{\partial x_{k}}X_{j}(x,t)
      =\frac{\partial}{\partial X_{k}}f_i(X(x,t),
      t)\frac{\partial }{\partial x_{k}}X_{k}(x,t)=0,
    \end{equation*}
    for $(x,t)\in\Omega_0\times [0,1]$. Since as before
    $\frac{\partial X_{k}(x_{[k]},t)}{\partial x_{k}} > 0$ we find
    $\frac{\partial f_i(X(x,t), t)}{\partial X_{j}} = 0$.
  \end{proof}

  In Theorem ~\ref{thm:LipDomain} we showed that the domain
  $\Omega_{[0,1]}$ (cp.~\eqref{eq:Omega}) of the velocity field $f$ is
  a Lipschitz domain. For triangular maps and if \eqref{eq:cube} it
  even holds $\Omega_{[0,1]}=[0,1]^d\times [0,1]$, i.e.\ the
  trajectories of the solutions to \eqref{eq:ODE} cover the whole
  $d+1$ dimensional cube:

  \begin{proposition}\label{prop:cube}
    Let $T:[0,1]^d\to [0,1]^d$ be a monotone, triangular and bijective
    map.
    Then
    \begin{equation*}
      \Omega_{[0,1]}
      = [0,1]^d\times [0,1].
    \end{equation*}
  \end{proposition}
  \begin{proof}
    It is easy to see, that for a monotone triangular map,
    $T:[0,1]^d\to [0,1]^d$ being bijective is equivalent to
    $x_j\mapsto T_j(x_{[j]})$ being bijective from $[0,1]\to [0,1]$
    for each $j=1,\dots,d$, see e.g., \cite[Lemma 3.1]{ZM2}.
    
    Denote $T_t(x)=(1-t)T(x)+tx$ and additionally
    $T_{t,j}(x)=(1-t)T_j(x_{[j]})+tx_j$
    $(x,t)\in\Omega_0\times [0,1]$.  For every $t\in [0,1]$,
    $T_t:[0,1]^d\to [0,1]^d$ is a convex combination of two monotone,
    triangular bijective maps. Hence $T_t:[0,1]^d\to [0,1]^d$ is also
    triangular. Moreover, for each $j\in\{1,\dots,d\}$ and
    $t\in [0,1]$, $x_j\mapsto T_{t,j}(x_{[j]})$ is a convex
    combination of two monotonically increasing maps that bijectively
    map $[0,1]$ onto itself, and thus this function has the same
    property. In all this shows that $T_t:[0,1]^d\to [0,1]^d$ is
    monotone, triangular and bijective for each $t\in [0,1]$, which
    implies the claim.
  \end{proof}

  We wish to apply Theorem \ref{thm:fNormGeneral} to bound
  $\|f\|_{C^k(\Omega)}$. To do so, it remains to bound
  $\|T\|_{C^k([0,1]^d, [0,1]^d)}$ and $c_T$ as below.

\begin{lemma}\label{lemma:cT}
  Let $\pi, \rho$ be densities supported on $[0,1]^d$ and satisfy regularity Assumption \ref{ass:dens2}. Let $T\in C^1([0,1]^d, [0,1]^d)$ be the Knothe-Rosenblatt map such that $T_\sharp\pi = \rho$. Then
  the constant $c_T$ from \eqref{eq:cT} satisfies
  $$c_T := \sup_{x\in[0,1]^d}\|(\nabla T_t(x))^{-1}\|_{R^{d\times d}}
  \leq (\frac{L_1}{L_2})^{2d}\max\{1, \|T\|_{C^1([0,1]^d)}\}^{d-1}.$$
\end{lemma}
\begin{proof}
  In this proof we use the notation and construction of the transport map provided in \cref{app:KRMap}. In particular $\pi_i$ is the marginal density of $\pi$ in $(x_1,\dots,x_i)$, and with $F_{\rho,i}$, $F_{\pi,i}$ as in \eqref{eq:Fpik}, we let
    \begin{equation*}
      G_{\rho, i}(x_{[i-1]}, \cdot) = F_{\rho,i}(x_{[i-1]}, \cdot)^{-1}.
  \end{equation*}

  By 
  construction, the Jacobian $\nabla T$ is a triangular
  matrix. We shall compute the diagonal entries of the Jacobian
  matrix, which are the eigenvalues. 
  By \eqref{eq:Tk}
  $$T_i(x) = G_{\rho, i}(T_1(x_1),...,T_{i-1}(x_{[i-1]}), F_{\pi,
    i}(x)).$$ Taking derivatives in $x_i$, we have
$$\partial_{x_i}T_i(x) = \partial_{x_i}G_{\rho, i}(T_1(x_1),...,T_{i-1}(x_{[i-1]}), \pi_i(x))\partial_{x_i}F_{\pi,i}(x).$$

Recall that $F_{\pi,i}(x)$ is the CDF of $x_i$ when viewing
$x_{[i-1]}$ as fixed, thus we have
$\partial_{x_i}F_{\pi,i}(x) = \pi_i(x)$. Note that
$G_{\rho,i}(x_{[i-1]}, F_{\rho,i}(x)) = x_i$. Taking derivative in
$x_i$, we have
$$(\partial_{x_i}G_{\rho,i})(x_{[i-1]}, F_{\rho,i}(x))(\partial_{x_i}F_{\rho,i}(x)) = 1.$$

Note that $F_{\rho,i}(x_{[i-1]}, \cdot):[0,1]\rightarrow[0,1]$ is a
bijection. We make the substitution $y_i = F_{\rho,i}(x)$ and we have
for all $(x_{[i-1]}, y_i)\in[0,1]^{i-1}\times[0,1]$,
$$(\partial_{x_i}G_{\rho,i})(x_{[i-1]}, y_i) = \frac{1}{\partial_{x_i}F_{\rho,i}(x_{[i-1]}, G_{\rho,i}(x_{[i-1]}, y_i))}.$$
Hence
$$\partial_{x_i}T_i(x_{[i]}) = \frac{\pi_i(x_{[i]})}{\rho_i(T_1(x_1),...,T_{i-1}(x_{[i-1]}), G_{\rho,i}(T_1(x_1),...,T_{i-1}(x_{[i-1]}),F_{\pi,i}(x)))}.$$

By Assumption \ref{ass:dens2}, $\rho$ and $\pi$ are bounded from above
and below by $L_1$ and $L_2$. 
Thus $\pi_i(x_{[i]})$ and
$\rho_i(x_{[i]})$ are bounded from above and below by
$\frac{L_1}{L_2}$ and $\frac{L_2}{L_1}$, 
and $\partial_{x_i}T_i(x_{[i]})$ is bounded from below by
$\frac{L_2^2}{L_1^2}$. Note that the diagonal entries of
$\nabla_x T_t(x)^{-1} = [(1-t)I_{d\times d} + t\nabla_F
T(F(y,t))]^{-1}$ are exactly $\{\frac{1}{1-t + t\sigma_i}\}_{i=1}^d$,
which are lower bounded by
$\frac{L_2^2}{(1-t)L_2^2 + tL_1^2} \geq \frac{L_2^2}{L_1^2}$.

Therefore, we have
$\inf_{x\in\Omega_0}\det (\nabla_x T(x)) \geq
\left(\frac{L_2}{L_1}\right)^{2d}$. Applying
 Lemma~\ref{lemma:NormTtinverse} then gives the result.
\end{proof}


When $T$ is  the KR triangular map, $\|T\|_{C^K}([0,1]^d,[0.1]^d)$ can be computed explicitly in terms of densities as we see below. 
\begin{theorem}\label{thm:CkNormTriangular}
  Let $\Omega_0=\Omega_1=[0,1]^d$ and let $\pi$, $\rho$ satisfy
    Assumption \ref{ass:dens2} with the constants $k\ge 2$,
    $0<L_1\le L_2<\infty$.

    Then there exist constants $C_{k,d}$ (depending on $k$ and $d$ but independent of
    $\rho$, $\pi$)
    and $\beta_d>0$ (depending on $d$ but independent of $k$,
    $\rho$, $\pi$)
    such that the KR map $T$ pushing forward $\pi$
    to $\rho$ satisfies
    \begin{equation}
      \|T\|_{C^k([0,1]^d, [0,1]^d)} \leq C_{k,d}\left(\frac{L_1}{L_2}\right)^{\beta_d k^{d+1}}.
    \end{equation}
\end{theorem}
\begin{proof}
Throughout this proof we use the notation and explicit
    construction of the KR map introduced in Appendix \ref{app:KRMap}: The
    $i$-th component of the KR map can then be expressed as
    \begin{equation}\label{eq:Tixi}
      T_i(x_{[i]}) = G_{\rho, i}(T_1(x_1), ..., T_{i-1}(x_{[i-1]}), F_{\pi,
        i}(x_{[i]})).
    \end{equation}
    Here
    $F_{\rho, i}(x_{[i]})= \int_{0}^{x_i} \rho_i(x_{[i-1]}, t_i)dt_i$
    with $\rho_i = \frac{\hat{\rho}_i}{\hat{\rho}_{i-1}}$, where
    \begin{equation}\label{eq:hatrhoi}
      \hat{\rho}_i(x_1,...x_i) =
      \int_{[0,1]^{d-i}}\rho(x_1,...,x_d)dx_{i+1}...dx_d.
    \end{equation}
    The function $F_{\pi, i}(x_{[i]})$ is defined analogous with
    $\rho$ replaced by $\pi$.  Finally,
    $x_i\mapsto G_{\rho, i}(x_{[i]})$ is defined as the inverse of
    $x_i\mapsto F_{\rho, i}(x_{[i]})$ on $[0,1]$, i.e.
    \begin{equation}\label{eq:Grhoidef}
      G_{\rho,i}(x_{[i-1]},F_{\rho,i}(x_{[i]}))=x_i.
    \end{equation}

    The proof proceeds as follows: In step 1 we bound
    $\|F_{\pi,i}\|_{C^k(\ci)}$, $\|F_{\rho,i}\|_{C^k(\ci)}$, and in
    step 2 we bound $\|G_{\rho,i}\|_{C^k(\cd)}$. In Step 3 we use
    induction over $i$ to bound the norm of $T_i$ in \eqref{eq:Tixi}.

    {\bf Step 1.} Fix $i\in\{1,\dots,d\}$. In this step we show
    \begin{equation}\label{eq:Step1Claim}
      \max\{\|F_{\pi,i}\|_{C^k(\ci)},\|F_{\rho,i}\|_{C^k(\ci)}\}\le C
      \Big(\frac{L_1}{L_2}\Big)^{k+1}
    \end{equation}
    for some constant $C$ depending on $d$ and $k$ but independent of
    $\pi$ and $\rho$. Since our assumptions on $\rho$ and $\pi$ are
    identical, it suffices to prove \eqref{eq:Step1Claim} for $\rho$.

  Fix a multiindex $\bsv\in\N_0^d$. If $v_i=0$, then $D^\bsv F_{\rho,
    i} = \int_0^{x_i}D^\bsv \rho_i(x_{[i-1]},
  t_i)dt_i$.  Otherwise with $\bsv' = \bsv -
  \bse_i$ (where $\bse_i=(\delta_{i,j})_{j=1}^d$) it holds $D^\bsv
  F_{\rho, i} = D^{\bsv'} \rho_i(x_{[i-1]},
  x_i)$. In either case $\|D^\bsv F_{\rho, i}\|_{L^\infty([0,1]^i)}
  \leq \|\rho_i\|_{C^k([0,1]^i)}$. Thus $\|F_{\rho,
    i}\|_{C^k([0,1]^i)} \leq \|\rho_i\|_{C^k([0,1]^i)}$.

  Similarly, Assumption \ref{ass:dens2} implies $$\hat\rho_i(x_1,...,x_i) = \int_{[0,1]^{d-i}}\rho(x_1,...x_d)dx_{i+1}...dx_d >  \int_{[0,1]^{d-i}}L_2dx_{i+1}...dx_d > L_2$$ and $$D^\bsv\hat\rho_i(x_1,...,x_i) = \int_{[0,1]^{d-i}}D^\bsv\rho(x_1,...x_d)dx_{i+1}...dx_d \leq \int_{[0,1]^{d-i}}L_1dx_{i+1}...dx_d = L_1$$ for all $(x_1,..,x_i)\in [0,1]^i$ and multi-index $\bsv\in\N_0^d$ with $|\bsv|\leq k$. Thus we can conclude that $\|\hat\rho_i\|_{C^k([0,1]^i)}\le L_1$ and Lemma \ref{lemma:NormReciprocal}
  gives
  $\|D^{n-j}(\frac{1}{\hat{\rho}_{i-1}})\|_{L^\infty([0,1]^{i-1})}
  \leq C_{n-j}\frac{L_1^{n-j}}{L_2^{n-j+1}}$ for some constant
  $C_{n-j}$ that is independent of $L_1$, $L_2$. By the Leibniz rule
  \begin{equation*}
    D^n\rho_i = \sum_{j=0}^n {n \choose j}D^j\hat{\rho}_i
    D^{n-j}\Big(\frac{1}{\hat{\rho}_{i-1}}\Big)
  \end{equation*}
  and thus
  \begin{equation*}
    \|D^{n}\rho_i\|_{L^\infty([0,1]^{i})} \le CL_1\frac{L_1^n}{L_2^{n+1}} =
    C\frac{L_1^{n+1}}{L_2^{n+1}}    \qquad\forall n\in\{0,\dots,k\}
  \end{equation*}
  for some constant $C$ that depends on $k$ but is independent of
  $L_1$, $L_2$. In all this shows \eqref{eq:Step1Claim} for $\rho$.
  
  {\bf Step 2.}  Fix $i\in\{1,\dots,d\}$. In this step we show
  \begin{equation}\label{eq:Step2Claim}
    \max\{\|G_{\pi,i}\|_{C^k(\ci)},\|G_{\rho,i}\|_{C^k(\ci)}\}\le C
    \Big(\frac{L_1}{L_2}\Big)^{(k+1)(3k-2)}
  \end{equation}
  for some constant $C$ depending on $d$ and $k$ but independent of
  $\pi$ and $\rho$. As before, by symmetry it suffices to provide the
  bound for $\pi$.

  For $x_{[i]}\in[0,1]^i$ define
  \begin{equation}\label{eq:tildeFdef}
    \tilde{F}_{\rho, i}(x_{[i]}) := (x_{[i-1]}, F_{\rho,
      i}(x_{[i]}))\in [0,1]^i.
  \end{equation}
  Since $x_i\mapsto \tilde{F}_{\rho, i}(x_{[i]})$ is bijective from
  $[0,1]\to [0,1]$ for every fixed $x_{[i-1]}\in [0,1]^{i-1}$, the map
  $\tilde{F}_{\rho, i}:[0,1]^i\to [0,1]^i$ is a bijection.  So is its
  inverse which we denote by $\tilde{G}_{\rho, i}:[0,1]^i\to [0,1]^i$.
  It holds for all $x_{[i]}\in [0,1]^i$ that
  \begin{equation*}
    \tilde{G}_{\rho, i}(\tilde{F}_{\rho, i}(x_{[i]})) =
    x_{[i]}.
  \end{equation*}
  Due to \eqref{eq:Grhoidef} and the definition of $\tilde F_{\rho,i}$
  in \eqref{eq:tildeFdef}, the $i$th component of
  $\tilde{G}_{\rho, i}$ is given by $G_{\rho,i}$ and thus
  \begin{equation*}
    \|G_{\rho,i}\|_{C^k([0,1]^i, [0,1])}\leq \|\tilde{G}_{\rho, i}\|_{C^k([0,1]^i, [0,1]^i)}.
  \end{equation*}

  In the following we wish to apply Prop.~\ref{prop:NormInverse} to
  bound the right-hand side, which will yield a bound on the left-hand
  side.  To this end we first derive a bound on the norm of
  $\tilde F_{\rho,i}$. By \eqref{eq:tildeFdef} and
  \eqref{eq:Step1Claim}
  \begin{equation}\label{eq:tildeFrhobound}
    \|\tilde{F}_{\rho,
      i}\|_{C^k([0,1]^i, [0,1]^i)} \leq \|F_{\rho,i}\|_{C^k([0,1]^i, [0,1])}
    + (i-1)\le C\Big(\frac{L_1}{L_2}\Big)^{k+1}.
  \end{equation}
  For the last inequality we used $\frac{L_1}{L_2}\ge 1$, so that
  $i-1\le k-1$ can be absorbed into the $k$-dependent constant $C$.
  Next we bound $\|[D\tilde{F}_{\rho, i}]^{-1}\|_2$.  It holds
  \begin{equation*}
    \begin{pmatrix} 
      1 &  \dots & 0  & 0\\
      \vdots & \ddots &  &\\
      0 &        & 1 & 0\\
      \partial_{x_1}F_{\rho, i}&\partial_{x_2}F_{\rho, i}&\dots &
      \partial_{x_i}F_{\rho, i}
    \end{pmatrix}
  \end{equation*}
  and
  \begin{equation*}
    \begin{pmatrix} 
      1 &  \dots & 0  & 0\\
      \vdots & \ddots &  &\\
      0 &        & 1 & 0\\
      -\frac{\partial_{x_1}F_{\rho, i}}{\partial_{x_i}F_{\rho,
          i}}&-\frac{\partial_{x_2}F_{\rho, i}}{\partial_{x_i}F_{\rho,
          i}}&\dots & -\frac{1}{\partial_{x_i}F_{\rho, i}}
    \end{pmatrix}.
  \end{equation*}
    
  Since
  $\partial_{x_i}F_{\rho, i}(x_{[i]}) = \rho_i(x_{[i]}) \geq L_2$ and
  $\|\partial_{x_j}F_{\rho, i}(x_{[i]})\| \leq L_1$ for all
  $x_{[i]}\in [0,1]^i$ and $j\in\{1,\dots,i\}$, we have
  \begin{equation*}
    \Big\|\frac{\partial_{x_j}F_{\rho, i}}{\partial_{x_i}F_{\rho,
        i}}\Big\|_{L^\infty([0,1]^i)} \leq \frac{L_1}{L_2}
    \qquad
    \forall j\in\{1,\dots,i\}.
  \end{equation*}
  Thus for all $x_{[i]}\in [0,1]^i$
  \begin{equation*}
    \|[D\tilde{F}_{\rho, i}(x_{[i]})]^{-1}\|_2 \le \|[D\tilde{F}_{\rho,
      i}(x_{[i]})]^{-1}\|_F \le
    \sqrt{(i-1) + i(\frac{L_1}{L_2})^2} \le
    \sqrt{2i}\frac{L_1}{L_2}.
  \end{equation*}
  Since $\frac{L_1}{L_2} \ge 1$, we conclude that there exists a
  constant $C$ depending on $k$, but independent of $L_1$, $L_2$ such
  that
  \begin{equation*}
    \max\{\|[D\tilde{F}_{\rho, i}]^{-1}\|_2, \|\tilde{F}_{\rho,
      i}\|_{C^k([0,1]^i, [0,1]^i)}\} \leq
    C\frac{L_1^{k+1}}{L_2^{k+1}}.
  \end{equation*}
  Thus by Proposition \ref{prop:NormInverse}
  \begin{equation*}
    \|G_{\rho, i}\|_{C^k([0,1]^i, [0,1])}\leq \|\tilde{G}_{\rho,
      i}\|_{C^k([0,1]^i, [0,1]^i)} \leq
    e^kk^{k^2}\left(C\frac{L_1^{k+1}}{L_2^{k+1}}\right)^{3k-2} \leq C
    \left(\frac{L_1}{L_2}\right)^{(k+1)(3k-2)}.
  \end{equation*}
  This shows \eqref{eq:Step2Claim} for $\rho$.
  
  {\bf Step 3.} We finish the proof by showing that for all
  $i\in\{1,\dots,d\}$
  \begin{equation}\label{eq:Step3Claim}
    \|T_i\|_{C^k([0,1]^i, [0,1])} \leq C\left(\frac{L_1}{L_2}\right)^{k^i(k+1)+(k+1)(3k-2)(\sum_{j=0}^{i-1}k^j)}.
  \end{equation}

  For $x_{[i]}\in [0,1]^i$ define
  \begin{equation}\label{eq:tildeTidef}
    \tilde{T}_i(x_{[i]}) := (T_1(x_1), \dots,
    T_{i-1}(x_{[i-1]}), F_{\pi, i}(x_{[i]}))\in [0,1]^i.
  \end{equation}
  By \eqref{eq:Tixi} it holds $T_i=G_{\rho,i}\circ\tilde T_i$, and
  thus Corollary \ref{cor:composition} implies
  \begin{equation}\label{eq:Step3main}
    \|T_i\|_{C^k([0,1]^i)}\le C\|G_{\rho,i}\|_{C^k([0,1]^i)}\max\{1,\|\tilde T_i\|_{C^k([0,1]^i)}\}^k
  \end{equation}
  for a $k$-dependent constant $C$.

  We proceed by induction over $i$, and start with $i=1$. In this case
  \eqref{eq:Step1Claim}, \eqref{eq:Step2Claim} and
  \eqref{eq:Step3main} yield
  \begin{equation*}
    \|T_1\|_{C^k([0,1])}
    =\|G_{\rho,1}\circ F_{\pi,1}\|_{C^k([0,1])}
    \le C \Big(\frac{L_1}{L_2}\Big)^{(k+1)(3k-2)+k(k+1)},
  \end{equation*}
  and thus \eqref{eq:Step3Claim} is satisfied. For the induction step
  assume the statement is true for $i-1\ge 1$. By
  \eqref{eq:Step1Claim}, \eqref{eq:tildeTidef} and the induction
  hypothesis
  \begin{align*}
    \|\tilde{T}_i\|_{C^k([0,1]^i)} &\le
    C\Bigg(
    \sum_{j=1}^{i-1}
    \|T_j\|_{C^k([0,1]^j)}
    +\|F_{\pi,i}\|_{C^k([0,1]^i)}\Bigg)\nonumber\\
    &\le
    C \Big(\frac{L_1}{L_2}\Big)^{k^{i-1}(k+1)+(k+1)(3k-2)(\sum_{j=0}^{i-2}k^j)},
  \end{align*}
  where again $C$ may depend on $k$ (or $i\le k$) but not on $L_1$,
  $L_2$.  Then \eqref{eq:Step2Claim} and \eqref{eq:Step3main} imply
  \begin{align*}
    \|T_i\|_{C^k([0,1]^i)}
    &\le C 
    \|G_{\rho,i}\|_{C^k([0,1]^i)}
      \max\{1,\|\tilde T_i\|_{C^k([0,1]^i)}\}^k\nonumber\\
    &\le C \Big(\frac{L_1}{L_2}\Big)^{(k+1)(3k-2)}
      \Big(\frac{L_1}{L_2}\Big)^{k(k^{i-1}(k+1)+(k+1)(3k-2)(\sum_{j=0}^{i-2}k^j))}\nonumber\\
    &= C\left(\frac{L_1}{L_2}\right)^{k^i(k+1)+(k+1)(3k-2)(\sum_{j=0}^{i-1}k^j)}.
  \end{align*}

\end{proof}
  
Finally, by putting all the estimates together, we obtain the following upper bound for the $C^k$ norm of velocity field. 

\begin{theorem}\label{thm:tri}
Let $\Omega_0 = \Omega_1 = [0,1]^d$ and $\pi, \rho$ satisfy Assumption \ref{ass:dens2} with constant $k \geq 2$, $0 < L_1 \leq L_2 < \infty$. Let $T: [0,1]^d\rightarrow[0,1]^d$ be the KR map pushing forward $\pi$ to $\rho$ and $f:[0,1]^d\times[0,1]\rightarrow[0,1]^d$ be the velocity field in \eqref{eq:V} that corresponds to the displacement interpolation between $x$ and $T(x)$. Then, there exists constants $C_{k,d}$ that only depends on $k, d$ and $\beta_d$ that only depends on $d$, 
such that the following holds:
$$\|f\|_{C^k([0,1]^d\times[0,1])} \leq C_{k,d}\left(\frac{L_1}{L_2}\right)^{\beta_dk^{d+3}}.$$
\end{theorem}
\begin{proof}
The proof of the theorem requires a combination of Theorem \ref{thm:fNormGeneral}, Lemma.\ref{lemma:cT} and Theorem \ref{thm:CkNormTriangular}. 

First, by Lemma.\ref{lemma:cT} and Theorem \ref{thm:CkNormTriangular}, there exists constant $C_{d}$ and $\beta_d'$ such that $\|T\|_{C^1([0,1]^d, [0,1])} \leq C_{d}(\frac{L_1}{L_2})^{\beta_d'}$ and 
$$c_T \leq (\frac{L_1}{L_2})^{2d}\max\{1, (C_{d}(\frac{L_1}{L_2})^{\beta_d'})^{d-1}\}\leq \max\{(\frac{L_1}{L_2})^{2d}, C_d^{d-1}(\frac{L_1}{L_2})^{\beta_d'(d-1)}\}.$$
By renaming $C_d = \max\{1, C_d^{d-1}\}$ and $\beta'_d = \max\{2d, \beta_d'(d-1)\}$, we can simplify the above expression as $c_T \leq C_d(\frac{L_1}{L_2})^{\beta'_d}$. Note it holds true that $c_T \geq 1$. 

By lemma.\ref{lemma:cT} and Theorem \ref{thm:CkNormTriangular}, there exists constants $C_{k,d}$ and $\beta_d$ and  
\begin{align*}
 \gamma &= \max\{2, 1+c_T\}(1 + \|T\|_{C^k} + \sup_{x\in\Omega_0}\|x\|)\leq (1 + c_T)(1 + C_{k,d}(\frac{L_1}{L_2})^{\beta_dk^{d+1}} + \sqrt{d})\\
 &\leq 2c_T(1 + C_{k,d}(\frac{L_1}{L_2})^{\beta_dk^{d+1}} + \sqrt{d})\leq 2c_TC_{k,d}(\frac{L_1}{L_2})^{\beta_dk^{d+1}}\\
 &\leq C_d(\frac{L_1}{L_2})^{\beta'_d}C_{k,d}(\frac{L_1}{L_2})^{\beta_d k^{d+1}} \leq C_{k,d}(\frac{L_1}{L_2})^{\beta_dk^{d+1}}.\\
\end{align*}
where we absorb constants whenever possible. In particular, we used $\frac{L_1}{L_2} > 1$ to absorb into $C_{k,d}$ in the third inequality and we used $k>1$ to absorb everything in the exponent into a $d-$dependent $\beta_d$ in the last inequality. 

Finally, applying Theorem \ref{thm:fNormGeneral}, we obtain
$$\|f\|_{C^k([0,1]^d\times[0,1])}\leq 2k^{k^3+k}e^{k^2+k}\gamma^{3k^2-2k+1} \leq C_{k,d} (\frac{L_1}{L_2})^{\beta_dk^{d+1}3k^2} \leq C_{k,d}(\frac{L_1}{L_2})^{\beta_dk^{d+3}}.$$
\end{proof}

\section{Stability in the velocity field}\label{sec:DistApproximation}
In the previous sections we showed existence of velocity fields $f$
that yield flow maps realizing a (triangular) transport that pushes
forward $\pi$ to $\rho$. In practice, a suitable velocity field $g$ is
obtained by minimizing the objective \eqref{Jobjective} over some
parametrized function class such as the set of Neural Networks of a
certain architecture. In general, such $g$ will only approximate $f$,
and it is therefore of interest to understand how errors in the
approximation of $f$ propagate to errors in the distributions realized
by the corresponding flow map. This is the purpose of the present
section.

\subsection{Wasserstein distance}
First, we present results when the divergence between probability distributions is measured by Wasserstein distance. That is, we take $\mathcal{D} = W_p$ in \eqref{Jobjective}, where $W_p$ is the $p-$Wasserstein distance.

\begin{theorem}\label{thm:FlowMapBound}
  Let $f$, $g\in\mathcal{V}$ (cp.~\eqref{eq:V}) and
  $\|f - g\|_{C^0(\R^d\times [0,1])}<\infty$. Assume that $L>0$ is
  such that $x\mapsto f(x,t)$ has Lipschitz constant $L$ for all
  $t\in [0,1]$.  Let $X_f$, $X_g:\R^d\times[0,1]\to\mathbb{R}^d$ be
  as in \eqref{eq:flowmap}.
 Then
 \begin{equation}\label{eq:FlowMapBound}
   \|X_f(\cdot,1) - X_g(\cdot,1)\|_{C^0(\R^d)} \leq \|f -
   g\|_{C^0(\R^d\times [0,1])} e^{L}.
 \end{equation}
\end{theorem}

The idea of the proof of Theorem ~\ref{thm:FlowMapBound} is to apply
Gr\"onwall's inequality to the evolution of the error
$|X_f(x,t) - X_g(x,t)|$ over time. We point out that this stability result
merely requires $g$ to approximate $f$ uniformly, however $f$ is
additionally assumed to be Lipschitz continuous in space. 
\begin{proof}[Proof of Theorem \ref{thm:FlowMapBound}]
Fix $x\in\R^d$. Then for all $s\in [0,1]$
\begin{align*}
|f(X_f(x,s),s) - g(X_g(x,s),s)| &\le |f(X_f(x,s),s) - f(X_g(x,s), s)| + |f(X_g(x,s), s) - g(X_g(x,s),s)|\\
                                                  &\le L|X_f(x,s)-X_g(x,s)|+\|f-g\|_{C^0(\R^d\times\R)},
\end{align*}
where we used the spatial Lipschitz continuity of $f$. Hence for $t\in [0,1]$
\begin{align*}
|X_f(x,t) - X_g(x,t)| &= \Big|\int_0^tf(X_f(x,s),s) - g(X_g(x,s),s)ds\Big|\\
&\le \int_0^tL|X_f(x,s)-X_g(x,s)|ds+t\|f-g\|_{C^0(\R^d\times\R)}.
\end{align*}
Using Gr\"onwall's inequality, we get $|X_f(x,1) - X_g(x,1)|\leq \|f-g\|_{C^0(\R^d\times\R)} e^{L}$
as claimed.
\end{proof}

Next we show how an approximation of the velocity field
affects the difference in distributions in terms of
the Wasserstein distance $W_p$. In the following
corollary, we denote by $|\supp(\pi)|$ the Lebesgue measure
of the support of $\pi$.

\begin{corollary}\label{cor:Wp}
  Let $f$, $g\in\mathcal{V}$ and $X_f$, $X_g$ be as in Theorem
  \ref{thm:FlowMapBound}. Let $\pi$ be a probability distribution on
  $\R^d$. Then for any $p\in [1,\infty)$
  \begin{equation*}
    W_p(X_f(\cdot,1)_\sharp\pi, X_g(\cdot,1)_\sharp\pi) \le
    \|f-g\|_{C^0(\R^d\times\R)} e^{L}
      |\supp(\pi)|^{1/p}.
    \end{equation*}
    Moreover, for $p=\infty$ holds
    $W_\infty(X_f(\cdot,1)_\sharp\pi, X_g(\cdot,1)_\sharp\pi) \le
    \|f-g\|_{C^0(\R^d\times\R)} e^{L}$.
\end{corollary}
\begin{proof}
  Let $F:\R^d\to\R^d\times\R^d$ via
  $F(x):=(X_f(x,1),X_g(x,1))$. Observe that $\gamma:=F_\sharp(\pi\otimes\pi)$ is
  then a probability distribution on $\R^d\times\R^d$ with marginals
  $X_f(\cdot,1)_\sharp\pi$ and $X_g(\cdot,1)_\sharp\pi$, i.e.\ it is a coupling
  of these measures. If $p<\infty$ then by definition of the Wasserstein distance
  \begin{align*}
    W_p(X_f(\cdot,1)_\sharp\pi, X_g(\cdot,1)_\sharp\pi)^p
    &\le \int_{\R^d\times\R^d}\|x-y\|^p\dd\gamma(x,y)\nonumber\\
    &=\int_{\R^d}\|X_f(x,1)-X_g(x,1)\|^p\dd\pi(x)\nonumber\\
    &\le |\supp(\pi)|(\|f-g\|_{C^0(\R^d\times\R)} e^{L})^p.
  \end{align*}
  The case $p=\infty$ is obtained with the usual adjustment of arguments.
\end{proof}

\subsection{KL-divergence}
In this subsection, we 
measure the distance in the KL-divergence, i.e.\
$\mathcal{D} = \mathcal{D}_{KL}$ in \eqref{Jobjective}. Unlike for the Wasserstein distance,
for $\mathcal{D}_{KL}(X(\cdot,1)_\sharp\pi,\rho)$ to be finite, we need in particular $X(\cdot, 1)_\sharp\pi\ll \rho$.
We restrict ourselves to distributions on cubes, and consider $\Omega_0 = \Omega_1 = [0,1]^d$.

Most of the work regarding the approximation distributions in KL-divergence using ODE flow maps has already been studied in our companion paper \cite{StatisticalNODE}; we include some of the relevant results here for the sake of completeness. In \cite{StatisticalNODE}, an 
ansatz space
\begin{align*}
  C^k_{\rm ansatz}(r) =  &\{f = (f_1,...,f_d)^T: f_j = \tilde{f}_jx_j(1-x_j), \tilde{f}_j\in C^k([0,1]^d\times[0,1], [0,1]^d)\} \\
  &\cap \{f\in C^2([0,1]^d\times[0,1], [0,1]^d): \|f\|_{C^2([0,1]^d\times[0,1], [0,1]^d)}\leq r\}
\end{align*}
was proposed. Its definition ensures that all push-forward distributions $X_f(\cdot,1)_\sharp\pi$ are supported on $[0,1]^d$ for any $f\in C^k_{\rm ansatz}(r)$. 
In Theorem \ref{thm:tri}, we showed the velocity field corresponding to the straight-line interpolation of Knothe-Rosenblatt maps $f^\Delta$ lies in $C^k([0,1]^d\times[0,1], [0,1]^d)$ and in Theorem 9 of \cite{StatisticalNODE}, it is shown that 
\begin{equation*}
  \frac{f^\Delta_j(x_1,\cdots,x_j)}{x_j(1-x_j)}\in C^k([0,1]^d\times[0,1], [0,1]).
\end{equation*}
  Therefore, by choosing $r$ to be large enough, for example, by taking $r= C_{k,d}(\frac{L_1}{L_2})^{\beta_dk^{d+3}}$ corresponding to the upper bound in Theorem \ref{thm:tri}, it suffices to consider an approximating element in $C^k_{\rm ansatz}(r)$.

We emphasize that bounding discrepancies in Wasserstein distance only requires $C^0$ control of the velocity fields; however, controlling discrepancies in the KL-divergence, requires $C^1$ control of the velocity fields, as stated in the next theorem.

\begin{theorem}\label{Thm:KLstatbility}
Let $\pi$, $\rho$ 
satisfy Assumptions \ref{ass:dens1} and \ref{ass:dens2} with $\Omega_0 = \Omega_1 = [0,1]^d$. Let $f^\Delta$ as in Theorem \ref{thm:tri} and  $g\in C^2_{ansatz}(r)$. 
Then
\[\mathcal{D}_{KL}(X_g(\cdot, 1)_\sharp\pi, X_{f^\Delta}(\cdot, 1)_\sharp\pi) \leq C \|f^\Delta - g\|^2_{C^1([0,1]^d)},\]
for some constant $C$ that depends on $L_1, L_2, d$. 
\end{theorem}
\begin{proof}
By Lemma 6, Theorem 7 and Theorem 8 of \cite{StatisticalNODE} there exists a constant $C_{L_1,L_2, d}$ that depends on $L_1, L_2, d$ such that 

\[\|X_g(\cdot, 1)_\sharp\pi - X_{f^\Delta}(\cdot, 1)_\sharp\pi\|_{C^0([0,1]^d)}  = \|X_g(\cdot, 1)_\sharp\pi - \rho\|_{C^0([0,1]^d)}\leq C_{L_1,L_2, d}\|f^\Delta - g\|_{C^1([0,1]^d)}.\] 

To get an upper bound for $\mathcal{D}_{KL}(X_g(\cdot,1)_\sharp\pi, X_{f^\Delta}(\cdot,1)_\sharp\pi)$, we bound
\begin{align*}
&\mathcal{D}_{KL}(X_g(\cdot,1)_\sharp\pi, X_{f^\Delta}(\cdot,1)_\sharp\pi) = \mathbb{E}_{\pi_g(\cdot,1)}\Big[\log \frac{\pi_g(\cdot,1)}{\rho(x)}\Big] \leq \log \mathbb{E}_{\pi_g(x,1)}[\frac{\pi_g(x,1)}{\rho(x)}] = \log\int_{[0,1]^d} \frac{\pi_g(x,1)^2}{\rho(x)}dx \\
&= \log\left(\int_{[0,1]^d} \frac{(\rho(x)-\pi_g(x,1))^2}{\rho(x)}dx + 1\right)  \leq \log\left(\frac{C_{L_1,L_2, d}^2}{L_2} \|f^\Delta - g\|^2_{C^1([0,1]^d)}+1\right).     
\end{align*}
The fact that $\log(1+x)\le x$ for all $x\ge 0$ gives the result.
\end{proof}

\section{Neural network approximation}\label{sec:NN}
In Section \ref{sec:trireg} we studied the regularity of the velocity
field corresponding to straight-line trajectories realizing the
Knothe-Rosenblatt map at time $t=1$. Building on earlier works on
neural network approximation theory such as
\cite{NNApproximation1,NNApproximation4, StatisticalNODE}
in the present section we conclude that by parameterizing the velocity
field via neural networks, NODE flows can achieve arbitrary accuracy
in terms of the Wasserstein distance and KL-divergence. Furthermore,
given a desired accuracy $\eps>0$, we give upper bounds on the
required network depth, width, and size in terms of $\eps$. Since the objective functional $J$ contains first-order derivatives, we shall consider the approximation theory using ReLU$^2$ networks developed in \cite{StatisticalNODE}. 

We first recall the definition of ReLU$^2$ networks.

\begin{definition}\label{def:ReLUm}
  Denote $\sigma_2(x):=\max\{0,x\}^2$ and let
  $d_1,d_2\ge 1$. Then, the class of ReLU$^2$
  networks mapping from $[0,1]^{d_1}$ to $\mathbb{R}^{d_2}$, with
  depth
  $L$, width $W$, sparsity 
  $S$, and 
  bound $B$ on the network weights, is denoted
  
  \begin{align*}
    \Phi^{d_1,d_2}(L, W, S, B) = \Big\{&\big(W^{(L)}\jz{\sigma_2}(\cdot) + b^{(L)}\big)\circ\cdots\circ\big(W^{(1)}\jz{\sigma_2}(\cdot) + b^{(1)}\big): W^{(L)} \in \mathbb{R}^{1\times W},\\
                                 &b^{(L)}\in\mathbb{R}^{d_2}, W^{(1)}\in\mathbb{R}^{W\times d_1},b^{(1)}\in\mathbb{R}^W, W^{(l)}\in\mathbb{R}^{W\times W},b^{(l)}\in\mathbb{R}^W (1<l<L),\\
                                 &\sum_{l=1}^L\big(\|W^{(l)}\|_0 + \|b^{(l)}\|_0\big) \leq S, \max_{1\leq l\leq L} \big(\|W^{(l)}\|_{\infty,\infty}\lor \|b^{(l)}\|_\infty\big)\leq B \Big\}.\\    
  \end{align*}
\end{definition}
\begin{remark}
  Since the representation of a function via neural networks is not
  unique in general, the statement ``$g$ is a ReLU$^2$ NN of depth $L$,
  width $W$, sparsity $S$, norm bound $B$'' merely implies the existence of such a network satisfying the above
  properties. Possibly, for some other $\tilde L$, $\tilde W, \tilde S$
  and $\tilde B$, it may additionally hold that ``$g$ is a ReLU$^2$
  NN of depth $\tilde L$, width $\tilde W$, sparsity $\tilde S$, and norm bound $\tilde B$''.
\end{remark}


\subsection{Wasserstein distance}


As noted before, approximation in Wasserstein distance only 
requires $C^0$, 
rather than $C^1$, control of the velocity fields. In terms of the Wasserstein distance, we have the following result,
which is a consequence of our regularity analysis of the velocity
field in Theorem \ref{thm:tri}, Corollary \ref{cor:Wp} and neural
network approximation results as e.g.\ presented in
\cite{StatisticalNODE}:
\begin{proposition}\label{prop:WpNN}
  Let $k\ge 1$, $p\in [1,\infty]$, and let $\rho$, $\pi$ be two
  probability distributions on $[0,1]^d$
  with Lebesgue densities in 
  $C^k([0,1]^d)$ satisfying Assumption \ref{ass:dens2}.
  Then there exist constants $C_{d,k}$ and $C_{d, k, L_1, L_2}'$ such that for every $\eps\in (0,1]$ there exists a ReLU$^2$ neural network $g\in\Phi^{d+1,d}(L, W, S, B)$ with
  \begin{equation*}\label{eq:NNbounds1}
  L \leq C_{d,k},\quad
  W \leq C_{d, k, L_1, L_2}'\eps^{-\frac{d+1}{k}},\quad
  S\leq C_{d, k, L_1, L_2}'\eps^{-\frac{d+1}{k}},\quad
  B\leq C_{d, k, L_1,L_2}'\eps^{-\frac{1}{k}}
\end{equation*}
such that for another constant $C_{d, k, L_1, L_2}$, we have
  \begin{equation*}
    W_p(\rho,X_g(\cdot,1)_\sharp\pi)\le C \eps.
  \end{equation*}
\end{proposition}
\begin{proof}
  According to Theorem \ref{thm:tri} there exists
  ${f^\Delta}\in C^k([0,1]^d\times [0,1])$ such that
  $X_{f^\Delta}(\cdot,1)_\sharp\pi=\rho$ and $\|{f^\Delta}\|_{C^k([0,1]^d\times [0,1])}$
  only depends on $L_1$, $L_2$, $k$, $d$. We can extend ${f^\Delta}$ to some
  $\tilde {f^\Delta}\in C^k(\R^d\times [0,1])$ with compact support and such
  that
  \begin{equation*}
    \|\tilde {f^\Delta}\|_{C^k(\R^d\times [0,1])}\le C \|{f^\Delta}\|_{C^k(\Omega_{[0,1]})}
  \end{equation*}
  for some $C$ solely depending on $d$ (see for example Step 1 of the
  proof of \cite[Theorem 33]{StatisticalNODE}). Since $\tilde {f^\Delta}$ has compact
  support and belongs to $C^1(\R^d\times [0,1])$, there exists $L<\infty$ such that
  $x\mapsto \tilde {f^\Delta}(x,t):\R^d\to\R^d$ has Lipschitz constant $L$ for all
  $t\in [0,1]$. Again, $L$ solely depends on
  $\|{f^\Delta}\|_{C^1([0,1]^d\times [0,1])}$ and thus on $L_1$, $L_2$, $k$,
  $d$. Next, let $M=1+\exp(L)$. 

  According to
  \citet[Theorem 16]{StatisticalNODE}, there exists a ReLU$^2$ neural
  network $g$ satisfying the bounds \eqref{eq:NNbounds1} and
  \begin{equation}\label{eq:f-geps}
\|{\tilde f^\Delta}(x)-g(x)\|_{C^0([-M,M]^d\times [0,1])}\le\eps.
  \end{equation}

  Fix $x\in [0,1]^d$. By \eqref{eq:FlowMapBound}
  and \eqref{eq:f-geps}, we
  have for all $t\in [0,1]$
    \begin{equation}\label{eq:XfDXg}
      \|X_{\tilde f^\Delta}(x,t)-X_g(x,t)\|\le \eps\exp(L)\le\exp(L)
    \end{equation}
    and thus $\|X_{g}(x,t)\|\le 1+\exp(L)=M$.
  Hence $X_g(x,1)\in\R^d$ is well-defined since the trajectory
  $t\mapsto X_g(x,t)$, $t\in[0,1]$, remains within
  $[-M,M]^d\times [0,1]$. 
  Finally, the first inequality in \eqref{eq:XfDXg} and an application of
  Corollary \ref{cor:Wp} concludes the proof.
\end{proof}

Next we discuss convergence of the objective $J$ defined in \eqref{Jobjective},
for the Wasserstein distance; specifically
  \begin{equation*}
    J_{\rm W}(f) := W_2(\rho,X_f(\cdot,1)_\sharp\pi)^2+R(f)
  \end{equation*}
  where
  \begin{equation}\label{eq:Rf}
    R(f) = \int_0^1\|\nabla_X f(X_f(x,t)) f(X_f(x,t),t)+\partial_t f(X_f(x,t),t)\|^2\dd t.
  \end{equation}
  Since the regularization term $R(f)$ contains first order
  derivatives of $f$, it will not suffice to have uniform
  approximation in $f$, but rather we'll additionally need uniform
  approximation of the derivatives of $f$.
We also point out that by Theorem \ref{thm:ExSol}, $\inf_{f\in\mathcal{V}}J_{\rm W}(f)=0.$

  \begin{theorem}\label{Thm:JBoundWasserstein}
  Let $k\ge 2$, $p\in [1,\infty]$ and let $\rho$, $\pi$ be two
  probability distributions on $[0,1]^d$ with Lebesgue densities in
  $C^k([0,1]^d)$ satisfying Assumption \ref{ass:dens2}.
  
  Then there exist constants 
  $C_{d,k}$ and $C_{d, k, L_1, L_2}'$ such that
  for every $\eps\in (0,1]$
  there exists a ReLU$^2$ neural network $g\in\Phi^{d+1,d}(L, W, S, B)$ with
  \begin{equation}\label{eq:NNbounds2}
  L \leq C_{d,k},\quad
  W \leq C_{d, k, L_1, L_2}'\eps^{-\frac{d+1}{k-1}},\quad
  S\leq C_{d, k, L_1, L_2}'\eps^{-\frac{d+1}{k-1}},\quad
  B\leq C_{d, k, L_1,L_2}'\eps^{-\frac{1}{k-1}}
\end{equation}
  and such that for another constant $C_{d,k, L_1, L_2}$, we have 
  \begin{equation*}
    J_{\rm W}(g)\le 
    C_{d,k,L_1,L_2} \eps^2.
  \end{equation*}
\end{theorem}
\begin{proof}
  According to Theorem \ref{thm:tri}, there exists
  ${f^\Delta}\in C^k([0,1]^d\times [0,1])$ such that
  $X_{f^\Delta}(\cdot,1)_\sharp\pi=\rho$ (i.e.\
  $W_2(\rho,X_{f^\Delta}(\cdot,1)_\sharp\pi)=0$), ${f^\Delta}$ realizes straight line
  trajectories (i.e.\ $R({f^\Delta})=0$) and $\|{f^\Delta}\|_{C^k([0,1]^d\times [0,1])}$
  only depends on $L_1$, $L_2$, $k$, $d$. In particular ${f^\Delta}$ can be
  extended to an element of $\mathcal{V}$ in \eqref{eq:V}
  such that $J_{\rm W}({f^\Delta})=0$. Hence it
  suffices to show existence of $g$ such that $J_{\rm W}(g)\le C\eps^2$.

  We can extend ${f^\Delta}$ to some $\tilde {f^\Delta}\in C^k(\R^d\times [0,1])$ with
  compact support and such that
  \begin{equation*}
    \|\tilde {f^\Delta}\|_{C^k(\R^d\times [0,1])}\le C \|{f^\Delta}\|_{C^k([0,1]^d\times[0,1])}
  \end{equation*}
  for some $C$ solely depending on $d$ (see for example Step 1 of the
  proof of \cite[Theorem 33]{StatisticalNODE}) 
  Since $\tilde {f^\Delta}$ has
  compact support and belongs to $C^k$, $k\ge 2$, there exists
  $L<\infty$ such that the three maps
  \begin{equation*}
    x\mapsto \tilde {f^\Delta}(x,t),\qquad
    x\mapsto \partial_t\tilde {f^\Delta}(x,t),\qquad
    x\mapsto \nabla_x\tilde {f^\Delta}(x,t),\qquad
  \end{equation*}
  have Lipschitz constant $L$ for all $t\in [0,1]$ and all
  $x\in\R^d$. Again, $L$ solely depends on
  $\|{f^\Delta}\|_{C^k([0,1]^d\times [0,1])}$ and thus on $L_1$, $L_2$, $k$,
  $d$.

  Next, let $M>\exp(L)$ be so large that
  $[0,1]^d \subseteq [-M+\exp(L),M-\exp(L)]^d$.  According to
  \cite[Theorem 16]{StatisticalNODE},
  there exists a ReLU$^2$ neural
  network $g$ satisfying the bounds \eqref{eq:NNbounds2} and 
  \begin{equation}\label{eq:f-geps2}
    \|{\tilde f^\Delta}(x)-g(x)\|_{C^1([-M,M]^d\times [0,1])}\le\eps.
  \end{equation}
  As in the proof of Proposition \ref{prop:WpNN}, we conclude that
  $X_g(x,1)\in\R^d$ is well-defined and
  \begin{equation*}
    \|X_{\tilde f^\Delta}(x,t)-X_g(x,t)\|\le \eps \exp(L)\qquad\forall x\in[0,1]^d,~t\in [0,1].
  \end{equation*}
  Since the trajectories $t\mapsto X_{\tilde f^\Delta}(x,t)$ remain in
  $[0,1]^d\times [0,1]$ according to Proposition \ref{prop:cube}, we conclude that
  \begin{equation}\label{eq:trajectoriesM}
    X_g(x,t)\in [-M,M]^d\qquad\forall x\in [0,1]^d,~t\in [0,1].
  \end{equation}
  Moreover, as in the proof of Proposition \ref{prop:WpNN} holds
  \begin{equation*}
    W_2(\rho,X_g(\cdot,1)_\sharp\pi)^2\le C\eps^2.
  \end{equation*}

  To bound $J_{\rm W}(g)$, it remains to treat the term
  $R(g)$ in \eqref{eq:Rf}. We have
  \begin{align*}
    \|\nabla_X \tilde f^\Delta(X_{\tilde f^\Delta}(x,t))-\nabla_Xg(X_g(x,t))\|
    \le
    &\|\nabla_X\tilde {f^\Delta}(X_{\tilde f^\Delta}(x,t))-\nabla_X\tilde {f^\Delta}(X_g(x,t))\|\nonumber\\
    &+\|\nabla_X\tilde {f^\Delta}(X_g(x,t))-\nabla_Xg(X_g(x,t))\|.
  \end{align*}
  Since $\nabla_X\tilde {f^\Delta}$ has Lipschitz constant $L$,
  The first term on the right-hand side is bounded by $L\exp(L)\eps$.
  Due to $X_g(x,t)\in [0,1]^d\subseteq [-M,M]^d$ by \eqref{eq:trajectoriesM},
  we can use \eqref{eq:f-geps2} to bound the second term which gives
  \begin{equation*}
    \|\nabla_X\tilde {f^\Delta}(X_{\tilde f^\Delta}(x,t))-\nabla_Xg(X_g(x,t))\|
    \le L\exp(L)\eps+\eps
    \qquad\forall x\in[0,1]^d,~t\in [0,1].
  \end{equation*}
  Similarly we obtain
  \begin{equation*}
    \|\partial_t \tilde { f^\Delta}(X_{\tilde f^\Delta}(x,t))-\partial_t g(X_g(x,t))\|
    \le L\exp(L)\eps+\eps
    \qquad\forall x\in[0,1]^d,~t\in [0,1]
  \end{equation*}
  and
  \begin{equation*}
    \|\tilde {f^\Delta}(X_{\tilde f^\Delta}(x,t))-g(X_g(x,t))\|
    \le L\exp(L)\eps+\eps \qquad\forall x\in[0,1]^d,~t\in [0,1].
  \end{equation*}
  Therefore
  \begin{align*}
    R(g)=|R({f^\Delta})-R(g)|
    &\le \int_0^1\|\nabla_X \tilde f(X_{\tilde f^\Delta}(x,t)) \tilde {f^\Delta}(X_{\tilde f^\Delta}(x,t),t)+\partial_t \tilde f(X_{\tilde f^\Delta}(x,t),t) \\
    &\qquad\qquad-\nabla_X g(X_g(x,t)) g(X_g(x,t),t)-\partial_t g(X_g(x,t),t)\|^2\dd t\\
    &\le C \int_0^1 \eps^2\dd t,
  \end{align*}
  where $C$ only depends on $L$.
\end{proof}

\subsection{KL-divergence} 

  

Now we discuss convergence of the objective $J$ defined in \eqref{Jobjective} for case where $\mathcal{D}$ is the KL divergence, i.e., 
  \begin{equation*}
    J_{\rm KL}(f) := {\rm KL}( X_f(\cdot,1)_\sharp \pi \vert \vert \rho ) + R(f).
  \end{equation*}
Note again that by Theorem \ref{thm:ExSol}, we have $\inf_{f\in\mathcal{V}} J_{\rm KL}(f)=0$. 
  
For the KL-divergence to be well-defined, we need to enforce the constraint that the ODE flow map is a diffeomorphism onto $[0,1]^d$. Therefore, we shall use the neural network based ansatz space introduced in \cite{StatisticalNODE}.
To this end define $\chi_d(x_1,\dots x_d): D\rightarrow D$ via
$$\chi_d(x_1,\dots x_d) = [x_1(1-x_1),\dots ,x_d(1-x_d)]^T.$$ Letting
$\otimes$ be the coordinate-wise (Hadamard) multiplication of two vectors,
for any velocity field
$f: [0,1]^d\times[0,1]\rightarrow \mathbb{R}^d$,
$f\otimes\chi_d$ then yields a vector field with vanishing
normal components at the boundary of $[0,1]^d$ at any 
time $t\in [0,1]$.  
We define our neural network ansatz space as
          \begin{align}\label{eq:FNN}
            &\mathcal{F}_{\text{NN}}(L,W, S, B, r) :=\nonumber\\
            &\qquad\set{f^{\text{NN}}(x_1,\ldots, x_d,t)\otimes\chi_d(x_1,\ldots, x_d)}{f^\text{NN}\in\Phi^{d+1,d}(L, W, S, B),~
              \|f\|_{W^{2,\infty}(\Omega)}\leq r}.
          \end{align}

\begin{proposition}\label{Prop:KLerrorNN}
    
Let $k \geq 2$ and $\rho, \pi$ two probability distributions satisfying Assumptions \ref{ass:dens1} and \ref{ass:dens2} with $\Omega_0 = \Omega_1 = [0,1]^d$.

Then there exists constants $C_{d,k}$
and $C_{d,k,L_1,L_2}'$
such that for every $\eps\in(0,1]$, there exists a ReLU$^2$ network $g$ in the ansatz space $\mathcal{F}_{NN}(L, W, S, B, r) $ with parameters satisfying
\begin{equation*}
  L \leq C_{d,k},\quad
  W \leq C_{d, k, L_1, L_2}'\eps^{-\frac{d+1}{k-1}},\quad
  S\leq C_{d, k, L_1, L_2}'\eps^{-\frac{d+1}{k-1}},\quad
  B\leq C_{d, k, L_1,L_2}'\eps^{-\frac{1}{k-1}},\quad
  r\leq C_{d, k, L_1, L_2}'
\end{equation*}
  such that for another constant $C_{d, k, L_1, L_2}$, we have 
\[ \mathrm{KL}(X_g(\cdot,1)_\sharp \pi, \rho)\leq C_{d, k ,L_1, L_2}'\eps^2.\]
\end{proposition}
\begin{proof}
According to Theorem \ref{thm:tri}, there exists
  $f^\Delta\in C^k([0,1]^d\times [0,1])$ such that
  $X_{f^\Delta}(\cdot,1)_\sharp\pi=\rho$ (i.e.\
  $\mathcal{D}_{KL}(\rho,X_{f^\Delta}(\cdot,1)_\sharp\pi)=0$), $f^\Delta$ realizes straight line
  trajectories (i.e.\ $R(f^\Delta)=0$) and $\|f^\Delta\|_{C^k([0,1]^d\times [0,1])}$
  only depends on $L_1$, $L_2$, $k$, $d$. By Step 1 of \cite[Theorem 20]{StatisticalNODE}, for all $N \geq 1$,  there exists 
  $g\in \mathcal{F}_{NN}(L, W, S, B, r)$ with
  \begin{equation*}
    L \leq C_{d, k},\quad W \leq N,\quad S \leq N,\quad B \leq C_{d,k}\|f^\Delta\|_{C^k([0,1]^d\times [0,1])} + N^{\frac{1}{d+1}},\quad
    r \leq C_{d, k, L_1, L_2},
  \end{equation*}
  such that $\|f^\Delta- g\|_{C^1([0,1]^d\times[0,1])} \leq C_{d, k, L_1, L_2}N^{-\frac{k-1}{d+1}}$, where $C_{d,k}$ and $\tilde C_{d, k, L_1, L_2}$ are constants depending on $d, k$ and $d, k, L_1, L_2$ respectively. Letting $\tilde C_{d, k, L_1, L_2}N^{-\frac{k-1}{d+1}} = \eps$, we solve for $N$ to get $N = C_{d, k, L_1, L_2}\eps^{-\frac{d+1}{k-1}}$
  for some $C_{d,k,L_1,L_2}$. By Theorem \ref{Thm:KLstatbility}, we then have $\mathcal{D}_{KL}(X_g(\cdot,1), \rho) \leq C_{d, k, L_1, L_2}\eps^2$. 
\end{proof}

\begin{theorem}\label{Thm:JBoundKL}
Let $k \geq 2$ and $\rho, \pi$ two probability distributions satisfying Assumptions \ref{ass:dens1} and \ref{ass:dens2} with $\Omega_0 = \Omega_1 = [0,1]^d$.

Then there exist constants $C_{d,k}$ and $C_{d, k, L_1, L_2}'$ such that for every $\eps\in(0,1]$, there exists a ReLU$^2$ network $g$ in the ansatz space $\mathcal{F}_{NN}(L, W, S, B, r) $ with parameters satisfying
\begin{equation*}
  L \leq C_{d,k},\quad
  W \leq C_{d, k, L_1, L_2}'\eps^{-\frac{d+1}{k-1}},\quad
  S\leq C_{d, k, L_1, L_2}'\eps^{-\frac{d+1}{k-1}},\quad
  B\leq C_{d, k, L_1,L_2}'\eps^{-\frac{1}{k-1}},\quad
  r\leq C_{d, k, L_1, L_2}'
\end{equation*}
  such that for another constant $C_{d, k, L_1, L_2}$, we have
\[J(g) \leq 
C_{d, k ,L_1, L_2}\eps^2.\]
\end{theorem}
\begin{proof}
The proof follows 
by the same arguments as the proof of Theorem \ref{Thm:JBoundWasserstein}. The error in the KL-divergence part of the objective follows from Proposition \ref{Prop:KLerrorNN} and the error in minimal-energy regularization has already been 
treated in Theorem~\ref{Thm:JBoundWasserstein}.    
\end{proof}

\begin{remark}
We comment here that the estimates we obtained above are $L^\infty$ in nature, thus we are able to derive distribution error estimates in other metrics as well, such as Hellinger, chi-square and total variation.  
\end{remark}

\section{Discussion and future work}
Our work is a crucial first step towards establishing a theoretical framework for sampling and distribution learning through ODE flow maps. In particular, the approximation results in this work can be viewed as quantifying the \textit{bias} term in the classical context of statistical learning theory. In our parallel work \cite{StatisticalNODE}, we explore the \textit{variance} term by analyzing the statistical complexity of the function class represented by ODE flow maps whose velocity fields come from a bounded neural network class. 

While our work establishes a theoretical framework for analyzing ODE-based models, several important questions still remain open. First, we only consider distributions supported on bounded domains because the Lipschitz constant of straight-line ansatz considered in this work and \cite{StatisticalNODE} can be uncontrollable when the distributions are not lower bounded. It will be interesting to see how our theories could be extended to the case of unbounded domains. Second, the approximation errors obtained in this work depend on $d+1$ and $k-1$, because we considered the velocity field as a generic function on $\mathbb{R}^{d+1}$ and the approximation error metric considered is in $C^1$ for distributional stability results. This approximation error leads to nonparametric convergence rate of $n^{-\frac{2(k-1)}{d+1 + 2(k-1)}}$ in \cite{StatisticalNODE}. We note this statistical rate obtained is suboptimal, compared to the minimax optimal rate of learning $k$-smooth densities on $d$-dimensional domains (which is $n^{-\frac{2k}{d + 2k}}$). Therefore, an important question to ask is whether neural-ODE based models can achieve the same minimax optimal statistical rates, as classical density estimators such as wavelets or kernel-based methods. We leave these open questions to future studies.

\section*{Acknowledgments and disclosure of funding}
ZR and YM acknowledge support from the US Air Force Office of Scientific Research (AFOSR) MURI, Analysis and Synthesis of Rare Events, award number FA9550-20-1-0397, and from the US Department of Energy (DOE), Office of Advanced Scientific Computing Research, under grants DE-SC0021226 and DE-SC0023187. ZR also acknowledges support from a US National Science Foundation Graduate Research Fellowship.

\appendix

\section{Comments on training}\label{app:training}
\subsection{Training algorithm}

Here we present pseudocode for a notional training algorithm, i.e., an algorithm to learn the velocity field $f$ of a neural ODE in the setting of problem P1 (see Section \ref{sec:setup}), when the information divergence used is the KL-divergence. 

Recall the training objective we consider consists of two parts: one is the KL divergence from the source distribution to the pushforward of the target by the time-one flow map of the ODE; the other is the integration of \eqref{eq:straightlinereg} along the trajectories of
the particles.

Given a distribution $\pi$ for the initial condition $x$, 
the KL divergence from the source distribution to $\pi_{f,t}$ at time $t=1$ can be written as
\begin{align*}
    \dkl( \pi_{f,1} \, || \, \rho ) = \dkl ( X_f(\cdot, 1)_\sharp \pi \, || \, \rho )
    & = \dkl ( \pi \, || \, X_f^{-1}(\cdot, 1)_\sharp \rho ) \\
    & = \mathbb{E}_{x \sim \pi} \left [ \log \pi(x) - \log \rho ( X_f(x,1) ) - \log \det \nabla_x X_f(x,1) \right ].
\end{align*}
where $X_f^{-1}(\cdot, 1)$ denotes the inverse of $x \mapsto X_f(x,1)$. The log determinant term above can be computed from the
instantaneous change of variables formula (e.g., \citet{NeuralODE}): $\log \det \nabla_x X(x,1) = - \int_0^1 \tr \left ( \nabla_X f(X(x,t),t) \right ) dt$. 

The combined training objective becomes:
\begin{equation}\label{eq:OBJ}
  \begin{aligned}
    J(f) &= 
    \mathbb{E}_{x\sim\pi} \left [ \log\pi(x) 
    - \log\rho(X(x,1)) +  \int_0^1 \tr \left( \nabla_X f(X(x,t),t) \right) dt + \lambda R(x,1) \right ]\\
  \end{aligned}
\end{equation}
where $\lambda$ controls the impact of penalization.

To evaluate the optimization objective, we need the ability to compute
the matrix $\nabla_X f(X(x,t),t)$, as its trace appears in the change-of-variables term and the entire matrix
appears in the regularization term. Also, we need the ability to compute
$\partial_t f(X(x,t),t)$. We can assemble these two terms
into a full Jacobian matrix, which we denote by
$\nabla_{X,t}f(X(x,t),t)$; in practice, this is the Jacobian of a neural network with respect to all of its inputs. With the
discretize-then-optimize approach of \cite{OTFlow}, we can compute
this matrix exactly via automatic differentiation. For details, see Appendix \ref{app:trainingJacobian}.

\begin{algorithm}
  \caption{Neural ODE training, problem P1}
  \label{algo1}
  \begin{algorithmic}[1]
    \STATE \textbf{Input}: sample $\mathcal{X} = \{ x_i\}_{i=1}^N$ from target
    measure $\pi$, parameterized neural network $f(x,t ; \theta)$,
    regularization parameter $\lambda$, source measure $\rho$.
    \STATE\textbf{Initialize} $\theta$, \WHILE{$\theta$ not converged}
    \STATE{Sample minibatch $\{ x_j\}$ of size $m$ from $\mathcal{X}$}
    \STATE{Set $ x_j(0) = x_j$, $l_j(0) = r_j(0)=0$ } 
    \STATE{Solve the
      following ODE system up to time $t=1$ 
      \begin{align*}
        \frac{dx_j}{dt} &= f(x_j,t;\theta)\\
        \frac{dl_j}{dt} &= -\tr(\nabla_xf(x_j,t;\theta))\\
        \frac{dr_j}{dt} &= |\nabla_xf(x_j,t;\theta)f(x_j,t;\theta) +
        \partial_tf(x_j,t)|^2
      \end{align*}
      } 
      \STATE{Compute the loss
      $L(\theta) =
      \frac{1}{m}\sum_{j=1}^m-\log\rho(x_j(1))-l_j(1)+\lambda r_j(1)$}
    \STATE{Use automatic differentiation to backpropagate and
      update $\theta$} \ENDWHILE
  \end{algorithmic}
\end{algorithm}




In practice, we only have access to  finite samples
from the target measure, so we replace the population risk in
objective \eqref{eq:OBJ} with an empirical risk based on this
sample. Moreover, since $\log \pi(x)$ is independent of the velocity field $f$, it can be ignored in the optimization procedure. Hence, we arrive at the following \textit{empirical risk minimization} problem:
\begin{equation}\label{ERM}
  \begin{aligned}
    J_\text{ERM}(f) 
    &= \frac{1}{N}\sum_{i=1}^N\left( -\log\rho(X(x_i,1)) +
      \int_0^1 \tr \left ( \nabla_X f(X(x_i,t),t) \right )dt + \lambda R(x_i,1)\right) \, .
  \end{aligned}\tag{ERM}
\end{equation}
Putting everything together, we have Algorithm~\ref{algo1} in Appendix~\ref{app:training}.

\subsection{Exact computation of the Jacobian}\label{app:trainingJacobian}
Training neural ODEs consists of minimizing the (regularized) loss over the network weights subject to the ODE constraint. The adjoint-based methods in \citet{NeuralODE}, \citet{ffjord}, and \citet{HowToTrain} can be viewed as an optimize-then-discretize approach: another continuous-time ODE (the \textit{adjoint} equation) provides exact gradients with respect to network weights. Both the forward and adjoint equations are then discretized, checkpointing is typically employed to reduce memory requirements, and some care is required to ensure consistency of gradients.
%
%
Alternatively, a discretize-then-optimize approach is proposed in \citet{DiscretizeOptimize} and \citet{OTFlow}, where one first discretizes the forward dynamics and computes gradients backwards in time via automatic differentiation.
%
%
Since the training objective proposed in our work involves the entire Jacobian matrix of the velocity field, it is natural to use the discretize-then-optimize approach, which allows exact computation of the Jacobian. As in \cite{OTFlow}, the velocity field can be implemented as a ResNet and we can compute the Jacobian recursively. 

Let $s = (x,t)$ be the new variable formed by appending the time variable to the space variable. We then have the following recursive relation in a $M$-layer ResNet, where the $\{u_i\}$ are the outputs from intermediate layers:
\begin{align*}
&u_0 = \sigma(K_0s + b_0)\\
&u_1 = u_0 + h\sigma(K_1u_0 + b_1)\\
& \vdots \\
& u_M = u_{M-1} + h\sigma(K_Mu_{M-1} + b_M)\\
\end{align*}
Taking the gradient with respect to variable $s$, we have $\nabla_s u_i^T = \nabla_s u_{i-1} + h\sigma'(K_iu_{i-1} + b_i)K^T_i\nabla_s u_{i-1}$. 
Therefore, we have the following update rule for the Jacobian: $$J \leftarrow J+h\sigma'(K_iu_{i-1} + b_i)K_i^TJ.$$	
The network parameters $K_i$ and $b_i$ are to be learned. Since we use a discretized version of the velocity field in the implementation, these parameters can be updated through automatic differentiation.

\section{Knothe--Rosenblatt construction of triangular transport maps}\label{app:KRMap}	
Given probability measures $\rho$ and $\pi$, the Knothe--Rosenblatt transport is, under appropriate conditions, the unique triangular monotone transport $T$ such that $T_\sharp\pi = \rho$. In this section, we describe the explicit Knothe--Rosenblatt construction of triangular transport maps, as presented in \cite{OTAppliedMathematician}. Let $d\in\mathbb{N}$ be the dimension. For simplicity of presentation, we assume that $\pi$ and $\rho$ are supported on the hypercube $[0,1]^d$. Let $\mu$ be a base measure (for example the Lebesgue measure) and assume that $\frac{d\pi}{d\mu} = \pi(x)\in C^0([0,1]^d, \mathbb{R}^+)$ and $\frac{d\rho}{d\mu} = \rho(x)\in C^0([0,1]^d, \mathbb{R}^+)$ are the corresponding densities. Assume also that the densities $\rho(x)$ and $\pi(x)$ are uniformly bounded from below by a positive constant. 

For a continuous density function $f \in \{\rho, \pi\}$, we define the following auxiliary functions for $x\in[0,1]^k$, $k\leq d$:
\begin{equation}
\begin{aligned}
\hat{f}_k(x) &= \int_{[0,1]^{d-k}}f(x, t_{k+1},\dots ,t_d)d\mu((t_j)_{j=k+1}^d)\\
f_k(x) &= \frac{\hat{f}_k(x)}{\hat{f}_{k-1}(x_{[k-1]})}.
\end{aligned}
\end{equation}
Hence $f_k(x_{[k-1]}, \cdot)$ is the marginal density of the variable $x_k$ conditioned on $x_{[k-1]} = (x_1,\dots ,x_{k-1})\in[0,1]^{k-1}$.

Then, we define the corresponding CDFs:
\begin{equation}\label{eq:Fpik}
\begin{aligned}
F_{\pi,k}(x_{[k-1]}, x_k) &= \int_{0}^{x_k}\pi_k(x_{[k-1]}, t_{k})d\mu(t_k)\\
F_{\rho,k}(x_{[k-1]}, x_k) &= \int_{0}^{x_k}\rho_k(x_{[k-1]}, t_{k})d\mu(t_k),
\end{aligned}
\end{equation}
which are well-defined for $x\in[0,1]^k$ and $k\in\{1,\dots ,d\}$. Note that these are interpreted as functions of the last variable $x_k$ with $x_{[k-1]}$ fixed. In particular, we let $F_{\rho,k}(x_{[k-1]}, \cdot)^{-1}$ be the inverse of the map $x_k\rightarrow F_{\rho,k}(x_{[k-1]}, x_k)$

For $x \in [0,1]^d$, the Knothe--Rosenblatt map is constructed recursively in the following way. First, define
$$T_1(x_1) = F_{\rho,1}^{-1}\circ F_{\pi,1}(x_1),$$
and for $k > 1$, define
\begin{equation}\label{eq:Tk}
  T_k(x_{[k-1]}, \cdot) = F_{\rho,k}(T_1(x_1),\dots ,T_{k-1}(x_{[k-1]}), \cdot)^{-1}\circ F_{\pi,k}(x_{[k-1]}, \cdot).
\end{equation}
Then the map $$T(x_1,\dots,x_d) =  \left ( T_1(x_1), T_2(x_{[2]}),\dots ,T_d(x_{[d]}) \right )$$
is the triangular Knothe--Rosenblatt transport $T:[0,1]^d\rightarrow[0,1]^d$, for which we have the following theorem:
\begin{theorem}\label{thm:TraingularPushForward}
The triangular Knothe--Rosenblatt map satisfies $T_\sharp\pi = \rho$ and  $\det\nabla T(x)\rho(T(x)) = \pi(x)$, $\forall x\in[0,1]^d$.
\end{theorem}
We comment that regularity assumptions for Theorem \ref{thm:TraingularPushForward} can be relaxed. For a more detailed discussion, see \cite{TrangularTransportofMeasure}. 

\section{
  Auxiliary results
}\label{app:AuxResults}
In this appendix, we collect statements and proofs 
that are 
required for the proofs of the main theorems. 

First, we present two technical results about domains. 
\begin{lemma}\label{lemma:convexLip}
(uniform-cone
characterization of convex domains) 
  Let $\Omega\subset\mathbb{R}^d$ be a bounded, convex, and open
  domain. Then $\Omega$ is a Lipschitz domain.
\end{lemma}	
\begin{proof}[Proof of Lemma \ref{lemma:convexLip}]
  Without loss of generality, we may assume that $0\in\Omega$. Since
  $\Omega$ is bounded and open, there exist $r, R > 0$ such that
  $B(0,r)\subset \Omega\subset B(0,R)$, where $B(0,r)$ and $B(0,R)$
  denote balls of radius $r$ and $R$ respectively.

  Then, we can cover the surface of the ball of radius $R$ by overlapping
  $d-1$ dimensional balls of radius $\epsilon$ such that the boundary
  of each such ball, $B_{d-1}(0, \epsilon)$, is completely covered by
  the adjacent balls. If $\vec{n}$ denotes the unit vector emanating
  from the origin in the direction of the center of such a ball, then
  $U = \setl{t\vec{n} + y}{t\geq 0, y\in B_{d-1}(0, \epsilon)}$ is the
  cylinder of radius $\epsilon$ whose intersection with $B(0,R)$ is
  the boundary of this ball.

  Since the surface of $B(0, R)$ can be covered by finitely many such
  $d-1$ dimensional balls, we can find a finite collection of such
  cylinders $\{U_j\}_{j=1}^J$ so that their union cover $\Omega$. From
  this construction of $\{U_j\}_{j=1}^J$, the first property in the
  definition of Lipschitz domain is clearly satisfied.

  To verify the second property, note that for each $j$, the
  coordinate system is simply the map that transforms the cylinder
  $U_j$ to align with the direction of $e_d$, where $e_d$ is the last
  vector of the standard basis of $\mathbb{R}^d$. For any
  $x \in \partial\Omega\cap U_j$, the cone defined by the convex
  closure of $\{x\} \cup B(0,r)$ is contained in the closure of
  $\Omega$ and the head angle $\bsalpha$ of the cone satisfies
  $\sin(\frac{\bsalpha}{2}) \geq \frac{r}{R}$, and thus the boundary
  is a Lipschitz function.
\end{proof}	

The image of a Lipschitz domain under a sufficiently regular map
  remains a Lipschitz domain:
\begin{theorem}[{\cite[Theorem ~4.1]{LipTransformation}}]\label{thm:LipTransformation}
  Assume $\Omega \subset\mathbb{R}^d$ is a bounded Lipschitz
    domain and $\mathcal{O}$ is an open neighborhood of $\bar{\Omega}$ and
  $f:\mathcal{O}\rightarrow \mathbb{R}^n$ is a $C^1$-diffeomorphism
  onto its image. Then, $\tilde{\Omega} = f(\Omega)$ is also a
  Lipschitz domain.
\end{theorem}
	
Next, we we shall state the following regularity result about the regularity of optimal transport map from \cite{OptimalRegularity}, which is used in the proof of Theorem \ref{thm:ExSol}.
\begin{theorem}\label{Thm:OptimalRegularity}
Fix open sets $\Omega_0, \Omega_1\subset\mathbb{R}^d$ with $\Omega_1$ convex and absolutely continuous measure $\mu, \nu$ with finite second moment and bounded, strictly positive densities $f, g$ respectively such that $\mu(\Omega_0) = \nu(\Omega_1) = 1$. Let $\phi$ be such that $\nabla\phi_\sharp\mu = \nu$. If $\Omega_0$ and $\Omega_1$ are bounded and $f, g$ bounded from below, then $\phi$ is strictly convex and of class $C^{1, \alpha}(\Omega_0)$ for some $\alpha > 0$. In addition, if $f, g\in C^{k, \alpha}$, then $\phi\in C^{k+2, \alpha}(\Omega_0)$. 
\end{theorem}

Then, we shall need the following results about composition and the inverse of a function in Section \ref{sec:regularity}.

  \begin{theorem}[Fa\'{a} di Bruno]
    Let $k\in\N$.
  Let $X$, $Y$ and $Z$ be three Banach spaces, and let
  $F\in C^k(X,Y)$ and $G\in C^k(Y,Z)$.

  Then for all $0\le n\le k$ and with
  $T_n := \setl{\bsalpha \in\N^n}{\sum_{j=1}^n j\alpha_j =
    n}$, 
  for all $x$, $h\in X$ the $n$th derivative
  $[D^n(G\circ F)](x)(h^n)\in Z$ of $G\circ F$ at $x$ evaluated at
  $h^n\in X^n$ equals
  \begin{equation*}
     \sum_{\bsalpha\in T_n}
    \frac{n!}{\bsalpha!}
    [D^{|\bsalpha|}G](F(x)) \Bigg(\smash[b]{\underbrace{\frac{[DF(x)](h)}{1!},\ldots, \frac{[DF(x)](h)}{1!}}_{\text{$\alpha_1$ times}}},\ldots,\underbrace{\frac{[D^{n}F(x)](h^n)}{n!},\ldots, \frac{[D^{n}F(x)](h^n)}{n!}}_{\text{$\alpha_n$ times}}\Bigg).
  \end{equation*}

\end{theorem}
\begin{proof}
  
  Without loss of generality assume
  $F(0) = 0\in Y$.
    Using Taylor's theorem (in Banach spaces) for $G$
  \begin{equation*}
    G(F(x)) =  \sum_{r=1}^k\frac{[D^rG](0)}{r!}({\underbrace{F(x),\ldots, F(x)}_\text{$r$ times}}) + o(\|F(x)\|_Y^k)\qquad\text{as }\|F(x)\|_Y\to 0.
  \end{equation*}
Taylor expanding $F(x)$ around $0\in X$ then implies
\begin{equation*}
  G(F(x)) =  \sum_{r=1}^k\frac{[D^rG](0)}{r!}(S_r(x))+ o(\|F(x)\|_Y^k)
\end{equation*}
where, using the notation $x^{\alpha_j}$ to denote $(x,\dots,x)\in X^{\alpha_j}$,
\begin{equation*}
  S_r(x)=
  \left(\sum_{\alpha_1=1}^k\frac{[D^{\alpha_1}F](0)}{\alpha_1!}(x^{\alpha_1}) + o(\|x\|_X^k),\ldots, \sum_{\alpha_r= 1}^k\frac{[D^{\alpha_r}F](0)}{n!}(x^{\alpha_r})+
  o(\|x\|_X^k)\right)\in Y^r.
\end{equation*}
Note that $F\in C^1$ and $F(0)=0\in Y$ implies
$o(\|F(x)\|^k) = o(\|x\|^k)$ as $x\to 0$.  Using multilinearity of the
differential operators we thus find
\begin{equation*}
  G(F(x)) = \sum_{r=1}^k\sum_{\setl{\bsalpha\in \N^r}{|\bsalpha|\le k}}
  \frac{[D^rG](0)}{r!}
  \left(\frac{[D^{\alpha_1}F](0)(x^{\alpha_1})}{\alpha_1!},\dots,\frac{[D^{\alpha_r}F](0)(x^{\alpha_r})}{\alpha_r!}\right) + o(\|x\|^k)
\end{equation*}
as $x\to 0$.

On the other hand, we can Taylor expand $G\circ F$. That is,
\begin{equation*}
  G(F(x)) = \sum_{r = 1}^k\frac{[D^r(G\circ F)(0)]}{r!}(x^r) + o(\|x\|^k)\qquad\text{as }x\to 0.
\end{equation*}

Comparing the powers of $x$, we get for $n\le k$
\begin{equation*}\label{eq:FdB2}
  D^n(G\circ F)(0)(x^n) = n!\sum_{r=1}^k
  \sum_{\setl{\bsalpha\in\N^r}{|\bsalpha|=n}}
  \frac{[D^rG](0)}{r!}
  \left(\frac{[D^{\alpha_1}F](0)(x^{\alpha_1})}{\alpha_1!},\dots,\frac{[D^{\alpha_r}F](0)(x^{\alpha_r})}{\alpha_r!}\right).
\end{equation*}

To show that the expression for $D^n(G\circ F)(0)(x^n)$ is equivalent to the one given by Fa\'{a} di Bruno's formula in Theorem \ref{thm:FdB}, we make the following observation: the summation $\sum_{\setl{\bsalpha\in\N^r}{|\bsalpha|=n}}$ is over the partition of a set of $n$ elements into $r$ subsets each of $\alpha_i$ elements such that $r \geq i \geq 1, \alpha_i\geq 1$. If we let set $T_n = \{\bsalpha' = (\alpha_1',\ldots,\alpha'_n) : \sum_{j=1}^n j\alpha_j' = n\}$ as in Theorem \ref{thm:FdB}, each $\alpha_j'$ can be interpreted as the number of subsets with $j$ elements and the total number of subsets is given by $|\bsalpha'|$. Another observation we make is that the summation in Theorem \ref{thm:FdB} takes ordered tuples while the summation in the above expression takes unordered tuple $(\alpha_1,...,\alpha_r)$. If $|\bsalpha'| = r$, the number of ways to arrange a tuple of $r$ elements such that $\alpha_1'$ of the elements are same (of value $1$), $\alpha_2'$ of the elements are the same (of value $2$),..., and $\alpha_n'$ of the elements are same (of value $n$), is given by $\frac{r!}{\alpha_1'!\dots\alpha_n'!} = \frac{r!}{\bsalpha'!}$. 

Therefore, by regrouping the summation, we obtain: 
\begin{align*}
&D^n(G\circ F)(0)(x^n) = n!\sum_{r=1}^k
  \sum_{\setl{\bsalpha\in\N^r}{|\bsalpha|=n}}
  \frac{[D^rG](0)}{r!}
  \left(\frac{[D^{\alpha_1}F](0)(x^{\alpha_1})}{\alpha_1!},\dots,\frac{[D^{\alpha_r}F](0)(x^{\alpha_r})}{\alpha_r!}\right)\\
&= \sum_{\bsalpha'\in T_n}\frac{n!}{r!}\frac{r!}{\bsalpha'!}[D^{|\bsalpha'|}G](0)\Bigg(\smash[b]{\underbrace{\frac{[DF](0)(x)}{1!},\ldots, \frac{[DF](0)(x)}{1!}}_{\text{$\alpha'_1$ times}}},\ldots,\underbrace{\frac{[D^{n}F](0)(x^n)}{n!},\ldots, \frac{[D^{n}F](0)(x^n)}{n!}}_{\text{$\alpha'_n$ times}}\Bigg)\\
&=\sum_{\bsalpha'\in T_n}\frac{n!}{\bsalpha'!}[D^{|\bsalpha'|}G](0)\Bigg(\smash[b]{\underbrace{\frac{[DF](0)(x)}{1!},\ldots, \frac{[DF](0)(x)}{1!}}_{\text{$\alpha'_1$ times}}},\ldots,\underbrace{\frac{[D^{n}F](0)(x^n)}{n!},\ldots, \frac{[D^{n}F](0)(x^n)}{n!}}_{\text{$\alpha'_n$ times}}\Bigg)\\ 
\end{align*}
\end{proof}

\begin{theorem}[Inverse function theorem]
  Let $k\ge 1$, let $X$, $Y$ be two Banach spaces, and let $F\in C^k(X,Y)$.
  At every $x\in X$ for which $DF(x)\in L^1(X,Y)$ is an isomorphism,
  there exists an open neighbourhood $O\subseteq Y$ of $F(x)$ and a function
  $G\in C^k(O,X)$ such that $F(G(y))=y$ for all $y\in O$.

  Moreover, for every $n\le k$ there exists a continuous function
  $C_n:\R_+^{n+1}\to \R_+$ (independent of $F$, $G$, $O$) such
  that for $y=F(x)$ with $x$ as above
  \begin{equation}\label{eq:Cn}
    \|D^n G(y)\|_{\cL^n_{\sym}(Y;X)} \le C_n(\|[DF(x)]^{-1}\|_{\cL^1_{\sym}(Y;X)}, \|DF(x)\|_{\cL^1_{\sym}(X;Y)}, \ldots, \|D^nF(x)\|_{\cL^n_{\sym}(X;Y)}).
  \end{equation}
\end{theorem}
\begin{proof}
  The stated local invertibility of $F$ holds by the inverse function
  theorem in Banach spaces, see for instance
  \cite[Cor.~15.1]{deimling}.

  Next, if $F(G(y))=y$ in a neighbourhood of $y=F(x)$, then by
  Theorem ~\ref{thm:FdB}, for $2\le n\le k$
  \begin{equation}\label{eq:solveDngy}
    0 = D^n(F\circ G)(y) = \sum_{\bsalpha\in T_n}\frac{n!}{\bsalpha!}[D^{|\bsalpha|}F](G(y))\prod_{m=1}^n\left(\frac{[D^m G](y)}{m!}\right)^{\alpha_m},
  \end{equation}
  where $T_n = \setl{\bsalpha \in\N^n}{\sum_{j=1}^n j\alpha_j = n}$.
  The only multiindex with $\alpha_n\neq 0$ is
  $\bsalpha = (0,\ldots,0,1)$. Solving \eqref{eq:solveDngy} for
  $D^nG(y)$, we obtain with
  $\bar{T}_n  \coloneqq  \setl{\bsalpha \in\N^{n-1}}{\sum_{j=1}^{n-1}j\alpha_j=n}$ and
  $y=F(x)$,
  \begin{equation}\label{eq:Dng}
    D^nG(y) = -([DF](G(y)))^{-1} \left ( \sum_{\bsalpha\in \bar{T}_n}\frac{n!}{\bsalpha!}[D^{|\bsalpha|}F](G(y)) \left( \prod_{m=1}^{n-1} \left (\frac{[D^m G](y)}{m!} \right )^{\alpha_m}\right)\right).
  \end{equation}
  Since $G(y)=x$, for $n=1$, we have
  \begin{equation*}
    \|DG(y)\|_{\cL^1(Y;X)} = \|[DF(x)]^{-1}\|_{\cL^1(Y;X)}
  \end{equation*}
  and thus \eqref{eq:Cn} holds with
  \begin{equation*}
    C_1(\|[DF(x)]^{-1}\|_{\cL^1(Y;X)},\|DF(x)\|_{\cL^1(X;Y)}) \coloneqq 
    \|[DF(x)]^{-1}\|_{\cL^1(Y;X)}.
  \end{equation*}

  Since the right-hand side of \eqref{eq:Dng} only depends on $G$
  through $D^mG$ with $m\le n-1$ and on $F$ through $D^{|\bsalpha|}F$
  with $|\bsalpha|\le n$, an induction argument implies the existence
  of $C_n:\R_+^{n+1}\to\R_+$ as in \eqref{eq:Cn} for every $n\ge 2$.
\end{proof}

A technical lemma that upper bounds terms in Fa\'{a} di Bruno's formula:

\begin{lemma}\label{lemma:stirling}
For every $n\in\N$ holds $\sum_{\bsalpha\in T_n}
  \frac{n!}{\bsalpha!}\prod_{j=1}^n\frac{1}{(j!)^{\alpha_j}}\le n^n. $   
\end{lemma}
\begin{proof}
The sum
  $\sum_{\setl{\bsalpha\in T_n}{|\bsalpha|=k}}
  \frac{n!}{\bsalpha!}\prod_{j=1}^n\frac{1}{(j!)^{\alpha_j}}$ is equal
  to the so-called Stirling number of the second kind, $S_n^k$ (\cite{FaadiBruno}). Therefore, all we need to do is to upper bound $\sum_{k=1}^nS_n^k$. 

Note $S_n^k$ denote the number of ways to distribute $n$ distinct items into $k$ non-distinct boxes such that each box contains at least one item. Then, $k!S_n^k$ is the number of ways to distribute $n$ distinct items into $k$ distinct boxes such that none of the boxes are empty. The total number of ways to distribute $n$ distinct items into $n$ distinct boxes is given by $n^n$ (without the restriction that the boxes are nonempty). Therefore, we have 
$$\sum_{k=1}^n{n\choose k}k! S_n^k = n^n,$$
and it follows that $\sum_{k=1}^nS_n^k \leq n^n$. 
\end{proof}

We also collect several auxiliary results about norm of matrices and their inverses. 
\begin{lemma}\label{lemma:blockTriangle}
Suppose $M$ is a block triangular matrix: \[
   M =
   \begin{pmatrix}
   A & B \\
   0 & D 
   \end{pmatrix}
\]
where $A, D$ are invertible. Then,
$$\|M^{-1}\|_2 
\leq \|D^{-1}\|_2 + \|A^{-1}BD^{-1}\|_2 + \|A^{-1}\|_2.$$
\end{lemma}
\begin{proof}
  Since $A, D$ are invertible, the inverse of $M$ is
\begin{equation*}
  M^{-1} = 
                       \begin{pmatrix}
                        A^{-1} & -A^{-1}BD^{-1} \\
                        0 & D^{-1}
                        \end{pmatrix}.
\end{equation*}
The claim follows from the triangle inequality.
\end{proof}

\begin{lemma}\label{lemma:MatrixInverseNorm}
  Let $A\in\R^{d\times d}$ be regular. Then
  $\|A^{-1}\|_2\le \frac{\|A\|_2^{d-1}}{|\det(A)|}$.
\end{lemma}
\begin{proof}
  Denote the singular values of $A$ by
  $\sigma_1\ge \sigma_2\ge\cdots\ge\sigma_d>0$. Then
  \begin{equation*}
    \frac{\|A\|_2^{d-1}}{|\det(A)|} =
    \frac{\sigma_1^{d-1}}{
      \prod_{i=1}^d\sigma_i} =
    \left(\prod_{i=1}^{d-1}\frac{\sigma_1}{\sigma_i}\right)\frac{1}{\sigma_1}
    \ge \frac{1}{\sigma_1} = \|A^{-1}\|_2. 
  \end{equation*}
\end{proof}

\begin{lemma}\label{lemma:NormTtinverse}
  Let $T\in C^1(\Omega_0,\Omega_1)$ be a monotonic triangular map. Moreover, assume $\nabla T(x)$ has positive eigenvalues
  $\lambda_j(x)$ for $j = 1,\cdots, d$ and all $x\in\Omega_0$.
  Then for all $t\in [0,1]$ and all $x\in\Omega_0$
  \begin{equation*}
    \|(\nabla_x T_t(x))^{-1}\|_{2} = \|((1-t)I + t\nabla_x
    T(x))^{-1}\|_{\jz{2}} \le
    \frac{
      \max\{1,\|T\|_{C^1(\Omega_0)}\}^{d-1}
    }{\min\{1, \inf_{x\in\Omega_0}\det dT(x)}\}.
  \end{equation*}
\end{lemma}
\begin{proof}
  By Lemma \ref{lemma:MatrixInverseNorm}, for every
  $(x,t)\in\Omega_0\times [0,1]$
  \begin{equation*}
    \|(\nabla_x T_t(x))^{-1}\|_2\leq \frac{\|\nabla_xT_t(x)\|_{2}^{d-1}}{|\det(T_t(x))|}.
  \end{equation*}
  We have
  $\|\nabla T_t(x)\|_{2} \leq (1-t) + t\|\nabla
  T(x)\|_{2} \le \max\{1,\|T\|_{C^1(\Omega_0)}\}$. For
  $s\in [0,1]$, $t\in (0,1)$ set $g_t(s) := \log(1-t+\e^{s})$. This
  function is convex in $s\in[0,1]$. Thus
  \begin{equation*}
    g_t\left(\frac{1}{d} \sum_{i=1}^d\log (t\lambda_i(x))\right)\le
    \frac{1}{d} \sum_{i=1}^d g_t(\log (t\lambda_i(x))),
  \end{equation*}
  and therefore
  \begin{align*}
    \sum_{i=1}^d \log(1-t+t\lambda_i(x)) &\ge d \log\Bigg(1-t +
                                           t\Bigg(\prod_{i=1}^d \lambda_i(x)\Bigg)^{1/d}\Bigg)\nonumber\\
                                         &= d\log(1-t+t\det(\nabla
                                           T(x))^{1/d}).
  \end{align*}
  Taking the exponential on both sides, we conclude that
  \begin{equation*}
    \det((1-t)I + t\nabla T(x)) \geq (1-t+t\det(\nabla T(x))^{1/d})^d \geq \min\{1, \det dT\jz{(x)}\}.
  \end{equation*}
  Finally, Lemma \ref{lemma:MatrixInverseNorm} gives the result.
\end{proof}

Finally, a technical lemma about the upper bounds on the
$C^k$ norm of the reciprocal of a $C^k$ function bounded away from
zero.  
\begin{lemma}\label{lemma:NormReciprocal}
  Let $k \in \mathbb{N}, k \geq 1$ and $f \in C^k(D)$ for domain $D\subset\mathbb{R}^d$ such that
  $\inf_{x\in D} f(x) > C_2$ for some constant $C_2 > 0$. Assume
  further $\|f\|_{C^k(D)} \leq C_1$ for another constant $C_1$ > 0. Then, it
  holds that
$$\|\frac{1}{f}\|_{C^k(D)} \leq C\frac{C_1^k}{C_2^{k+1}}$$
for some constant $C$ that depends on $k$ but independent of $f$,
$C_1$, $C_2$.
\end{lemma}
\begin{proof}
  We proceed as in \cite[Lemma C.4 (iii)]{ZM1}.  Let
  $n\in\mathbb{N}$ be an integer such that $0 \leq n \leq k$. We shall
  show by induction that $D^n(\frac{1}{f}) = \frac{p_n}{f^{n+1}}$ for
  some function $p_n$ such that
  $\|p_n\|_{C^{k-n}(D)} \leq C(\|f\|_{C^k(D)})^n$.

  When $n = 1$, it holds that $D(\frac{1}{f}) = \frac{-Df}{f^2}$ and
  thus $\|D(\frac{1}{f})\|_{C^{k-1}(D)} \leq \frac{C_1}{C_2^2}$.

  Assume the induction hypothesis holds, for $n+1$, we have
  $$D^{n+1}(\frac{1}{f}) = D(\frac{p_n}{f^{n+1}}) = \frac{f^{n+1}Dp_n
    + p_n(n+1)f^nDf}{f^{2n+2}} = \frac{fDp_n + (n+1)p_nDf}{f^{n+2}} :=
  \frac{p_{n+1}}{f^{n+2}}.$$

  Then it holds that
  \begin{align*}
    \|p_{n+1}\|_{C^{k-n-1}(D)} &= \|fDp_n + (n+1)p_nDf\|_{C^{k-n-1}(D)}\\
                                    &\leq C(n+2)\|f\|_{C^k(D)}\|p_n\|_{C^{k-n}}\\
                                    &\leq C(n+2)(\|f\|_{C^k(D)})^{n+1}.
  \end{align*}

  Therefore, we have
  $\|\frac{1}{f}\|_{C^k(D)} \leq C\frac{C_1^k}{C_2^{k+1}}$ for
  some constant $C$ that depends on $k$ but independent of $f$, $C_1$,
  $C_2$.
\end{proof}

\bibliography{ref.bib}

\begin{thebibliography}{56}
\providecommand{\natexlab}[1]{#1}
\providecommand{\url}[1]{\texttt{#1}}
\expandafter\ifx\csname urlstyle\endcsname\relax
  \providecommand{\doi}[1]{doi: #1}\else
  \providecommand{\doi}{doi: \begingroup \urlstyle{rm}\Url}\fi

\bibitem[Albergo and Vanden{-}Eijnden(2023)]{StochasticInterpoalnt2}
Michael~S. Albergo and Eric Vanden{-}Eijnden.
\newblock Building normalizing flows with stochastic interpolants.
\newblock In \emph{The Eleventh International Conference on Learning
  Representations, {ICLR} 2023, Kigali, Rwanda, May 1-5, 2023}. OpenReview.net,
  2023.
\newblock URL \url{https://openreview.net/forum?id=li7qeBbCR1t}.

\bibitem[Albergo et~al.(2023)Albergo, Boffi, and
  Vanden{-}Eijnden]{StochasticInterpolant1}
Michael~S. Albergo, Nicholas~M. Boffi, and Eric Vanden{-}Eijnden.
\newblock Stochastic interpolants: {A} unifying framework for flows and
  diffusions.
\newblock \emph{CoRR}, abs/2303.08797, 2023.
\newblock \doi{10.48550/ARXIV.2303.08797}.
\newblock URL \url{https://doi.org/10.48550/arXiv.2303.08797}.

\bibitem[{\'A}lvarez-L{\'o}pez et~al.(2024){\'A}lvarez-L{\'o}pez, Geshkovski,
  and Ruiz-Balet]{alvarez2024constructive}
Antonio {\'A}lvarez-L{\'o}pez, Borjan Geshkovski, and Dom{\`e}nec Ruiz-Balet.
\newblock Constructive approximate transport maps with normalizing flows.
\newblock \emph{arXiv preprint arXiv:2412.19366}, 2024.

\bibitem[Arnold and Silverman(1978)]{ArnoldODE}
V.I. Arnold and R.A. Silverman.
\newblock \emph{Ordinary Differential Equations}.
\newblock London, 1978.
\newblock ISBN 9780262510189.
\newblock URL \url{https://books.google.com/books?id=5NumQgAACAAJ}.

\bibitem[Baptista et~al.(2023)Baptista, Marzouk, and Zahm]{baptista23}
Ricardo Baptista, Youssef Marzouk, and Olivier Zahm.
\newblock On the representation and learning of monotone triangular transport
  maps.
\newblock \emph{Foundations of Computational Mathematics}, Nov 2023.
\newblock ISSN 1615-3383.
\newblock \doi{10.1007/s10208-023-09630-x}.
\newblock URL \url{https://doi.org/10.1007/s10208-023-09630-x}.

\bibitem[Baptista et~al.(2024)Baptista, Hosseini, Kovachki, Marzouk, and
  Sagiv]{baptista2023approximation}
Ricardo Baptista, Bamdad Hosseini, Nikola~B Kovachki, Youssef~M Marzouk, and
  Amir Sagiv.
\newblock An approximation theory framework for measure-transport sampling
  algorithms.
\newblock \emph{Mathematics of Computation}, 2024.

\bibitem[Behrmann et~al.(2019)Behrmann, Grathwohl, Chen, Duvenaud, and
  Jacobsen]{InvNN}
Jens Behrmann, Will Grathwohl, Ricky T.~Q. Chen, David Duvenaud, and
  J{\"{o}}rn{-}Henrik Jacobsen.
\newblock Invertible residual networks.
\newblock In Kamalika Chaudhuri and Ruslan Salakhutdinov, editors,
  \emph{Proceedings of the 36th International Conference on Machine Learning,
  {ICML} 2019, 9-15 June 2019, Long Beach, California, {USA}}, volume~97 of
  \emph{Proceedings of Machine Learning Research}, pages 573--582. {PMLR},
  2019.
\newblock URL \url{http://proceedings.mlr.press/v97/behrmann19a.html}.

\bibitem[Benamou and Brenier(2000)]{OT-CFD}
Jean{-}David Benamou and Yann Brenier.
\newblock A computational fluid mechanics solution to the monge-kantorovich
  mass transfer problem.
\newblock \emph{Numerische Mathematik}, 84\penalty0 (3):\penalty0 375--393,
  2000.
\newblock \doi{10.1007/s002110050002}.
\newblock URL \url{https://doi.org/10.1007/s002110050002}.

\bibitem[Benton et~al.(2024)Benton, Deligiannidis, and
  Doucet]{BentonErrorBoundsFlowMatching}
Joe Benton, George Deligiannidis, and Arnaud Doucet.
\newblock Error bounds for flow matching methods.
\newblock \emph{Trans. Mach. Learn. Res.}, 2024, 2024.
\newblock URL \url{https://openreview.net/forum?id=uqQPyWFDhY}.

\bibitem[Bogachev et~al.(2007)Bogachev, Kolesnikov, and
  Medvedev]{TrangularTransportofMeasure}
Vladimir Bogachev, Alexander Kolesnikov, and Kirill Medvedev.
\newblock Triangular transformations of measures.
\newblock \emph{Sbornik: Mathematics}, 196:\penalty0 309, 10 2007.
\newblock \doi{10.1070/SM2005v196n03ABEH000882}.

\bibitem[Brenier(1991)]{BrenierMap}
Yann Brenier.
\newblock Polar factorization and monotone rearrangement of vector-valued
  functions.
\newblock \emph{Communications on Pure and Applied Mathematics}, 44:\penalty0
  375--417, 1991.

\bibitem[Chae(1985)]{chae}
Soo~Bong Chae.
\newblock \emph{Holomorphy and calculus in normed spaces}, volume~92 of
  \emph{Monographs and Textbooks in Pure and Applied Mathematics}.
\newblock Marcel Dekker, Inc., New York, 1985.
\newblock ISBN 0-8247-7231-8.
\newblock With an appendix by Angus E. Taylor.

\bibitem[Chen et~al.(2018)Chen, Bettencourt, and Duvenaud]{NeuralODE}
Tian~Qi Chen, Jesse Bettencourt, and David Duvenaud.
\newblock Neural ordinary differential equations.
\newblock \emph{CoRR}, abs/1806.07366, 2018.
\newblock URL \url{http://arxiv.org/abs/1806.07366}.

\bibitem[Constantine and Savits(1996)]{FaadiBruno}
G.~Constantine and T.~H. Savits.
\newblock A multivariate faa di bruno formula with applications.
\newblock \emph{Transactions of the American Mathematical Society},
  348:\penalty0 503--520, 1996.

\bibitem[Deimling(1985)]{deimling}
Klaus Deimling.
\newblock \emph{Nonlinear functional analysis}.
\newblock Springer-Verlag, Berlin, 1985.
\newblock ISBN 3-540-13928-1.
\newblock \doi{10.1007/978-3-662-00547-7}.
\newblock URL \url{https://doi.org/10.1007/978-3-662-00547-7}.

\bibitem[Dinh et~al.(2015)Dinh, Krueger, and Bengio]{CouplingFlows}
Laurent Dinh, David Krueger, and Yoshua Bengio.
\newblock {NICE:} non-linear independent components estimation.
\newblock In Yoshua Bengio and Yann LeCun, editors, \emph{3rd International
  Conference on Learning Representations, {ICLR} 2015, San Diego, CA, USA, May
  7-9, 2015, Workshop Track Proceedings}, 2015.
\newblock URL \url{http://arxiv.org/abs/1410.8516}.

\bibitem[Finlay et~al.(2020)Finlay, Jacobsen, Nurbekyan, and
  Oberman]{HowToTrain}
Chris Finlay, J{\"{o}}rn{-}Henrik Jacobsen, Levon Nurbekyan, and Adam~M.
  Oberman.
\newblock How to train your neural {ODE:} the world of jacobian and kinetic
  regularization.
\newblock In \emph{Proceedings of the 37th International Conference on Machine
  Learning, {ICML} 2020, 13-18 July 2020, Virtual Event}, volume 119 of
  \emph{Proceedings of Machine Learning Research}, pages 3154--3164. {PMLR},
  2020.
\newblock URL \url{http://proceedings.mlr.press/v119/finlay20a.html}.

\bibitem[Gao et~al.(2024)Gao, Huang, Jiao, and
  Zheng]{GaoConvergenceofContinusNormalizingFlows}
Yuan Gao, Jian Huang, Yuling Jiao, and Shurong Zheng.
\newblock Convergence of continuous normalizing flows for learning probability
  distributions.
\newblock \emph{CoRR}, abs/2404.00551, 2024.
\newblock \doi{10.48550/ARXIV.2404.00551}.
\newblock URL \url{https://doi.org/10.48550/arXiv.2404.00551}.

\bibitem[Gholaminejad et~al.(2019)Gholaminejad, Keutzer, and
  Biros]{DiscretizeOptimize}
Amir Gholaminejad, Kurt Keutzer, and George Biros.
\newblock {ANODE:} unconditionally accurate memory-efficient gradients for
  neural odes.
\newblock In Sarit Kraus, editor, \emph{Proceedings of the Twenty-Eighth
  International Joint Conference on Artificial Intelligence, {IJCAI} 2019,
  Macao, China, August 10-16, 2019}, pages 730--736. ijcai.org, 2019.
\newblock \doi{10.24963/ijcai.2019/103}.
\newblock URL \url{https://doi.org/10.24963/ijcai.2019/103}.

\bibitem[Grathwohl et~al.(2019)Grathwohl, Chen, Bettencourt, Sutskever, and
  Duvenaud]{ffjord}
Will Grathwohl, Ricky T.~Q. Chen, Jesse Bettencourt, Ilya Sutskever, and David
  Duvenaud.
\newblock {FFJORD:} free-form continuous dynamics for scalable reversible
  generative models.
\newblock In \emph{7th International Conference on Learning Representations,
  {ICLR} 2019, New Orleans, LA, USA, May 6-9, 2019}. OpenReview.net, 2019.
\newblock URL \url{https://openreview.net/forum?id=rJxgknCcK7}.

\bibitem[Hofmann et~al.(2007)Hofmann, Mitrea, and Taylor]{LipTransformation}
S.~Hofmann, M.~Mitrea, and Michael~E. Taylor.
\newblock Geometric and transformational properties of lipschitz domains,
  semmes-kenig-toro domains, and other classes of finite perimeter domains.
\newblock \emph{The Journal of Geometric Analysis}, 17:\penalty0 593--647,
  2007.

\bibitem[Huang et~al.(2018)Huang, Krueger, Lacoste, and
  Courville]{NeuralAutoFlow}
Chin{-}Wei Huang, David Krueger, Alexandre Lacoste, and Aaron~C. Courville.
\newblock Neural autoregressive flows.
\newblock In Jennifer~G. Dy and Andreas Krause, editors, \emph{Proceedings of
  the 35th International Conference on Machine Learning, {ICML} 2018,
  Stockholmsm{\"{a}}ssan, Stockholm, Sweden, July 10-15, 2018}, volume~80 of
  \emph{Proceedings of Machine Learning Research}, pages 2083--2092. {PMLR},
  2018.
\newblock URL \url{http://proceedings.mlr.press/v80/huang18d.html}.

\bibitem[Huang et~al.(2021)Huang, Chen, Tsirigotis, and
  Courville]{huang2021convex}
Chin-Wei Huang, Ricky T.~Q. Chen, Christos Tsirigotis, and Aaron Courville.
\newblock Convex potential flows: Universal probability distributions with
  optimal transport and convex optimization.
\newblock In \emph{International Conference on Learning Representations}, 2021.
\newblock URL \url{https://openreview.net/forum?id=te7PVH1sPxJ}.

\bibitem[Huang et~al.(2024)Huang, Huang, and
  Lin]{HuangProbabilisticODEConvergence}
Daniel~Zhengyu Huang, Jiaoyang Huang, and Zhengjiang Lin.
\newblock Convergence analysis of probability flow {ODE} for score-based
  generative models.
\newblock \emph{CoRR}, abs/2404.09730, 2024.
\newblock \doi{10.48550/ARXIV.2404.09730}.
\newblock URL \url{https://doi.org/10.48550/arXiv.2404.09730}.

\bibitem[Kingma et~al.(2016)Kingma, Salimans, J{\'{o}}zefowicz, Chen,
  Sutskever, and Welling]{autoregressiveflow}
Diederik~P. Kingma, Tim Salimans, Rafal J{\'{o}}zefowicz, Xi~Chen, Ilya
  Sutskever, and Max Welling.
\newblock Improving variational autoencoders with inverse autoregressive flow.
\newblock In Daniel~D. Lee, Masashi Sugiyama, Ulrike von Luxburg, Isabelle
  Guyon, and Roman Garnett, editors, \emph{Advances in Neural Information
  Processing Systems 29: Annual Conference on Neural Information Processing
  Systems 2016, December 5-10, 2016, Barcelona, Spain}, pages 4736--4744, 2016.
\newblock URL
  \url{https://proceedings.neurips.cc/paper/2016/hash/ddeebdeefdb7e7e7a697e1c3e3d8ef54-Abstract.html}.

\bibitem[Kingma and Dhariwal(2018)]{glow}
Durk~P Kingma and Prafulla Dhariwal.
\newblock Glow: Generative flow with invertible 1x1 convolutions.
\newblock In S.~Bengio, H.~Wallach, H.~Larochelle, K.~Grauman, N.~Cesa-Bianchi,
  and R.~Garnett, editors, \emph{Advances in Neural Information Processing
  Systems}, volume~31. Curran Associates, Inc., 2018.
\newblock URL
  \url{https://proceedings.neurips.cc/paper_files/paper/2018/file/d139db6a236200b21cc7f752979132d0-Paper.pdf}.

\bibitem[Kobyzev et~al.(2020)Kobyzev, Prince, and
  Brubaker]{NormalizingFlowIntro}
Ivan Kobyzev, Simon Prince, and Marcus Brubaker.
\newblock Normalizing flows: An introduction and review of current methods.
\newblock \emph{IEEE Transactions on Pattern Analysis and Machine
  Intelligence}, page 1–1, 2020.
\newblock ISSN 1939-3539.
\newblock \doi{10.1109/tpami.2020.2992934}.
\newblock URL \url{http://dx.doi.org/10.1109/TPAMI.2020.2992934}.

\bibitem[Kong and Chaudhuri(2020)]{DisAppro2}
Zhifeng Kong and Kamalika Chaudhuri.
\newblock The expressive power of a class of normalizing flow models.
\newblock In Silvia Chiappa and Roberto Calandra, editors, \emph{The 23rd
  International Conference on Artificial Intelligence and Statistics, {AISTATS}
  2020, 26-28 August 2020, Online [Palermo, Sicily, Italy]}, volume 108 of
  \emph{Proceedings of Machine Learning Research}, pages 3599--3609. {PMLR},
  2020.
\newblock URL \url{http://proceedings.mlr.press/v108/kong20a.html}.

\bibitem[Li et~al.(2024)Li, Wei, Chi, and Chen]{LiProbabilisticODEConvergence}
Gen Li, Yuting Wei, Yuejie Chi, and Yuxin Chen.
\newblock A sharp convergence theory for the probability flow odes of diffusion
  models.
\newblock \emph{CoRR}, abs/2408.02320, 2024.
\newblock \doi{10.48550/ARXIV.2408.02320}.
\newblock URL \url{https://doi.org/10.48550/arXiv.2408.02320}.

\bibitem[Li et~al.(2019)Li, Lin, and Shen]{DynamicalSystem}
Qianxiao Li, Ting Lin, and Zuowei Shen.
\newblock Deep learning via dynamical systems: An approximation perspective.
\newblock \emph{CoRR}, abs/1912.10382, 2019.
\newblock URL \url{http://arxiv.org/abs/1912.10382}.

\bibitem[Lipman et~al.(2023)Lipman, Chen, Ben{-}Hamu, Nickel, and
  Le]{FlowMatching}
Yaron Lipman, Ricky T.~Q. Chen, Heli Ben{-}Hamu, Maximilian Nickel, and Matthew
  Le.
\newblock Flow matching for generative modeling.
\newblock In \emph{The Eleventh International Conference on Learning
  Representations, {ICLR} 2023, Kigali, Rwanda, May 1-5, 2023}. OpenReview.net,
  2023.
\newblock URL \url{https://openreview.net/forum?id=PqvMRDCJT9t}.

\bibitem[Liu et~al.(2023)Liu, Gong, and Liu]{RectifiedFlow}
Xingchao Liu, Chengyue Gong, and Qiang Liu.
\newblock Flow straight and fast: Learning to generate and transfer data with
  rectified flow.
\newblock In \emph{The Eleventh International Conference on Learning
  Representations, {ICLR} 2023, Kigali, Rwanda, May 1-5, 2023}. OpenReview.net,
  2023.
\newblock URL \url{https://openreview.net/forum?id=XVjTT1nw5z}.

\bibitem[Lu et~al.(2018)Lu, Zhong, Li, and Dong]{NumericalSchemeODE}
Yiping Lu, Aoxiao Zhong, Quanzheng Li, and Bin Dong.
\newblock Beyond finite layer neural networks: Bridging deep architectures and
  numerical differential equations.
\newblock In \emph{6th International Conference on Learning Representations,
  {ICLR} 2018, Vancouver, BC, Canada, April 30 - May 3, 2018, Workshop Track
  Proceedings}. OpenReview.net, 2018.
\newblock URL \url{https://openreview.net/forum?id=B1LatDVUM}.

\bibitem[Marzouk et~al.(2016)Marzouk, Moselhy, Parno, and
  Spantini]{measure-transport}
Youssef Marzouk, Tarek Moselhy, Matthew Parno, and Alessio Spantini.
\newblock Sampling via measure transport: An introduction.
\newblock \emph{Handbook of Uncertainty Quantification}, page 1–41, 2016.
\newblock \doi{10.1007/978-3-319-11259-6_23-1}.
\newblock URL \url{http://dx.doi.org/10.1007/978-3-319-11259-6_23-1}.

\bibitem[Marzouk et~al.(2024)Marzouk, Ren, Wang, and Zech]{StatisticalNODE}
Youssef Marzouk, Zhi Ren, Sven Wang, and Jakob Zech.
\newblock Distribution learning via neural differential equations: a
  nonparametric statistical perspective.
\newblock \emph{Journal of Machine Learning Research}, 25:\penalty0 1--61,
  2024.
\newblock arXiv:2309.01043.

\bibitem[{Moselhy} and Marzouk(2012)]{moselhy2012}
Tarek~A. {Moselhy} and Youssef~M. Marzouk.
\newblock Bayesian inference with optimal maps.
\newblock \emph{Journal of Computational Physics}, 231\penalty0 (23):\penalty0
  7815--7850, 2012.
\newblock ISSN 0021-9991.
\newblock \doi{https://doi.org/10.1016/j.jcp.2012.07.022}.
\newblock URL
  \url{https://www.sciencedirect.com/science/article/pii/S0021999112003956}.

\bibitem[Onken et~al.(2021)Onken, Fung, Li, and Ruthotto]{OTFlow}
Derek Onken, Samy~Wu Fung, Xingjian Li, and Lars Ruthotto.
\newblock Ot-flow: Fast and accurate continuous normalizing flows via optimal
  transport.
\newblock In \emph{Thirty-Fifth {AAAI} Conference on Artificial Intelligence,
  {AAAI} 2021, Thirty-Third Conference on Innovative Applications of Artificial
  Intelligence, {IAAI} 2021, The Eleventh Symposium on Educational Advances in
  Artificial Intelligence, {EAAI} 2021, Virtual Event, February 2-9, 2021},
  pages 9223--9232. {AAAI} Press, 2021.
\newblock URL \url{https://ojs.aaai.org/index.php/AAAI/article/view/17113}.

\bibitem[Panaretos and Zemel(2020)]{OptimalRegularity}
Victor~M. Panaretos and Yoav Zemel.
\newblock \emph{An Invitation to Statistics in Wasserstein Space}.
\newblock Springer, 2020.
\newblock ISBN 978-3-030-38437-1.
\newblock URL
  \url{https://link.springer.com/content/pdf/10.1007%2F978-3-030-38438-8.pdf}.

\bibitem[Rezende and Mohamed(2015)]{PlanarFlow}
Danilo~Jimenez Rezende and Shakir Mohamed.
\newblock Variational inference with normalizing flows.
\newblock In Francis~R. Bach and David~M. Blei, editors, \emph{Proceedings of
  the 32nd International Conference on Machine Learning, {ICML} 2015, Lille,
  France, 6-11 July 2015}, volume~37 of \emph{{JMLR} Workshop and Conference
  Proceedings}, pages 1530--1538. JMLR.org, 2015.
\newblock URL \url{http://proceedings.mlr.press/v37/rezende15.html}.

\bibitem[Ruiz{-}Balet and Zuazua(2023)]{NeuralODE-Control}
Dom{\`{e}}nec Ruiz{-}Balet and Enrique Zuazua.
\newblock Neural {ODE} control for classification, approximation, and
  transport.
\newblock \emph{{SIAM} Rev.}, 65\penalty0 (3):\penalty0 735--773, 2023.
\newblock \doi{10.1137/21M1411433}.
\newblock URL \url{https://doi.org/10.1137/21m1411433}.

\bibitem[Ruthotto and Haber(2020)]{PDEmotivatedNN}
Lars Ruthotto and Eldad Haber.
\newblock Deep neural networks motivated by partial differential equations.
\newblock \emph{J. Math. Imaging Vis.}, 62\penalty0 (3):\penalty0 352--364,
  2020.
\newblock \doi{10.1007/s10851-019-00903-1}.
\newblock URL \url{https://doi.org/10.1007/s10851-019-00903-1}.

\bibitem[Santambrogio(2015)]{OTAppliedMathematician}
Filippo Santambrogio.
\newblock Optimal transport for applied mathematicians. calculus of variations,
  pdes and modeling.
\newblock 2015.
\newblock URL \url{https://www.math.u-psud.fr/~filippo/OTAM-cvgmt.pdf}.

\bibitem[Song et~al.(2021)Song, Sohl{-}Dickstein, Kingma, Kumar, Ermon, and
  Poole]{SongDiffusionModels}
Yang Song, Jascha Sohl{-}Dickstein, Diederik~P. Kingma, Abhishek Kumar, Stefano
  Ermon, and Ben Poole.
\newblock Score-based generative modeling through stochastic differential
  equations.
\newblock In \emph{9th International Conference on Learning Representations,
  {ICLR} 2021, Virtual Event, Austria, May 3-7, 2021}. OpenReview.net, 2021.
\newblock URL \url{https://openreview.net/forum?id=PxTIG12RRHS}.

\bibitem[Sonoda and Murata(2019)]{TransportAnalysisofDL}
Sho Sonoda and Noboru Murata.
\newblock Transport analysis of infinitely deep neural network.
\newblock \emph{J. Mach. Learn. Res.}, 20:\penalty0 2:1--2:52, 2019.
\newblock URL \url{http://jmlr.org/papers/v20/16-243.html}.

\bibitem[Spivak(1965)]{CalculusOnManifolds}
Michael~D. Spivak.
\newblock \emph{Calculus on Manifolds: A Modern Approach to Classical Theorems
  of Advanced Calculus}.
\newblock Harper Collins Publishers, 1965.
\newblock ISBN 0805390219.
\newblock URL
  \url{https://www.maa.org/press/maa-reviews/calculus-on-manifolds-a-modern-approach-to-classical-theorems-of-advanced-calculus}.

\bibitem[Stein(1970)]{SteinBook}
Elias~M. Stein.
\newblock \emph{Singular Integrals and Differentiability Properties of
  Functions (PMS-30)}.
\newblock Princeton University Press, 1970.
\newblock ISBN 9780691080796.
\newblock URL \url{http://www.jstor.org/stable/j.ctt1bpmb07}.

\bibitem[Tabak and Vanden-Eijnden(2010)]{cms/1266935020}
Esteban~G. Tabak and Eric Vanden-Eijnden.
\newblock {Density estimation by dual ascent of the log-likelihood}.
\newblock \emph{Communications in Mathematical Sciences}, 8\penalty0
  (1):\penalty0 217 -- 233, 2010.

\bibitem[Teshima et~al.(2020{\natexlab{a}})Teshima, Ishikawa, Tojo, Oono,
  Ikeda, and Sugiyama]{DisAppro1}
Takeshi Teshima, Isao Ishikawa, Koichi Tojo, Kenta Oono, Masahiro Ikeda, and
  Masashi Sugiyama.
\newblock Coupling-based invertible neural networks are universal
  diffeomorphism approximators.
\newblock In Hugo Larochelle, Marc'Aurelio Ranzato, Raia Hadsell,
  Maria{-}Florina Balcan, and Hsuan{-}Tien Lin, editors, \emph{Advances in
  Neural Information Processing Systems 33: Annual Conference on Neural
  Information Processing Systems 2020, NeurIPS 2020, December 6-12, 2020,
  virtual}, 2020{\natexlab{a}}.
\newblock URL
  \url{https://proceedings.neurips.cc/paper/2020/hash/2290a7385ed77cc5592dc2153229f082-Abstract.html}.

\bibitem[Teshima et~al.(2020{\natexlab{b}})Teshima, Tojo, Ikeda, Ishikawa, and
  Oono]{SupApproximation}
Takeshi Teshima, Koichi Tojo, Masahiro Ikeda, Isao Ishikawa, and Kenta Oono.
\newblock Universal approximation property of neural ordinary differential
  equations.
\newblock \emph{CoRR}, abs/2012.02414, 2020{\natexlab{b}}.
\newblock URL \url{https://arxiv.org/abs/2012.02414}.

\bibitem[Tzen and Raginsky(2019)]{NeuralSDE}
Belinda Tzen and Maxim Raginsky.
\newblock Theoretical guarantees for sampling and inference in generative
  models with latent diffusions.
\newblock In Alina Beygelzimer and Daniel Hsu, editors, \emph{Conference on
  Learning Theory, {COLT} 2019, 25-28 June 2019, Phoenix, AZ, {USA}}, volume~99
  of \emph{Proceedings of Machine Learning Research}, pages 3084--3114. {PMLR},
  2019.
\newblock URL \url{http://proceedings.mlr.press/v99/tzen19a.html}.

\bibitem[Villani(2008)]{OptimalOldAndNew}
Cedric Villani.
\newblock \emph{Optimal Transport: Old and New}.
\newblock Springer, 2008.
\newblock ISBN 9783540710493.

\bibitem[Wang et~al.(2023)Wang, Baptista, Marzouk, Ruthotto, and
  Verma]{2310.16975}
Zheyu~Oliver Wang, Ricardo Baptista, Youssef Marzouk, Lars Ruthotto, and
  Deepanshu Verma.
\newblock Efficient neural network approaches for conditional optimal transport
  with applications in bayesian inference, 2023.

\bibitem[Yarotsky(2017)]{NNApproximation1}
Dmitry Yarotsky.
\newblock Error bounds for approximations with deep relu networks.
\newblock \emph{Neural Networks}, 94:\penalty0 103--114, 2017.
\newblock \doi{10.1016/j.neunet.2017.07.002}.
\newblock URL \url{https://doi.org/10.1016/j.neunet.2017.07.002}.

\bibitem[Yarotsky and Zhevnerchuk(2020)]{NNApproximation4}
Dmitry Yarotsky and Anton Zhevnerchuk.
\newblock The phase diagram of approximation rates for deep neural networks.
\newblock In Hugo Larochelle, Marc'Aurelio Ranzato, Raia Hadsell,
  Maria{-}Florina Balcan, and Hsuan{-}Tien Lin, editors, \emph{Advances in
  Neural Information Processing Systems 33: Annual Conference on Neural
  Information Processing Systems 2020, NeurIPS 2020, December 6-12, 2020,
  virtual}, 2020.
\newblock URL
  \url{https://proceedings.neurips.cc/paper/2020/hash/979a3f14bae523dc5101c52120c535e9-Abstract.html}.

\bibitem[Zech and Marzouk(2022{\natexlab{a}})]{ZM1}
Jakob Zech and Youssef Marzouk.
\newblock Sparse approximation of triangular transports, part i: The
  finite-dimensional case.
\newblock \emph{Constructive Approximation}, 2022{\natexlab{a}}.
\newblock \doi{https://doi.org/10.1007/s00365-022-09569-2}.

\bibitem[Zech and Marzouk(2022{\natexlab{b}})]{ZM2}
Jakob Zech and Youssef Marzouk.
\newblock Sparse approximation of triangular transports, part ii: The
  infinite-dimensional case.
\newblock \emph{Constructive Approximation}, 2022{\natexlab{b}}.
\newblock \doi{https://doi.org/10.1007/s00365-022-09570-9}.

\end{thebibliography}

\end{document}